\documentclass{article}

\usepackage[noorphans,vskip=0ex]{quoting}

\newcommand{\boxin}[1]{\textcolor{blue}{#1}}
\usepackage{microtype}
\usepackage{graphicx}
\usepackage{subfigure}
\usepackage{booktabs} 
\usepackage{enumitem}

\usepackage{amsmath,amsfonts,bm, amsthm}

\DeclareMathOperator{\tr}{tr}
\DeclareMathOperator{\rowsp}{rowsp}

\newtheorem{theorem}{Theorem}
\newtheorem{assumption}{Assumption}
\newtheorem{proposition}{Proposition}

\newtheorem{lemma}{Lemma}
\newtheorem{definition}{Definition}

\numberwithin{theorem}{section}
\numberwithin{lemma}{section}
\numberwithin{proposition}{section}

\newcommand*\diff{\mathop{}\!\mathrm{d}}

\def\eqref#1{equation~\ref{#1}}

\def\1{\bm{1}}

\def\vb{{\bm{b}}}

\def\vu{{\bm{u}}}
\def\vv{{\bm{v}}}
\def\vw{{\bm{w}}}
\def\vx{{\bm{x}}}
\def\vy{{\bm{y}}}

\def\mA{{\bm{A}}}

\def\mH{{\bm{H}}}
\def\mI{{\bm{I}}}
\def\mJ{{\bm{J}}}

\def\mM{{\bm{M}}}
\def\mN{{\bm{N}}}

\def\mP{{\bm{P}}}

\def\mS{{\bm{S}}}
\def\mT{{\bm{T}}}
\def\mU{{\bm{U}}}
\def\mV{{\bm{V}}}
\def\mW{{\bm{W}}}

\DeclareMathAlphabet{\mathsfit}{\encodingdefault}{\sfdefault}{m}{sl}
\SetMathAlphabet{\mathsfit}{bold}{\encodingdefault}{\sfdefault}{bx}{n}

\def\gD{{\mathcal{D}}}

\def\gJ{{\mathcal{J}}}

\def\sG{{\mathbb{G}}}

\def\sJ{{\mathbb{J}}}

\def\sR{{\mathbb{R}}}

\newcommand{\E}{\mathbb{E}}
\newcommand{\Ls}{\mathcal{L}}

\newcommand{\supp}{\text{supp}}

\DeclareMathOperator*{\argmax}{arg\,max}
\DeclareMathOperator*{\argmin}{arg\,min}

\usepackage{hyperref}
\usepackage{autonum}

\usepackage[accepted]{icml2021}

\icmltitlerunning{Uncovering the Connections Between 
Adversarial Transferability and Knowledge Transferability}

\begin{document}

\twocolumn[
\icmltitle{Uncovering the Connections Between \\
Adversarial Transferability and Knowledge Transferability}





\icmlsetsymbol{equal}{*}

\begin{icmlauthorlist}
\icmlauthor{Kaizhao Liang}{equal,uiuc}
\icmlauthor{Jacky Y. Zhang}{equal,uiuc}
\icmlauthor{Boxin Wang}{uiuc}
\icmlauthor{Zhuolin Yang}{uiuc}
\icmlauthor{Oluwasanmi Koyejo }{uiuc}
\icmlauthor{Bo Li}{uiuc}
\end{icmlauthorlist}

\icmlaffiliation{uiuc}{Department of Computer Science, the University of Illinois at Urbana-Champaign, Urbana, USA}

\icmlcorrespondingauthor{Oluwasanmi Koyejo}{sanmi@illinois.edu}
\icmlcorrespondingauthor{Bo Li}{lbo@illinois.edu}

\icmlkeywords{Machine Learning, ICML}

\vskip 0.3in
]



\printAffiliationsAndNotice{\icmlEqualContribution} 

\begin{abstract}

Knowledge transferability, or transfer learning, has been widely adopted to allow a pre-trained model in the source domain to be effectively adapted to downstream tasks in the target domain. It is thus important to explore and understand the factors affecting knowledge transferability. In this paper, as the first work, we analyze and demonstrate the connections between knowledge transferability and another important phenomenon--adversarial transferability, \emph{i.e.}, adversarial examples generated against one model can be transferred to attack other models. Our theoretical studies show that adversarial transferability indicates knowledge transferability, and vice versa. Moreover, based on the theoretical insights, we propose two practical adversarial transferability metrics to characterize this process,  serving as bidirectional indicators between adversarial and knowledge transferability. We conduct extensive experiments for different scenarios on diverse datasets, showing a positive correlation between adversarial transferability and knowledge transferability. Our findings will shed light on future research about effective knowledge transfer learning and adversarial transferability analyses. All code and data are available \href{https://github.com/AI-secure/Uncovering-the-Connections-BetweenAdversarial-Transferability-and-Knowledge-Transferability}{here}.

\end{abstract}

\section{Introduction}
Knowledge transfer is quickly becoming the standard approach for fast learning adaptation across domains. Also known as transfer learning or learning transfer, knowledge transfer has been a critical technology for enabling several real-world applications, including object detection~\cite{zhang2014facial}, image segmentation~\cite{kendall2018multi}, multi-lingual machine translation~\cite{dong2015multi}, and language understanding evaluation~\cite{wang2018glue}, among others.
For example, since the release of ImageNet~\citep{russakovsky2015imagenet}, pretrained ImageNet models (e.g., on TensorFlow Hub or PyTorch-Hub) have become the default option for the knowledge transfer source due to its broad coverage of visual concepts and compatibility with various visual tasks~\citep{huh2016makes}. Motivated by its importance, many studies have explored the factors associated with knowledge transferability. Most recently, \citet{salman2020adversarially} showed that more robust pretrained ImageNet models transfer better to downstream tasks, which reveals that \textit{adversarial training} helps to improve knowledge transferability.

In the meantime, \textit{adversarial transferability} has been extensively studied---a phenomenon that an adversarial instance generated against one model has high probability attack another one without additional modification~\citep{papernot2016transferability,goodfellow2014explaining,joon2017adversarial}. 
Hence, adversarial transferability is widely exploited in black-box attacks~\citep{ilyas2018black,liu2016delving, naseer2019cross}. 
A line of work has been conducted to bound the adversarial transferability based on model (gradient) similarity~\cite{tramer2017space}. 
Given that both \textit{adversarial transferability} and \textit{knowledge transferability} are impacted by certain model similarity and adversarial ML properties, in this work, we aim to conduct the \textit{first} study to analyze the connections between them and ask,

\begin{quoting}
\itshape
\vspace*{-0.2em}
 What is the fundamental connection between knowledge transferability and adversarial transferability? Can we measure one and indicate the other?
\vspace*{-0.1em}
\end{quoting}

{\bf {Technical Contributions.}}
In this paper, we take the \textit{first} step
towards exploring the fundamental relation between adversarial transferability and knowledge transferability.
We make contributions on both theoretical
and empirical fronts.   
\vspace{-0.5cm}
\begin{itemize}[leftmargin=*,itemsep=-0.5mm]
        \vspace{-0.8em}
    \item We formally define the adversarial transferability for the \textit{first} time by considering all potential adversarial perturbation vectors. We then conduct thorough and novel theoretical analysis  to characterize the precise connection between adversarial transferability and knowledge transferability based on our definition. 
    
    \item In particular, we prove that high adversarial transferability will indicate high knowledge transferability, which can be represented as the distance in an inner product space defined by the Hessian of the adversarial loss.
    In the meantime, we prove that high knowledge transferability will indicate high adversarial transferability.
    
    \item Based on our theoretical insights, we propose two practical adversarial transferability metrics that quantitatively measure the adversarial transferability in practice. We then provide simulational results to verify how these metrics connect with the knowledge transferability in a bidirectional way.
    
    \item Extensive experiments justify our theoretical insights and the proposed adversarial transferability metrics, leading to our discussion on potential applications and future research.
\end{itemize}

\textbf{Related Work} There is a line of research studying different factors that affect knowledge transferability~\citep{yosinski2014transferable,long2015learning,wang2019characterizing,xu2019larger,shinya2019understanding}. Further, empirical observations show that the correlation between learning tasks~\citep{achille2019task2vec, zamir2018taskonomy}, the similarity of model architectures, and data distribution are all correlated with different knowledge transfer abilities. Interestingly, recent empirical evidence suggests that adversarially-trained models transfer  better than non-robust models \cite{salman2020adversarially, utrera2020adversarially}, suggesting a connection between the adversarial properties and knowledge transferability. On the other hand, several approaches have been proposed to boost the adversarial transferability~\citep{zhou2018transferable,demontis2019adversarial,dong2019evading,xie2019improving}. Beyond the above empirical studies, there are a few existing analyses of adversarial transferability, which explore different conditions that may enhance adversarial transferability~\citep{athalye2018obfuscated,tramer2017space,ma2018characterizing,demontis2019adversarial}. 
In this work, we aim to bridge the connection between adversarial and knowledge transferability, both of which reveal interesting properties of ML model similarities from different perspectives.

\section{Adversarial Transferability and Knowledge Transferability}\label{sec:preliminaries}

This section introduces the preliminaries and the formal definitions of the knowledge and adversarial transferability, and formally defines our problem of interest. 

\textbf{Notation.} Sets are denoted in blackboard bold, e.g., $\sR$, and the set of integers $\{1\dots n\}$ is denoted as $[n]$. Distributions are denoted in calligraphy, e.g., $\gD$, and the support of a distribution $\gD$ is denoted as $\supp(\gD)$. Vectors are denoted as bold lower case letters, e.g., $\vx \in \sR^n$, and matrices are denoted as bold uppercase letters, e.g., $\mW$. We denote the entry-wise product operator between vectors or matrices as $\odot$. The Moore–Penrose inverse of a matrix $\mW$ is denoted as $\mW^\dagger$.
We use $\|\cdot\|_2$ to denote Euclidean norm induced by Euclidean inner product $\langle \cdot, \cdot \rangle$.
The standard inner product of two matrices is defined as $\langle \mW, \mM \rangle=\tr(\mW^\top \mM)$, where $\tr(\cdot)$ is  the trace of a matrix. The Frobenius norm $\|\cdot\|_F$ is induced by the standard matrix inner product.
Moreover, in the (semi-)inner product space defined by a positive (semi-)definite matrix $\mS$, the (semi-)inner product of two vectors or matrices is defined by $\langle \vv_1, \vv_2 \rangle_\mS=\vv_1^\top \mS \vv_2$ or $\langle \mW, \mM \rangle_\mS=\tr(\mW^\top \mS \mM)$, respectively. Given a vector $\vv$, we define its normalization as $\widehat{\vv}=\vv / \|\vv\|_2$. When using a denominator $\|\cdot\|_*$ other than Euclidean norm, we denote the normalization as $\widehat{\vv}|_*$.

Given a (vector-valued) function $f$, we denote $f(\vx)$ as its evaluated value at $\vx$, and $f$ represents the function itself in the corresponding Hilbert space. 
Composition of functions is denoted as $g\circ f(\vx)=g(f(\vx))$. We use $\langle \cdot, \cdot \rangle_\gD$ to denote the inner product induced by distribution $\gD$ and inherited from Euclidean inner product, i.e., $\langle f_1, f_2 \rangle_\gD = \E_{\vx\sim \gD} \langle f_1(\vx), f_2(\vx)\rangle$.
Accordingly, we use $\|\cdot\|_{\gD}$ to denote the norm induced by the inner product $\langle \cdot, \cdot \rangle_\gD$, i.e., $\|f\|_{\gD}=\sqrt{\langle f, f \rangle_\gD}$. When the inherited inner product is defined by $\mS$, we denote $\langle f_1, f_2 \rangle_{\gD, \mS} = \E_{\vx\sim \gD} \langle f_1(\vx), f_2(\vx)\rangle_\mS$, and similarly for $\|f\|_{\gD, \mS}$.



\textbf{Knowledge Transferability}
Given a pre-trained {\em source} model $f_S:\sR^n \to \sR^m$ and a {\em target} domain $\vx \in \sR^n$ with data distribution $\vx \sim \gD$ and {\em target} labels $y(\vx)\in \sR^d$, {\em knowledge transferability} is defined as the performance of fine-tuning $f_S$ on $\gD$ to predict $y$. Concretely, knowledge transferability can be represented as a loss $\Ls(\ \cdot \ , y, \gD)$ after fine-tuning by composing the fixed source model with a trainable function $g:\sR^m\to \sR^d$, typically from a small function class $g\in \sG$, \textit{i.e.},
\begin{align}
    \min_{g\in \sG} \quad  \Ls(g\circ f_S, y, \gD),  \label{def:knowledge-transf}
\end{align}
where the loss function $\Ls$ measures the error between $g\circ f_S$ and the ground truth $y$ under the {\em target} data distribution $\gD$. For example, for neural networks it is usual to stack on and fine-tune a linear layer; here $\sG$ is the affine function class. We will focus on the affine setting in this paper. 

For our purposes, a more useful measure of transfer is to compare the quality of the fine-tuned model to a model trained directly on the target domain $f_T:\sR^n\to \sR^d$. Thus, we study the following surrogate of knowledge transferability, where the ground truth target is replaced by a reference target model $f_T$:
\begin{align}
    \min_{g\in \sG} \quad  \Ls(g\circ f_S, f_T, \gD).  \label{def:knowledge-transf-fT}
\end{align}

\begin{figure*}[t!]
    \centering
    \begin{minipage}{0.57\linewidth}
    \centering
        \includegraphics[width=\linewidth]{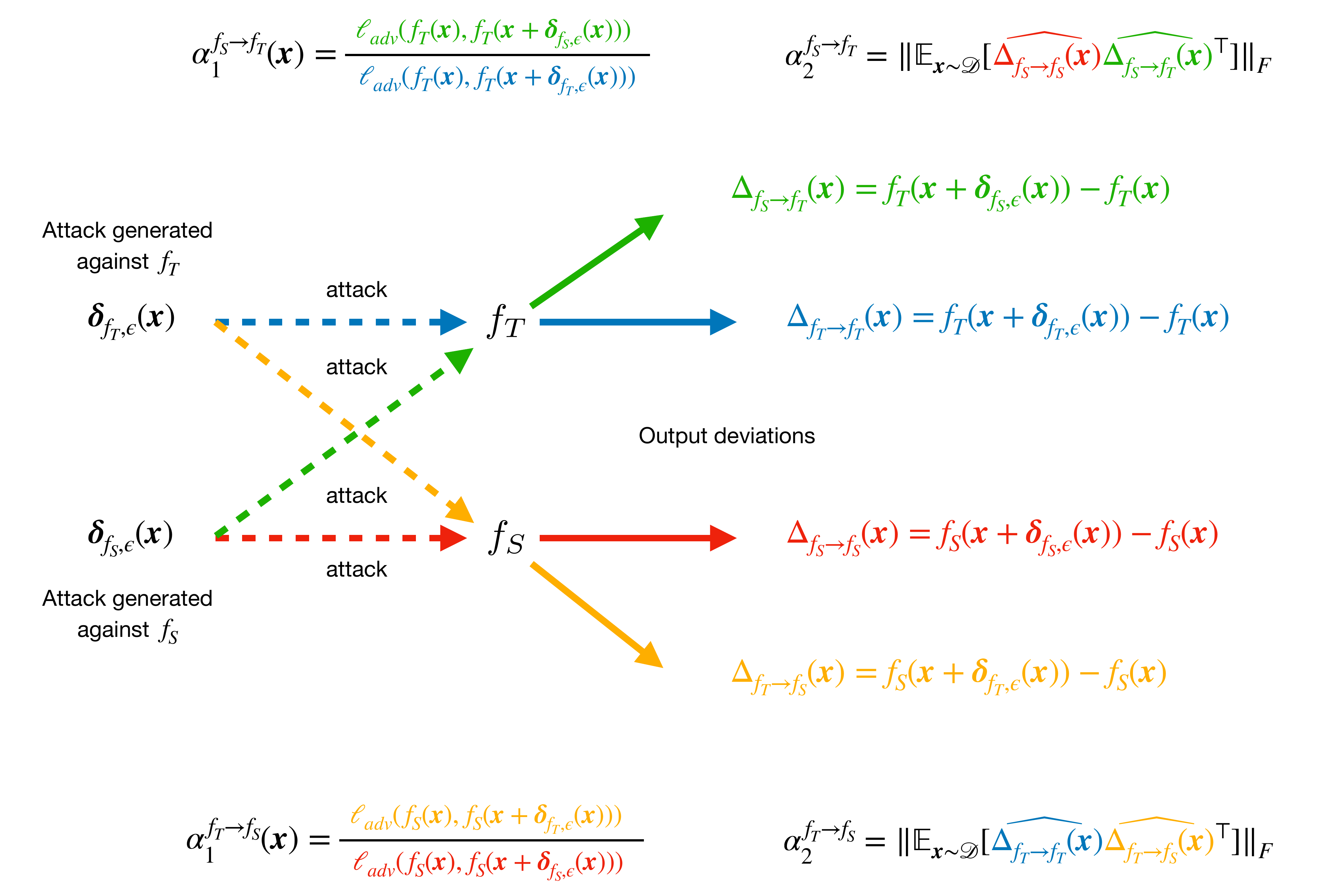}
        \small (a) 
    \end{minipage}
    \begin{minipage}{0.4\linewidth}
    \centering
        \includegraphics[width=\linewidth]{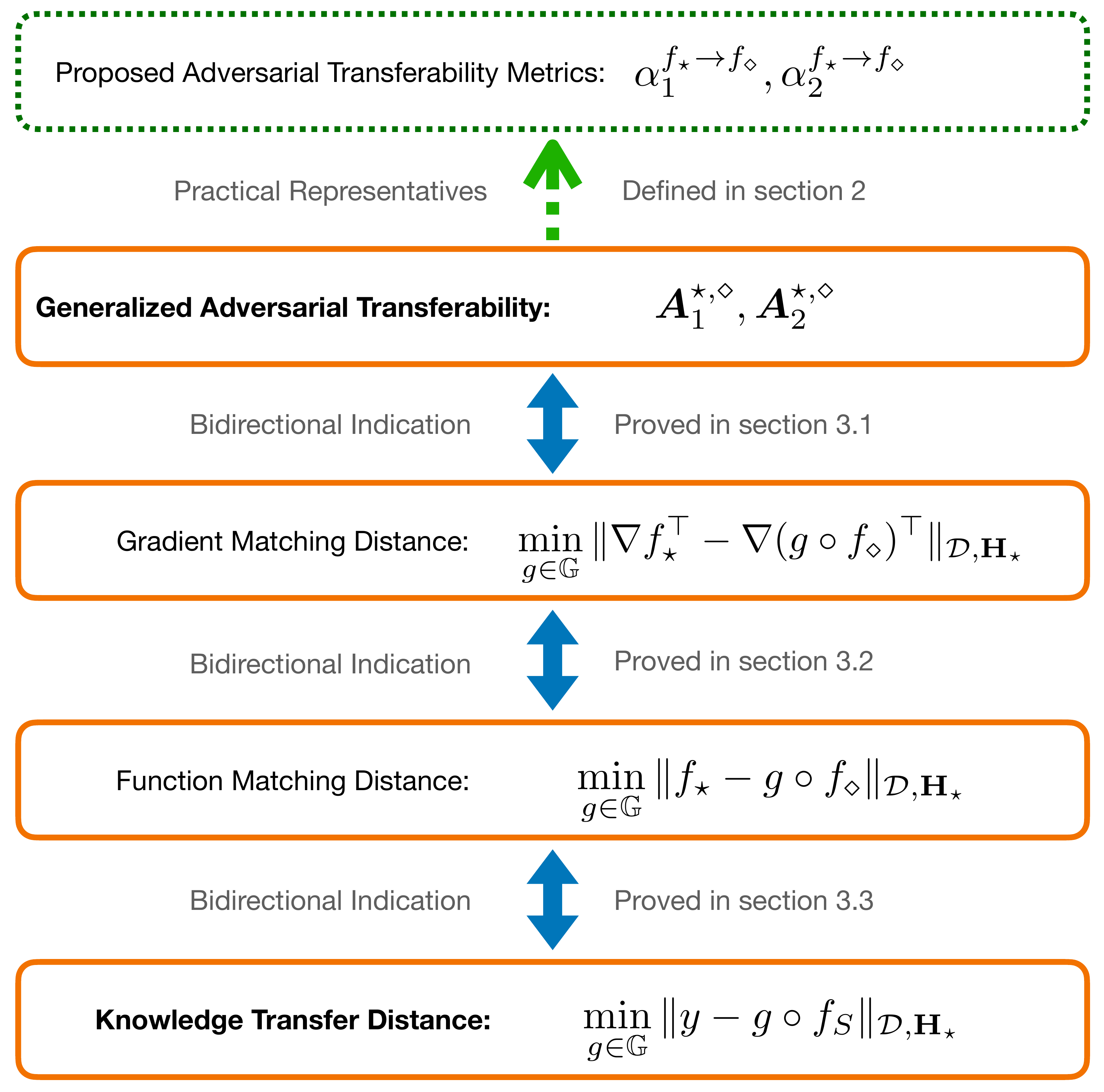}
        \small (b) 
    \end{minipage}
    \caption{\small (a) An illustration of the two proposed adversarial transferability metrics $\alpha_1, \alpha_2$ under different adversarial transferability settings, \emph{i.e.}, $\alpha_1^{f_S\to f_T}, \alpha_1^{f_T\to f_S}, \alpha_2^{f_S\to f_T}$, and $\alpha_2^{f_T\to f_S}$. (b) An overview of the theoretical analysis framework, and its practical inspirations, where $\star, \diamond \in \{T, S\}$ and $\star \neq \diamond$. The three blue double-headed arrows are the bidirectional indication relationships proved in our theory section, and the dashed green arrow shows in practice how the two proposed adversarial transferability metrics are measured as representatives of the generalized adversarial transferability  based on our theory. 
    }
    \label{fig:theory-conclusion}
\end{figure*}

\textbf{Adversarial Attacks.}  
For simplicity we consider untargeted attacks that seeks to maximize the deviation of model output
as measured by a given adversarial loss function $\ell_{adv}(\cdot, \cdot)$. The targeted attack can be viewed as a special case. 
Without loss of generality, we assume the adversarial loss is \emph{non-negative}. 
Given a datapoint $\bm{x}$ and model $f$, an adversarial example of magnitude $\epsilon$ is denoted by $\bm{\delta}_{f,\epsilon}(\bm{x})$, computed as:
\begin{align}
    \bm{\delta}_{f,\epsilon}(\bm{x}) \ =\  \argmax_{\|\bm{\delta}\|\leq \epsilon} \   \ell_{adv}(f(\vx),f(\vx+\bm{\delta}) ). \label{def:adv-attack}
\end{align} 
We note that in theory $\bm{\delta}_{f,\epsilon}(\bm{x})$ may not be unique, and its generalized definition and its discussion are provided in our theoretical analysis (Section~\ref{sec:theory}). 

\textbf{Adversarial Transferability.} The process of adversarial transfer involves applying the adversarial example generated against a model $f_1$ to another model $f_2$. Thus, adversarial transferability from $f_1$ to $f_2$ measures how well $\bm{\delta}_{f_1,\epsilon}$ attacks $f_2$. We propose two metrics, namely, $\alpha_1$ and $\alpha_2$ 
that characterize adversarial transferability from complementary perspectives. To provide a visual overview of our definitions for the proposed adversarial transferability metrics, we present an illustration in Figure~\ref{fig:theory-conclusion} (a).
\begin{definition}[The First Adversarial Transferability]\label{def:alpha_1}
The first adversarial transferability from $f_1$ to $f_2$ at data sample $\vx\sim \gD$, is defined as
\begin{align}
    \alpha_1^{f_1\to f_2}(\vx) =\frac{\ell_{adv}(f_2(\vx),f_2(\vx+\bm{\delta}_{f_1, \epsilon}(\bm{x})) )}{\ell_{adv}(f_2(\vx),f_2(\vx+\bm{\delta}_{f_2, \epsilon}(\vx)) )} .
\end{align}
Taking the expectation, the first adversarial transferability is defined as
\begin{align}
    \alpha_1^{f_1\to f_2} =\E_{\vx\sim \gD}\left[ \alpha^{f_1\to f_2}(\vx)\right].
\end{align}
\end{definition}
Observe that the first adversarial transferability characterize how well the adversarial  attacks $\bm{\delta}_{f_1, \epsilon}$ generated against $f_1$ perform on $f_2$, compared to $f_2$'s whitebox adversarial attacks $\bm{\delta}_{f_2, \epsilon}$. Thus, high $\alpha_1$ indicates high adversarial transferability. Note that the two attacks use the same magnitude constraint $\epsilon$. 

Recall that $\ell_{adv}(f(\vx),f(\vx+\bm{\delta}) )$ measures the effect of the attack $\bm{\delta}$ on the model output $f(\vx)$. 
$\alpha_1$ characterizes the relative magnitude of this deviation. However, this magnitude information is incomplete, as the direction of the deviation also encodes information about the adversarial transfer process. To this end, we propose the second adverserial metric, inspired by our theoretical analysis, which characterizes adversarial transferability from the directional perspective. 
\begin{definition}[The Second Adversarial Transferability]\label{def:alpha_2}
    The second adversarial transferability from $f_1$ to $f_2$, under data distribution $\vx\sim \gD$, is defined as
    \begin{align}
        \alpha_2^{f_1\to f_2}&=\|\E_{\vx\sim \gD}[ \widehat{\Delta_{f_1\to f_1}(\vx)} \widehat{\Delta_{f_1\to f_2}(\vx)}^\top ]\|_F,
    \end{align}
    where 
    \begin{align}
        \Delta_{f_1\to f_1}(\vx)&=f_1(\vx+\bm{\delta}_{f_1,\epsilon}(\vx))-f_1(\vx)\\
        \Delta_{f_1\to f_2}(\vx)&=f_2(\vx+\bm{\delta}_{f_1,\epsilon}(\vx))-f_2(\vx)
    \end{align}
    are deviations in model output given the adversarial attack $\bm{\delta}_{f_1,\epsilon}(\vx)$ generated against $f_1$, and $\widehat{\ \cdot\ }$ denotes the corresponding unit-length vector.
\end{definition}
To further clarify the second adversarial transferability metric, consider the following alternative form of $\alpha_2$. 
\begin{proposition} \label{prop:1} The $\alpha_2^{f_1\to f_2}$ can be reformulated as
\begin{align}
    (\alpha_2^{f_1\to f_2})^2   = \E_{\vx_1, \vx_2} \left[  \theta_{f_1\to f_1}(\vx_1, \vx_2)\theta_{f_1\to f_2}(\vx_1, \vx_2)\right],
\end{align}
where $\vx_1, \vx_2\overset{\text{i.i.d.}}{\sim}\gD$, and 
\begin{align}
    \theta_{f_1\to f_1}(\vx_1, \vx_2) &=\langle \widehat{\Delta_{f_1\to f_1}(\vx_1)}, \widehat{\Delta_{f_1\to f_1}(\vx_2)} \rangle \\
    \theta_{f_1\to f_2}(\vx_1, \vx_2) &=\langle \widehat{\Delta_{f_1\to f_2}(\vx_1)}, \widehat{\Delta_{f_1\to f_2}(\vx_2)} \rangle 
\end{align}
\end{proposition}
We can see that high $\alpha_2$ indicates that it is more likely for the two inner products (\emph{i.e.}, $\theta_{f_1\to f_1}$ and $\theta_{f_1\to f_2}$) to have the same sign. Given that the direction of $f_1$'s output deviation indicates its attack $\bm{\delta}_{f_1,\epsilon}$, and the direction of $f_2$'s output deviation indicates the transferred attack $\bm{\delta}_{f_1,\epsilon}$, high $\alpha_2$ implies that the two directions will rotate by a similar angle as the data changes.

$\alpha_1$ and $\alpha_2$ represent complementary aspects of the adversarial transferability: 
$\alpha_1$ can be understood as how often the adversarial attack transfers, while $\alpha_2$ encodes directional information of the output deviation caused by adversarial attacks. An example is provided in the appendix section~\ref{sec:example-alpha} to illustrate the necessity of both the metrics in characterizing the relation between adversarial transferability and knowledge transferability.   To jointly take the two adversarial transferability metrics into consideration, we propose the following metric as the combined value of $\alpha_1$ and $\alpha_2$. 
\begin{align}
    &(\alpha_1 * \alpha_2)^{f_1\to f_2}=\\
    &\qquad \big\|\E_{\vx\sim \gD}[ \alpha_1^{f_1\to f_2}(\vx)\widehat{\Delta_{f_1\to f_1}(\vx)} \widehat{\Delta_{f_1\to f_2}(\vx)}^\top ]\big\|_F.
\end{align}


We defer the justification for the combined adversarial transferability metric in the next section, and move on to state a useful proposition.
\begin{proposition}\label{prop:2}
The adversarial transferabililty metrics $\alpha_1^{f_1\to f_2}$, $\alpha_2^{f_1\to f_2}$ and $(\alpha_1 * \alpha_2)^{f_1\to f_2}$ are in $[0,1]$.
\end{proposition}
So far, we have defined knowledge transferability, and two adversarial trasferability metrics. We can now analyze their connections more precisely. 

\textbf{Problem of Interest.} Given a {\em source} model $f_S:\sR^n \to \sR^m$, the {\em target} data distribution $\vx\sim \gD$, the ground truth target $y:\sR^n\to \sR^d$, and a {\em target} reference model $f_T:\sR^n\to \sR^d$, we aim to study how the adversarial transferability between $f_S$ and $f_T$, characterized by the two proposed adversarial transferability metrics, connects to the knowledge transfer loss $\min_{g\in \sG}  \Ls(g\circ f_S, y, \gD)$ with affine functions $g\in \sG$ (\eqref{def:knowledge-transf}). 
\section{Theoretical Analysis}\label{sec:theory}
In this section, we present the theoretical analysis on how the adversarial transferability and the knowledge transfer process are tied together. 
To simplify the discussion, as the objects studied in this section are specifically focused on the source domain $S$ and the target domain $T$, we can use $\star$ or $\diamond$ as a placeholder for either $S$ or $T$ throughout this section. 

\textbf{Theoretical Analysis Overview.} 
In subsection~\ref{subsec:theory-1}, we define the two \textit{generalized adversarial transferabilities}, (\textit{i.e.}, $\mA_1$, $\mA_2$), and present Theorem~\ref{theorem-1} showing that $\mA_1$, $\mA_2$ together determine a gradient matching distance $\min_{g\in \sG}\|\nabla f_\star- \nabla g\circ f_\diamond\|$, between the Jacobian matrices of the source and target models in an inner product space defined by the Hessian of the adversarial loss function. In the same subsection, we also show that $\alpha_1$ and $\alpha_2$ represent the most influential factors in $\mA_1$ and $\mA_2$, respectively. Next, we explore the connection to knowledge transferability in subsection~\ref{subsec:theory-2} via Theorem~\ref{theorem-2} which shows the gradient matching distance approximates the function matching distance, \emph{i.e.}, $\min_{g\in \sG}\|f_\star- g\circ f_\diamond\|$, with a distribution shift up to a Wasserstein distance. Finally, in subsection~\ref{subsec:theory-conclusion} we complete the analysis by outlining the connection between the function matching distance and the knowledge transfer loss. A visual overview is shown in Figure~\ref{fig:theory-conclusion} (b).

\textbf{Setting.} As adversarial perturbations are constrained in a small $\epsilon$-ball, it is reasonable to approximate the deviation of model outputs by its first-order Taylor approximation. Specifically, in this section we consider the Euclidean $\epsilon$-ball. Therefore, the output deviation of a function $f$ at $\vx$ given a small perturbation $\|\bm{\delta}_\epsilon\|_2\leq \epsilon$ can be approximated by
\begin{align}
  f(\vx + \bm{\delta}_\epsilon) - f(\vx) \approx \nabla f(\vx)^\top\bm{\delta}_\epsilon,
\end{align}
where $\nabla f(\vx)$ is the Jacobian matrix of $f$ at $\vx$.

We consider a convex and twice-differentiable adversarial loss function $\ell^\star_{adv}(\cdot)$ that measures the deviation of model output $f_\star(\vx + \bm{\delta}_\epsilon) - f_\star(\vx)$, with minimum $\ell^\star_{adv}(\bm{0})=0$, for $\star\in\{S,T\}$. We note that we should treat the adversarial loss on $f_S$ and $f_T$ differently, as they may have different output dimensions.  Accordingly, the adversarial attack (\eqref{def:adv-attack}) can be written as
\begin{align}\label{def:adv-attack-grad}
    \bm{\delta}_{f_\star,\epsilon}(\vx)\ =\ \argmax_{\|\bm{\delta}\|_2\leq \epsilon}\  \ell^\star_{adv}(\nabla f_\star(\vx)^\top\bm{\delta}).
\end{align}
Another justification of the small-$\epsilon$ approximation follows the literature; since the ideal attack defined in \eqref{def:adv-attack} is often intractable to compute, much of the literature uses the proposed formulation (\ref{def:adv-attack-grad}) in practice, \emph{e.g.}, see \cite{miyato2018virtual}, with experimental results suggesting similar behaviour as the standard definition.


\textbf{The Small-$\epsilon$ Regime.} 
%
Recall that the adversarial loss $\ell^\star_{adv}(\cdot)$ studied in this section is convex, twice-differentiable, and achieves its minimum at $\bm{0}$, thus in the small $\epsilon$ regime:
\begin{align}
    \ell^\star_{adv}(\nabla f_\star(\vx)^\top\bm{\delta}_\epsilon) &= \left(\bm{\delta}_\epsilon^\top \nabla f_\star(\vx)\mH_\star \nabla f_\star(\vx)^\top\bm{\delta}_\epsilon \right)^{1/2}\\
    &= \|\nabla f_\star(\vx)^\top\bm{\delta}_\epsilon\|_{\mH_\star},
\end{align}
which is the norm of $f_\star$'s output deviation in the inner product space defined by the Hessian $\mH_\star$ of the squared adversarial loss $(\ell_{adv}^\star)^2$.

Accordingly, the adversarial attacks (\ref{def:adv-attack-grad}) can be written as
\begin{align}\label{def:adv-attack-norm}
     \bm{\delta}_{f_\star,\epsilon}(\vx)\ =\ \argmax_{\|\bm{\delta}\|_2\leq \epsilon}\  \|\nabla f_\star(\vx)^\top\bm{\delta}\|_{\mH_\star}, 
\end{align} 
and we can measure the output deviation of $f_\diamond$'s caused by $f_\star$'s adversarial attack $ \bm{\delta}_{f_\star, \epsilon}(\vx)$, denoted as:
\begin{align}\label{def:deviation-small}
    \Delta_{f_\star\to f_\diamond, \epsilon}(\vx)=\nabla f_\diamond(\vx)^\top \bm{\delta}_{f_\star, \epsilon}(\vx).
\end{align}
Note that in the small-$\epsilon$ regime, the actual value of $\epsilon$ becomes trivial (\emph{e.g., $\alpha_1$}), consequently we will omit the $\epsilon$ for notational ease:
\begin{align}
    \alpha_1^{f_\star\to f_\diamond}(\vx) &=\frac{\|\Delta_{f_\star\to f_\diamond}(\vx)\|_{\mH_\diamond}}{\|\nabla f_\diamond(\vx)\|_{\mH_\diamond}} .
\end{align}
Similarly, $\alpha_2$ can be computed using (\ref{def:deviation-small}) in Definition~\ref{def:alpha_2}, \emph{i.e.},
\begin{align}
    \alpha_2^{f_\star\to f_\diamond}&=\|\E_{\vx\sim \gD}[ \widehat{\Delta_{f_\star\to f_\star}(\vx)} \widehat{\Delta_{f_\star\to f_\diamond}(\vx)}^\top ]\|_F.
\end{align}
With these insights, next we will derive our first theorem.

\subsection{Adversarial Transfer Indicates the Gradient Matching Distance, and Vice Versa}\label{subsec:theory-1}
We present an interesting finding in this subsection, \emph{i.e.}, the generalized adversarial transferabilities $\mA_1, \mA_2$ have a direct connection to the gradient matching distance between the source model $f_S:\sR^n\to \sR^m$ and target model $f_T:\sR^n\to\sR^d$. The gradient matching distance is defined as the smallest distance an affine transformation can achieve between their Jacobians $\nabla f_T: \sR^n \to \sR^{n\times d}$ and $\nabla f_S :\sR^n\to \sR^{n\times m}$ in the inner product space defined by $\mH_\star$ and data sample distribution $\vx \sim \gD$, as shown below.
\begin{align}\label{def:gradient-matching-dist}
    \min_{g\in \sG} \quad \|\nabla f_\star^\top -  \nabla(g\circ f_\diamond)^\top\|_{\gD, \mH_\star},
\end{align}
where $g\in \sG$ are affine transformations. Note that $g: \sR^m \to \sR^d$ if $(\star, \diamond)=(T,S)$, and $g: \sR^d \to \sR^m$ if $(\star, \diamond)=(S,T)$. 
We defer the analysis of how the gradient matching distance approximates the knowledge transfer loss, and focus on its connection to adversarial transfer. 

\textbf{A Full Picture of Adversarial Transferability.} A key observation is that the adversarial attack (\eqref{def:adv-attack-norm}) is the singular vector corresponding to the largest singular value of the Jacobian $\nabla f_\star(\vx)$ in the $\mH_\star$ inner product space. Thus, information regarding other singular values that are not revealed by the adversarial attack. Therefore, we can consider other singular values, corresponding to smaller signals than the one revealed by adversarial attacks, to complete the analysis.
We denote $\bm{\sigma}_{f_\star, \mH_\star}(\vx)\in \sR^n$ as the descending (in absolute value) singular values of the Jacobian $\nabla f_\star(\vx)^\top \in \sR^{\cdot \times n}$ in the $\mH_\star$ inner product space. In other words, we denote $\bm{\sigma}_{f_\star, \mH_\star}(\vx)\in \sR^n$ as the square root of the descending eigenvalues of $\nabla f_\star(\vx) \mH_\star \nabla f_\star(\vx)^\top$, \textit{i.e.},
\begin{align} \label{def:sigma}
    \bm{\sigma}_{f_\star, \mH_\star}(\vx)=
        [\sigma_{f_\star}^{(1)}(\vx),
        \dots ,
        \sigma_{f_\star}^{(n)}(\vx)]^\top.
\end{align}
Note that the number of non-zero singular values may be less than $n$, in which case we fill the rest with zeros such that vector is $n$-dimensional. 

Since the adversarial attack $\bm{\delta}_{f_\star,\epsilon}(\vx)$ corresponds to the largest singular value $\sigma_{f_\star}(\vx)^{(1)}$, we can also generalize the adversarial attack by including all the singular vectors. \emph{i.e.},
\begin{align}\label{def:generalized-attack}
     \bm{\delta}_{f_\star}^{(i)}(\vx) \quad \text{ corresponds to }\quad \sigma_{f_\star}^{(i)}(\vx),\quad \forall i\in [n].
\end{align}
Loosely speaking, one could think $\bm{\delta}_{f_\star}^{(i)}(\vx)$ as the adversarial attack of $f_\star(\vx)$ in the subspace orthogonal to all the previous attacks, \emph{i.e.}, $\bm{\delta}_{f_\star}^{(j)}(\vx)$ for $\forall j\in[i-1]$.

Accordingly, for $\forall i\in[i]$ we denote the output deviation as
\begin{align}\label{def:generalied-Delta}
    \Delta^{(i)}_{f_\star\to f_\diamond}(\vx)=\nabla f_\diamond(\vx)^\top \bm{\delta}_{f_\star}^{(i)}(\vx).
\end{align} 
As a consequence, we \textit{generalize the first adversarial transferability} to be a $n$-dimensional vector $\bm{A}_1^{\star, \diamond}(\vx)$ including the adversarial losses of all of the generalized adversarial attacks, where the $i^{th}$ element in the vector is
\begin{align}
    \bm{A}_1^{\star, \diamond}(\vx)^{(i)}=\frac{\|\Delta_{f_\star\to f_\diamond}^{(i)}(\vx)\|_{\mH_\diamond}}{\|\nabla f_\diamond (\vx)\|_{\mH_\diamond}}. \label{def:alpha-1-generalized}
\end{align}
Note that the first entry of $\bm{A}_1^{\star, \diamond}(\vx)$ is the original adversarial transferability, \emph{i.e.},  ${\mA}_1^{\star, \diamond}(\vx)^{(1)}$ is the same as the ${\alpha}_1^{f_\star\to f_\diamond}(\vx)$ in Definition~\ref{def:alpha_1}.

With the above generalization that captures the full picture of the adversarial transfer process, we able to derive the following theorem.
\begin{theorem}\label{theorem-1}
Given the target and source models $f_\star, f_\diamond$, where $(\star, \diamond)\in\{(S, T), (T,S)\}$, the gradient matching distance (\eqref{def:gradient-matching-dist}) can be written as
    \begin{align}
        & \min_{g\in \sG} \quad \|\nabla f_\star^\top -  \nabla(g\circ f_\diamond)^\top\|_{\gD, \mH_\star}= \label{eq:theorem-1}\\
        & \sqrt{1-\frac{\E[\vv^{\star,\diamond}_{}(\vx_1)^\top {\mA}_2^{\star,\diamond}({\vx_1, \vx_2}) \vv^{\star,\diamond}_{}(\vx_2)]  }{\|\nabla f_\star^\top\|_{\gD, \mH_\star}^2\cdot \|\mJ^\dagger\|^{-1}_{\mH_\diamond}} }\|\nabla f_\star^\top\|_{\gD, \mH_\star},
    \end{align}
    where the expectation is taken over $\vx_1, \vx_2 \overset{\text{i.i.d.}}{\sim} \gD$, and
    \begin{align}
        \vv^{\star,\diamond}_{}(\vx)&=\sigma^{(1)}_{f_\diamond, \mH_\diamond}(\vx)\bm{\sigma}_{f_\star, \mH_\star}(\vx) \odot \mA_1^{\star, \diamond}(\vx)\\
        \mJ&=\E_{\vx\sim\gD} [\nabla f_\diamond(\vx)^\top \nabla f_\diamond(\vx)].
    \end{align}
    Moreover, ${\mA}_2^{\star,\diamond}({\vx_1, \vx_2})$ is a matrix, and its element in the $i^{th}$ row and $j^{th}$ column is 
    \begin{align}
        {\mA}_2^{\star,\diamond}(\vx_1, \vx_2&)^{(i, j)}=\langle \widehat{ \Delta^{(i)}_{f_\star\to f_\star}(\vx_1)}  \big |_{\mH_\star}, \widehat{ \Delta^{(j)}_{f_\star\to f_\star}(\vx_2)}\big |_{\mH_\star} \rangle\\
        &\cdot \langle \widehat{ \Delta^{(i)}_{f_\star\to f_\diamond}(\vx_1)}\big |_{\mH_\diamond}, \widehat{ \Delta^{(j)}_{f_\star\to f_\diamond}(\vx_2)}\big |_{\mH_\diamond} \rangle_{\widehat{ \mJ^\dagger}|_{\mH_\diamond}} .
    \end{align}
\end{theorem}
Recall the alternative representation of the second adversarial transferability $\alpha_2$, and we can immediately observe that $\alpha_2$ is determined by $\mA_2$. Therefore, both $\alpha_1$ and $\alpha_2$ appear in this relation. Let us interpret the theorem, and justify the two proposed adversarial transferability metrics. 

\textbf{Interpretation of Theorem~\ref{theorem-1}.} First, we consider components that are not directly related to the adversarial transfer in the RHS of (\ref{eq:theorem-1}). The $\|\nabla f_\star^\top\|_{\gD, \mH_\star}$ outside represents the overall magnitude of the loss. In the fraction, the $\|\nabla f_\star^\top\|_{\gD, \mH_\star}$ in the denominator normalizes the $\bm{\sigma}_{f_\star}$ in the numerator. Similarly, though more complicated, the $\|\mJ^\dagger\|^{-1}_2$ in the denominator corresponds to the $\sigma_{f_\diamond}^{(1)}$ in the numerator. We note that these are properties of $f_\star, f_\diamond$. 

Next, observe that the components directly related to the adversarial transfer process are the \textit{generalized adversarial transferability} $\mA_1$ and $\mA_2$. Let us neglect the superscript $^{(i)}$ or $^{(j)}$ for now, so we can see that their interpretations are the same as we introduced for $\alpha_1$ and $\alpha_2$ in section~\ref{sec:preliminaries}. That is, $\mA_1$ captures the magnitude of the deviation in model outputs caused by adversarial attacks, while $\mA_2$ captures the direction of the deviation. A minor difference between $\alpha_2$ and $\mA_2$ is that the second inner product in the elements of $\mA_2$ is defined by a positive semi-definite matrix $\widehat{\mJ^\dagger}$. For practical implementation, we choose to neglect this term, and use the standard Euclidean inner product in $\alpha_2$, which can be understood as a stretched version of the $\widehat{\mJ^\dagger}$ inner product space. 

Moreover, as the singular vector $\bm{\sigma}_{f_\star}$ has descending entries, we can see that in the vector $\mA_1$ and the matrix $\mA_2$, the elements with superscript $^{(1)}$ have the most influence in the relations. In other words, the two proposed adversarial transferability metrics, $\alpha_1$ and $\alpha_2$, are the most influential factors in \eqref{eq:theorem-1}.  We can also see that the combined metric $(\alpha_1*\alpha_2)$ also stems from here by only considering the components with the first superscript.

To interpret the relation between the gradient matching distance and the adversarial transferabilities, we introduce the following proposition. This shows that, in general, $\mA_1$ and $\mA_2$ with their elements closer to $1$ can serve as a bidirectional indicator of a smaller gradient matching distance. 
\begin{proposition}\label{prop:3}
    In Theorem~\ref{theorem-1}, 
    \begin{align}
        0 \leq\frac{\E[\vv^{\star,\diamond}_{}(\vx_1)^\top \mA_2^{\star,\diamond}({\vx_1, \vx_2}) \vv^{\star,\diamond}_{}(\vx_2)]  }{\|\nabla f_\star^\top\|_{\gD, \mH_\star}^2\cdot \|\mJ^\dagger\|^{-1}_{\mH_\diamond}} \leq 1.
    \end{align}
\end{proposition}
In conclusion, Theorem~\ref{theorem-1} reveals a bidirectional relation between the adversarial transfer process and the gradient matching distance, where the adversarial transfer process can be encoded by the generalized adversarial transferabilities, \emph{i.e.}, $\mA_1$ and $\mA_2$. Moreover, $\alpha_1$ and $\alpha_2$ play the most influential role in their generalization, \textit{i.e.,} $\mA_1$ and $\mA_2$.

\subsection{The Gradient Matching Distance indicates the Function Matching Distance, and Vice Versa}\label{subsec:theory-2}

To bridge the gradient matching distance to the knowledge transfer loss, an immediate step is to connect the gradient distance to the function distance which directly serves as a surrogate knowledge transfer loss as defined in (\eqref{def:knowledge-transf-fT}).
Specifically, in this subsection, we present a connection between the function matching distance, \emph{i.e.},
\begin{align}\label{def:function-matching-dist}
    \min_{g\in \sG} \  \| f_\star - g\circ f_\diamond\|_{\gD, \mH_\star},
\end{align}
and the gradient matching distance, \emph{i.e.},
\begin{align}\label{def:gradient-matching-dist-new}
   \min_{g\in \sG} \ \|\nabla f_\star^\top - \nabla (g\circ f_\diamond)^\top\|_{\gD, \mH_\star},
\end{align}
where $g\in \sG$ are affine transformations. 

For intuition, consider a point $\vx_0$ in the input space $\sR^{n}$, a path $\gamma_\vx:[0,1]\to \sR^{n}$ such that $\gamma_\vx(0) = \vx_0$ and $\gamma_\vx(1)=\vx$. Then, denoting $\gamma$ as the function of $\vx$, we can write the difference between the two functions as
\begin{align}
    f_\star - {g}\circ f_\diamond = &\int_0^1 (\nabla f_\star(\gamma(t)) - \nabla (g \circ f_\diamond(\gamma(t))))^\top \dot{\gamma}(t) \diff t\\
    &+( f_\star(\vx_0)- g\circ f_\diamond(\vx_0))  .
\end{align}
Noting that the function difference is a path integral of the gradient difference, we should expect a distribution shift when characterizing their connection, \emph{i.e.}, the integral path affects the original data distribution $\gD$. Accordingly, as the integral path may leave the support of $\gD$, it is necessary to assume the smoothness of the function, as shown below. 

Denoting the optimal $g\in \sG$ in (\ref{def:function-matching-dist}) as $\Tilde{g}$, and one of the optimal $g\in \sG$ in (\ref{def:gradient-matching-dist-new}) as $\Tilde{g}'$, we define
\begin{align}
    h_{\star, \diamond} :=f_\star - \Tilde{g}\circ f_\diamond \quad \text{and} \quad h_{\star, \diamond}' :=f_\star - \Tilde{g}'\circ f_\diamond, \label{def:h}
\end{align}
and we can see that the gradient matching distance and the function matching distance can be written as
\begin{align}
    (\ref{def:function-matching-dist}) = \|h_{\star, \diamond}\|_{\gD, \mH_\star}  \quad \text{and} \quad (\ref{def:gradient-matching-dist-new}) = \|\nabla {h'_{\star, \diamond}}^\top\|_{\gD, \mH_\star}.
\end{align}
\begin{assumption}[$\beta$-smoothness]\label{assum:smooth}
   We assume $h_{\star, \diamond}$ and $h'_{\star, \diamond}$ are both $\beta$-smooth, \emph{i.e.},
    \begin{align}
        \|\nabla h_{\star, \diamond}^\top (\vx_1)-\nabla h_{\star, \diamond}^\top (\vx_2)\|_{\mH_\star} \leq \|\vx_1-\vx_2\|_{2},
    \end{align}
    and similarly for $h'_{\star, \diamond}$.
\end{assumption}
With this assumption, we can prove that the gradient matching distance and the function matching distance can bound each other. 
\begin{theorem}\label{theorem-2}
With the notation defined in \eqref{def:h}, assume the $\beta$-smoothness assumption holds. Given a data distribution $\gD$ and $\tau>0$, there exist distributions $\gD_1, \gD_2$ such that the type-1 Wasserstein distance $W_1(\gD, \gD_1)\leq \tau$ and $W_1(\gD, \gD_2)\leq \tau$ satisfying
    \begin{align}
        {\tfrac{1}{2B^2 }}\|h_{\star, \diamond}\|^2_{\gD, \mH_\star} &\leq \|\nabla h_{\star, \diamond}'^\top\|^2_{\gD_1, \mH_\star}+\beta^2(B-\tau)_+^2\\
        \tfrac{1}{3n}\|\nabla h_{\star, \diamond}'^\top\|^2_{\gD, \mH_\star}&\leq \tfrac{2}{\tau^2} \|h_{\star, \diamond}\|^2_{\gD_2, \mH_\star}+\beta^2\tau^2,
    \end{align}
    where $n$ is the dimension of $\vx\sim\gD$, and $B=\inf_{\vx_0\in \sR^n} \sup_{\vx \in \supp(\gD)} \|\vx-\vx_0\|_2$ is the radius of $\supp(\gD)$.
\end{theorem}
We note that the above theorem compromises some tightness in exchange for a cleaner presentation without losing its core message, which is discussed in the proof of the theorem. 

\textbf{Interpretation of Theorem~\ref{theorem-2}.} The theorem shows that under the smoothness assumption, the gradient matching distance indicates the function matching distance, and vice versa, with a distribution shift bounded in Wasserstein distance. As the distribution shift is in general necessary, we conjecture that using different data distributions for adversarial transfer and knowledge transfer can also be applicable.

\subsection{The Function Matching Distance Indicates Knowledge Transferability, and Vice Versa} \label{subsec:theory-conclusion}
To complete the story, it remains to connect the function matching distance to knowledge transferability. As the adversarial transfer is symmetric (\emph{i.e.}, either from $f_S\to f_T$ or $f_T\to f_S$), we are able to use the placeholders $\star, \diamond \in \{S, T\}$ all the way through. However, as the knowledge transfer is asymmetric (\emph{i.e.}, $f_S\to y$ to the target ground truth), we need to instantiate the direction of adversarial transfer to further our discussion. 

\textbf{Adversarial Transfer from $f_T\to f_S$.} As we can see from the $\mA_1^{\star, \diamond}$ in Theorem~\ref{theorem-1}, this direction corresponds to $(\star, \diamond) = (T, S)$. Accordingly, the function matching distance  (\eqref{def:function-matching-dist}) becomes
\begin{align}
        \min_{g\in \sG} \  \| f_T - g\circ f_S\|_{\gD, \mH_T}.\label{eq:function-distance-S-T}
\end{align}
We can see that \eqref{eq:function-distance-S-T} directly translates to the surrogate knowledge transfer loss that uses the ``pseudo ground truth'' from the target reference model $f_T$.

In other words, the function matching distance serves as an approximation of the knowledge transfer loss defined as their distance in the inner product space of $\mH_T$, \emph{i.e.},
\begin{align}
    \min_{g\in \sG} \ \| y - g\circ f_S\|_{\gD, \mH_T}. \label{def:know-transfer-dist}
\end{align}
The accuracy of the approximation depends on the performance of $f_T$, as shown in the following theorem.
\begin{theorem}\label{prop:4} The surrogate transfer loss (\ref{eq:function-distance-S-T}) and the true transfer loss (\ref{def:know-transfer-dist}) are close, with an error of  $\| f_T - y\|_{\gD, \mH_T}$.
\begin{align}
       -\| f_T - y\|_{\gD, \mH_T} \leq (\ref{def:know-transfer-dist})-(\ref{eq:function-distance-S-T}) \leq  \| f_T - y\|_{\gD, \mH_T}
\end{align}
\end{theorem}

\textbf{Adversarial Transfer from $f_S\to f_T$.} This direction corresponds to $(\star, \diamond) = (S, T)$. Accordingly, the function matching distance  (\eqref{def:function-matching-dist}) becomes
\begin{align}
    \min_{g\in \sG} \  \| f_S - g\circ f_T\|_{\gD, \mH_S}. \label{eq:function-distance-T-S}
\end{align}
Since the affine transformation $g$ acts on the target reference model, it can not be directly viewed as a surrogate transfer loss. However, interesting interpretations can be found in this  direction, depending on the output dimension of $f_S:\sR^n \to \sR^m$ and $f_T:\sR^n \to \sR^d$. 


That is, when the direction of adversarial transfer is from $f_S\to f_T$, the indicating relation between it and knowledge transferability would possibly be unidirectional, depending on the dimensions. More discussion is included in the appendix section~\ref{sec:adv-direction} due to space limitation.

\section{Synthetic Experiments}\label{sec:synthetic-exp}

The synthetic experiment aims to bridge the gap between theory and practice by verifying some of the theoretical insights that may be difficult to compute for large-scale experiments.  Specifically, the synthetic experiment aims to verify: first, how influential are the two proposed adversarial transferability metrics $\alpha_1, \alpha_2$ comparing to the other factors in the generalized adversarial attacks (\eqref{def:generalized-attack}); Second, how does the gradient matching distance track the knowledge transfer loss. The dataset ($N=5000$) is generated by a Gaussian mixture of $10$ Gaussians. The ground truth target is set to be the sum of $100$ radial basis functions. The dimension of $\vx$ is $50$, and the dimension of the target is $10$. Details of the datasets are defer to appendix section~\ref{sec:appendix-synthetic-exp}.

\textbf{Models} Both the source model $f_S$ and target model $f_T$ are one-hidden-layer neural networks with sigmoid activation.


\textbf{Methods} First, sample $D=\{(\vx_i, \vy_i)\}_{i=1}^N$ from the distribution, where $\vx$ is $50$-dimensional, $\vy$ is $10$-dimensional. 
Then we train a target model $f_T$ on $D$.
To derive the source models, we first train a target model on $D$ with width $m = 100$. Denoting the weights of a target model as $\mW$, we randomly sample a direction $\mV$ where each entry of $\mV$ is sampled from $U(-0.5, 0.5)$, and choose a scale $t\in [0, 1]$. Subsequently, we perturb the model weights of the clean source model as $\mW’ := \mW + t\mV$, and define the source model $f_S$ to be a one-hidden-layer neural network with weights $\mW’$. 
Then, we compute each of the quantities we care about, including $\alpha_1$, $\alpha_2$ from both $f_S\to f_T$ and $f_T\to f_S$, the gradient matching distance (\eqref{def:gradient-matching-dist}), and the actual knowledge transfer distance (\eqref{def:know-transfer-dist}). We use the standard $\ell_2$ loss as the adversarial loss function.

\textbf{Results} 
We present two sets of experiment in Figure~\ref{fig:exp-synthetic}. The indication relations between adversarial transferability and knowledge transferability can be observed. Moreover: 1. the metrics $\alpha_1, \alpha_2$ are more meaningful if using the regular attacks $\bm{\delta}_{f_\star}^{(1)}$; 2. the gradient matching distance tracks the actual knowledge transferability loss; 3. the directions of $f_T\to f_S$ and $f_S \to f_T$ are similar.
\begin{figure}[h!]\centering
    \begin{minipage}{0.4\linewidth}\centering
        \includegraphics[width=\linewidth]{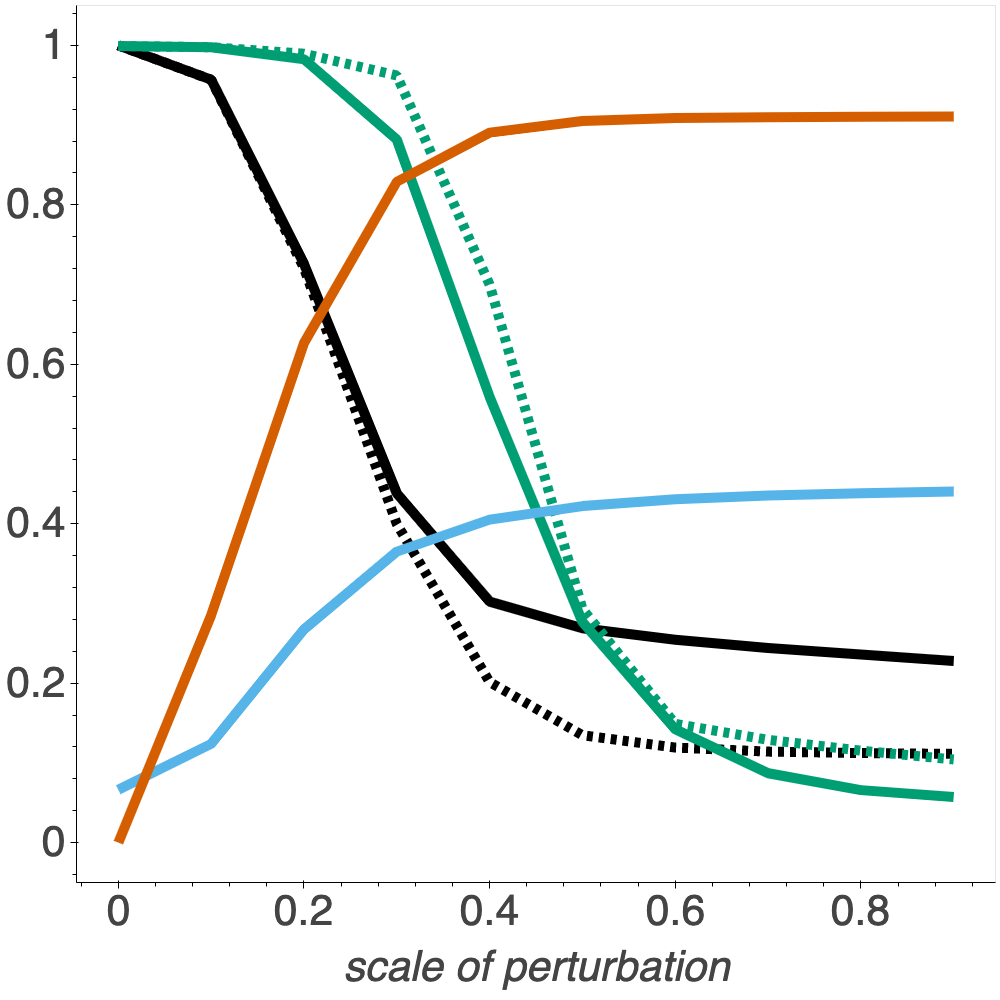}
        {\small (a) $\bm{\delta}_{f_\star}^{(1)}$}
    \end{minipage}
    \begin{minipage}{0.4\linewidth}\centering
        \includegraphics[width=\linewidth]{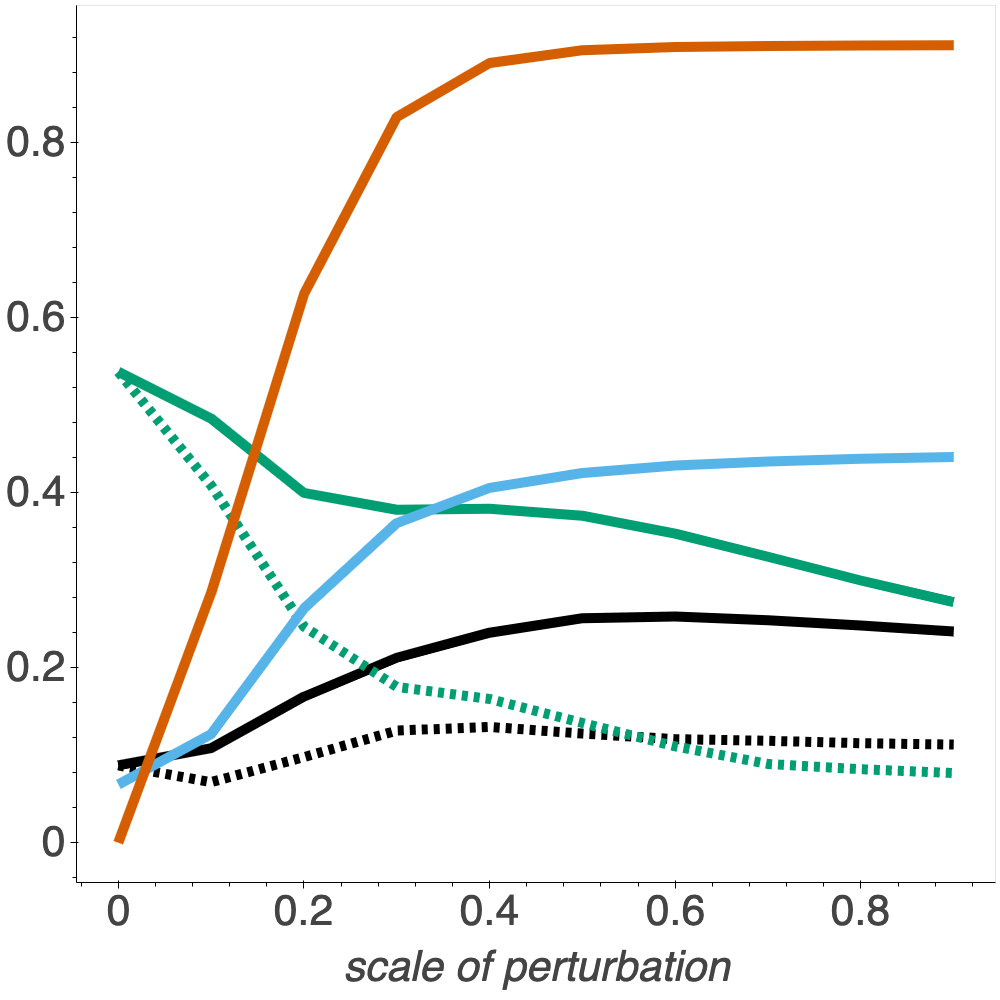}
        {\small (b) $\bm{\delta}_{f_\star}^{(2)}$}
    \end{minipage}
    \vspace{-3mm}
    \caption{ \small As defined in \eqref{def:generalized-attack}, (a) corresponds to the regular adversarial attacks, while (b) the secondary adversarial attack. That is, (b) represents the other information in the adversarial transferring process compared with the first.  The x-axis shows the scale of perturbation $t\in [0,1]$ that controls how much the source model deviates from the target model. There are in total 6 quantities reported. Specifically, $\alpha_1^{f_T\to f_S}$ is \textbf{black solid}; $\alpha_1^{f_S\to f_T}$ is \textbf{black dotted};  $\alpha_2^{f_T\to f_S}$ is \textcolor{green}{\textbf{green solid}}; $\alpha_2^{f_S\to f_T}$ is \textcolor{green}{\textbf{green dotted}}; the gradient matching loss is \textcolor{red}{\textbf{red solid}}; and the knowledge transferability distance is  \textcolor{blue}{\textbf{blue solid}}.
    }
    \vspace{-5mm}
    \label{fig:exp-synthetic}
\end{figure}

\section{Experimental Evaluation}\label{sec:exp}
We present the real-data experiments based on both image and natural language datasets in this section, and discuss the  potential applications. 


\textbf{Adversarial Transferability Indicating Knowledge Transferability. }
In this experiment, we show how to use adversarial transferability to identify the optimal transfer learning candidates from a pool of models trained on the same source dataset.
We first train 5 different architectures (AlexNet, Fully connected network, LeNet,  ResNet18, ResNet50) on cifar10~\cite{krizhevsky2009learning}. Then we perform transfer learning to STL10~\cite{coates2011analysis} to obtain the knowledge transferability of each, measured by accuracy. At the same time, we also train one ResNet18 on STL10 as the target model, which has poor accuracy because of the lack of data. To measure the adversarial transferability, we generate adversarial examples with PGD~\cite{madry2017towards} on the target model and use the generated adversarial examples to attack each source model. The adversarial transferability is expressed in the form of $\alpha_1$ and $\alpha_2$. Our results in Table \ref{table:exp1} indicate that we can use adversarial tarnsferability to forecast knowledge transferability, where the only major computational overheads are training a naive model on the target domain and generating a few adversarial examples. In the end, We further evaluate the significance of our results by  Pearson score. More details about training and generation of adversarial examples can be found in the appendix~\ref{section:appendix_exp}.
\begin{table}[h]
    \centering
    \begin{small}
     \resizebox{0.5\textwidth}{!}{
    \begin{tabular}{ccccc}
     \toprule
     \textbf{Model} &\textbf{Knowledge Trans.}&
     $\alpha_1$ &  
     $\alpha_2$ & 
     \textbf{$\alpha_1 * \alpha_2$}\\
     \midrule
    Fully Connected & 28.30 & 0.346 & 0.189 & 0.0258\\
    \hline
    LeNet & 45.65 & 0.324 & 0.215 & 0.0254\\
    \hline
    AlexNet  & 55.09 & 0.337 & 0.205 & 0.0268\\
    
    \hline
    ResNet18 & 76.60 & 0.538 & 0.244 & 0.0707\\
    \hline
    ResNet50 & 77.92 & 0.614 & 0.234 & 0.0899\\
     \bottomrule
     \end{tabular}
     }
     \end{small}
     \vspace{-3mm}
    \caption{\small Knowledge transferability (Knowledge Trans.) among different model architectures. Our correlation  analysis shows Pearson score of -0.51 between the transfer loss and $\alpha_1$. Lower transfer loss corresponds to higher transfer accuracy. More details can be found in fig \ref{fig:exp1_pgd_img} in the Appendix~\ref{section:appendix_exp}} 
    \vspace{-3mm}
    \label{table:exp1}
\end{table}

To further validate our idea, we also conduct experiments on the NLP domain. We first finetune 5 different BERT classification models on different data domain (IMDB, Moview Review (MR), Yelp, AG, Fake). We refer the models trained on MR, Yelp, AG and Fake datasets as the source models, and take the model trained on IMDB datset as the target model. 
To measure the knowledge transferability, we fine-tune the source models with new linear layers on the target dataset for one epoch. We report the accuracy of the transferred models on the target test set as the metric to indicate the knowledge transferability. 
In terms of the adversarial transferability, we generate adversarial examples by the state-of-the-art whitebox attack algorithm T3 \citep{t3} against the target model and transfer the adversarial examples to source models to evaluate the adversarial transferability. Following our previous experiment, we also calculate $\alpha_1$ and $\alpha_2$. Experimental results are shown in Table \ref{table:nlp_exp1}. We observe that source models with larger adversarial transferability, measured by $\alpha_1$, $\alpha_2$ and $\alpha_1 * \alpha_2$, indeed tend to have larger knowledge transferability. 


\begin{table}[h]
    \centering
    \begin{small}
    \begin{tabular}{ccccc}
     \toprule
     \textbf{Model} &\textbf{Knowledge Trans.}&
     $\alpha_1$ &  
     $\alpha_2$ & 
     $\alpha_1 * \alpha_2$\\
     \midrule
    MR & 89.34 & 0.743 & 0.00335 & 3.00e-3\\
    \hline
    Yelp & 88.81 & 0.562 & 0.00135 & 8.87e-4\\
    \hline
    AG  &  87.58 & 0.295 & 0.00021 & 8.56e-5 \\
    \hline
    Fake &  84.06 & 0.028 & 0.00032 & 5.58e-6\\
     \bottomrule
     \end{tabular}
     \end{small}
     \vspace{-3mm}
    \caption{\small Knowledge transferability (Knowledge Trans.) from the Source Models (MR, Yelp, AG, Fake) to the Target Model (IMDB). Adversarial transferability is measured by using the adversarial examples generated against the Target Model (IMDB) to attack the Source Models and estimate $\alpha_1$ and $\alpha_2$. The correlation analysis shows Pearson Score of $0.27$ between the transfer confidence and $\alpha_1$. Higher transfer confidence indicates higher knowledge transferability. More details can be found in Figure \ref{fig:exp1_nlp_img} in Appendix \S \ref{section:appendix_exp}.}
    \vspace{-3mm}
    \label{table:nlp_exp1}
\end{table}

\begin{table}[h]
    \centering
    \begin{small}
    \vspace{-3mm}
    \begin{tabular}{ccccc}
     \toprule
     \textbf{Similarity} &\textbf{Knowledge Trans.}&
     $\alpha_1$ &  
     $\alpha_2$ & 
     \textbf{$\alpha_1*\alpha_2$}\\
     \midrule
    0\% & 45.00 & 0.310 & 0.146 & 0.0169\\
    \hline
    25\% & 45.68 & 0.318 & 0.305 & 0.0383\\
    \hline
    50\% & 59.09 & 0.338 & 0.355 & 0.0436\\
    \hline
    75\% & 71.62 & 0.337 & 0.312 & 0.0402\\
    \hline
    100\% & 81.84 & 0.358 & 0.357 & 0.0489\\
     \bottomrule
     \end{tabular}
     \end{small}
     \vspace{-3mm}
    \caption{\small Knowledge transferability (Knowledge Trans.) of different source model. Similarity indicates how similar the source distributions are with the target distribution. Our correlation  analysis shows Pearson score of -0.06 between the transfer loss and $\alpha_1$. Lower transfer loss corresponds to higher knowledge transferability. More details can be found in fig \ref{fig:exp2_img} in the Appendix~\ref{section:appendix_exp}.}
    \vspace{-3mm}
    \label{table:exp2}
\end{table}


\begin{table}[h]
    \centering
    \begin{small}
    \begin{tabular}{ccccc}
     \toprule
     \textbf{Model} &\textbf{Knowledge Trans.}&
     $\alpha_1$ &  
     $\alpha_2$ & 
     \textbf{$\alpha_1*\alpha_2$}\\
     \midrule
    MR & 89.34 & 0.584 & 0.00188 & 2.32e-3\\
    \hline
    Yelp & 88.81 & 0.648 & 0.00120 & 9.52e-4\\
    \hline
    AG  &  87.58 & 0.293 & 0.00016 &  4.35e-6 \\
    \hline
    Fake &  84.06 & 0.150 & 0.00073 & 3.55e-5\\
     \bottomrule
     \end{tabular}
     \end{small}
     \vspace{-3mm}
    \caption{\small Knowledge transferability (Knowledge Trans.) from the Source Models (MR, Yelp, AG, Fake) to the Target Model (IMDB). Adversarial transferability is measured by using the adversarial examples generated against the Source Models to attack the Target Models and estimate $\alpha_1$ and $\alpha_2$. The correlation analysis shows Pearson Score of $0.27$ between the transfer confidence and $\alpha_1$. Higher transfer confidence indicates higher knowledge transferability. More details can be found in Figure \ref{fig:exp2_nlp_img} in Appendix \S \ref{section:appendix_exp}. }
    \vspace{-3mm}
    \label{table:nlp_exp2}
\end{table}

\textbf{Knowledge Transferability Indicating Adversarial Transferability.}
In addition, we are interested in the impact of knowledge transferability on adversarial transferability. As predicted by our theory, the more knowledge transferable a source model is to the target domain, the more adversarial transferable it is.

We split cifar10 into 5 different subsets containing different percentages of animals and vehicles. We train a resNet18 on each of them as source models, which are later fine-tuned to obtained the knowledge transferability measured by accuracy. Then we train another resNet18 on a subset of stl10 that only contains vehicles. Different from the last experiment, we generate adversarial examples with PGD on each of the source models and transfer them to the target model. Table \ref{table:exp2} shows, the source model that transfers knowledge better generates more transferable adversarial examples. This implies we can use this relation to facilitate blackbox attack against a hidden target model, given some knowledge about the source and target domains. More details of training and generation of adversarial examples can be found in the appendix.

We evaluate the impact of knowledge transferability to adversarial transferability in the NLP domain as well. We mostly follow the setting describe in the previous section, where we have four source models and one target model, and the knowledge transferability from source models to the target model is measured by the accuracy of the transferred models on the target test set. The difference lies on the evaluation of the adversarial transferability, where we generate adversarial examples against the source models and evaluate their attack capability on the target model. As shown in Table \ref{table:nlp_exp2}, we note that when the source data domain is getting closer to the target data domain, the knowledge transferability grows, and the adversarial transferability also increases. More experimental details can be found in Appendix~\ref{section:appendix_exp}.

\textbf{Ablation Studies}
Following the settings in table \ref{table:exp1}, we conduct ablation studies (table \ref{tab:ablation}) on two additional attack methods, MI~\cite{tramer2017ensemble}, PGD-L2 and two additional $\epsilon$ with PGD, 2/225, 4/255, we discover that neither the attack method nor $\epsilon$ has significant impact on our conclusion.

\begin{table}[htbp!]
    \centering
     
     \begin{minipage}[b]{\hsize}\centering
     \resizebox{1.0\textwidth}{!}{
    \begin{tabular}{ccccc}
     \toprule
     \textbf{Model} &\textbf{Knowledge Trans.}&
     $\alpha_1$ &  
     $\alpha_2$ & 
     \textbf{$\alpha_1 * \alpha_2$}\\
     \midrule
     Fully Connected & 28.30 & 0.0985 & 0.0196 & 0.00027\\
    \hline
    LeNet & 45.65 & 0.2106 & 0.0259 & 0.00158\\
    \hline
    AlexNet  & 55.09 & 0.1196 & 0.0206 & 0.00037\\
    
    \hline
    ResNet18 & 76.60 & 0.2739 & 0.0413 & 0.00405\\
    \hline
    ResNet50 & 77.92 & 0.1952 & 0.0320 & 0.00172\\
     \bottomrule
     \end{tabular}
     }
     
     {\small $\epsilon=2/255$. Pearson score is -0.45.} 
     
     \end{minipage}
     
     \begin{minipage}[b]{\hsize}\centering
     \resizebox{1.0\textwidth}{!}{
    \begin{tabular}{ccccc}
     \toprule
     \textbf{Model} &\textbf{Knowledge Trans.}&
     $\alpha_1$ &  
     $\alpha_2$ & 
     \textbf{$\alpha_1 * \alpha_2$}\\
     \midrule
     Fully Connected & 28.30 & 0.0974 & 0.0225 & 0.00029\\
    \hline
    LeNet & 45.65 & 0.2099 & 0.0309 & 0.00192\\
    \hline
    AlexNet  & 55.09 & 0.1283 & 0.0230 & 0.00048\\
    
    \hline
    ResNet18 & 76.60 & 0.2853 & 0.0481 & 0.00496\\
    \hline
    ResNet50 & 77.92 & 0.2495 & 0.0414 & 0.00337\\
     \bottomrule
     \end{tabular}
    }
    
     {\small $\epsilon=4/255$. Pearson score is -0.49.} 
     \end{minipage}
     
     \begin{minipage}[b]{\hsize}\centering
     \resizebox{1.0\textwidth}{!}{
     \begin{tabular}{ccccc}
     \toprule
     \textbf{Model} &\textbf{Knowledge Trans.}&
     $\alpha_1$ &  
     $\alpha_2$ & 
     \textbf{$\alpha_1 * \alpha_2$}\\
     \midrule
     Fully Connected & 28.30 & 0.1678 & 0.0379 & 0.0013\\
    \hline
    LeNet & 45.65 & 0.0997 & 0.0503 & 0.0005\\
    \hline
    AlexNet  & 55.09 & 0.1229 & 0.0506 & 0.0009\\
    
    \hline
    ResNet18 & 76.60 & 0.2731 & 0.0630 & 0.0052\\
    \hline
    ResNet50 & 77.92 & 0.3695 & 0.0550 & 0.0081\\
     \bottomrule
     \end{tabular}
     }
     
     {\small Attack with MI. Pearson score is -0.45.} 
    \vspace{1.5mm}
     \end{minipage}
     \begin{minipage}[b]{\hsize}\centering
     \resizebox{1.0\textwidth}{!}{
     \begin{tabular}{ccccc}
     \toprule
     \textbf{Model} &\textbf{Knowledge Trans.}&
     $\alpha_1$ &  
     $\alpha_2$ & 
     \textbf{$\alpha_1 * \alpha_2$}\\
     \midrule
     Fully Connected & 28.30 & 0.0809 & 0.0175 & 0.00018\\
    \hline
    LeNet & 45.65 & 0.2430 & 0.0190 & 0.00149\\
    \hline
    AlexNet  & 55.09 & 0.1101 & 0.0188 & 0.00031\\
    
    \hline
    ResNet18 & 76.60 & 0.3619 & 0.0303 &  0.00464\\
    \hline
    ResNet50 & 77.92 & 0.2506 & 0.0237 & 0.00179\\
     \bottomrule
     \end{tabular}
     }
     
     {\small $\ell_2$ attack with $\epsilon=1$. Pearson score is -0.40.} 
     \end{minipage}
    \label{tab:ablation}
     \vspace{-7mm}
    \caption{\small With varying attack methods and $\epsilon$, adversarial transferability is still correlated with knowledge transferability.} 
     \vspace{-2mm}
\end{table}
\section{Conclusion}
We theoretically analyze the relation between adversarial transferability and knowledge transferability. We provide empirical experimental justifications in pratical settings. Both our theoretical and empirical results show that adversarial transferability can indicate knowledge transferability and vice versa. We expect our work will inspire future work  on further exploring other factors that impact knowledge transferability and adversarial transferability.
\paragraph{Acknowledgments}
This work is partially supported by NSF IIS 1909577, NSF CCF 1934986, NSF CCF 1910100, NSF CNS 20-46726 CAR, Amazon Research Award, and the Intel RSA 2020.
\clearpage
\bibliography{ref}

\begin{thebibliography}{36}
\providecommand{\natexlab}[1]{#1}
\providecommand{\url}[1]{\texttt{#1}}
\expandafter\ifx\csname urlstyle\endcsname\relax
  \providecommand{\doi}[1]{doi: #1}\else
  \providecommand{\doi}{doi: \begingroup \urlstyle{rm}\Url}\fi

\bibitem[Achille et~al.(2019)Achille, Lam, Tewari, Ravichandran, Maji, Fowlkes,
  Soatto, and Perona]{achille2019task2vec}
Achille, A., Lam, M., Tewari, R., Ravichandran, A., Maji, S., Fowlkes, C.~C.,
  Soatto, S., and Perona, P.
\newblock Task2vec: Task embedding for meta-learning.
\newblock In \emph{Proceedings of the IEEE International Conference on Computer
  Vision}, pp.\  6430--6439, 2019.

\bibitem[Athalye et~al.(2018)Athalye, Carlini, and
  Wagner]{athalye2018obfuscated}
Athalye, A., Carlini, N., and Wagner, D.
\newblock Obfuscated gradients give a false sense of security: Circumventing
  defenses to adversarial examples.
\newblock In \emph{International Conference on Machine Learning}, pp.\
  274--283, 2018.

\bibitem[Coates et~al.(2011)Coates, Ng, and Lee]{coates2011analysis}
Coates, A., Ng, A., and Lee, H.
\newblock An analysis of single-layer networks in unsupervised feature
  learning.
\newblock In \emph{Proceedings of the fourteenth international conference on
  artificial intelligence and statistics}, pp.\  215--223, 2011.

\bibitem[Demontis et~al.(2019)Demontis, Melis, Pintor, Jagielski, Biggio,
  Oprea, Nita-Rotaru, and Roli]{demontis2019adversarial}
Demontis, A., Melis, M., Pintor, M., Jagielski, M., Biggio, B., Oprea, A.,
  Nita-Rotaru, C., and Roli, F.
\newblock Why do adversarial attacks transfer? explaining transferability of
  evasion and poisoning attacks.
\newblock In \emph{28th $\{$USENIX$\}$ Security Symposium ($\{$USENIX$\}$
  Security 19)}, pp.\  321--338, 2019.

\bibitem[Dong et~al.(2015)Dong, Wu, He, Yu, and Wang]{dong2015multi}
Dong, D., Wu, H., He, W., Yu, D., and Wang, H.
\newblock Multi-task learning for multiple language translation.
\newblock In \emph{Proceedings of the 53rd Annual Meeting of the Association
  for Computational Linguistics and the 7th International Joint Conference on
  Natural Language Processing (Volume 1: Long Papers)}, pp.\  1723--1732, 2015.

\bibitem[Dong et~al.(2019)Dong, Pang, Su, and Zhu]{dong2019evading}
Dong, Y., Pang, T., Su, H., and Zhu, J.
\newblock Evading defenses to transferable adversarial examples by
  translation-invariant attacks.
\newblock In \emph{Proceedings of the IEEE Conference on Computer Vision and
  Pattern Recognition}, pp.\  4312--4321, 2019.

\bibitem[Goodfellow et~al.(2014)Goodfellow, Shlens, and
  Szegedy]{goodfellow2014explaining}
Goodfellow, I.~J., Shlens, J., and Szegedy, C.
\newblock Explaining and harnessing adversarial examples.
\newblock \emph{arXiv preprint arXiv:1412.6572}, 2014.

\bibitem[Huh et~al.(2016)Huh, Agrawal, and Efros]{huh2016makes}
Huh, M., Agrawal, P., and Efros, A.~A.
\newblock What makes imagenet good for transfer learning?
\newblock \emph{arXiv preprint arXiv:1608.08614}, 2016.

\bibitem[Ilyas et~al.(2018)Ilyas, Engstrom, Athalye, and Lin]{ilyas2018black}
Ilyas, A., Engstrom, L., Athalye, A., and Lin, J.
\newblock Black-box adversarial attacks with limited queries and information.
\newblock In \emph{International Conference on Machine Learning}, pp.\
  2137--2146, 2018.

\bibitem[Joon~Oh et~al.(2017)Joon~Oh, Fritz, and Schiele]{joon2017adversarial}
Joon~Oh, S., Fritz, M., and Schiele, B.
\newblock Adversarial image perturbation for privacy protection--a game theory
  perspective.
\newblock In \emph{Proceedings of the IEEE International Conference on Computer
  Vision}, pp.\  1482--1491, 2017.

\bibitem[Kariyappa \& Qureshi(2019)Kariyappa and
  Qureshi]{kariyappa2019improving}
Kariyappa, S. and Qureshi, M.~K.
\newblock Improving adversarial robustness of ensembles with diversity
  training.
\newblock \emph{arXiv preprint arXiv:1901.09981}, 2019.

\bibitem[Kendall et~al.(2018)Kendall, Gal, and Cipolla]{kendall2018multi}
Kendall, A., Gal, Y., and Cipolla, R.
\newblock Multi-task learning using uncertainty to weigh losses for scene
  geometry and semantics.
\newblock In \emph{Proceedings of the IEEE conference on computer vision and
  pattern recognition}, pp.\  7482--7491, 2018.

\bibitem[Krizhevsky et~al.(2009)Krizhevsky, Hinton,
  et~al.]{krizhevsky2009learning}
Krizhevsky, A., Hinton, G., et~al.
\newblock Learning multiple layers of features from tiny images.
\newblock 2009.

\bibitem[Liu et~al.(2016)Liu, Chen, Liu, and Song]{liu2016delving}
Liu, Y., Chen, X., Liu, C., and Song, D.
\newblock Delving into transferable adversarial examples and black-box attacks.
\newblock \emph{arXiv preprint arXiv:1611.02770}, 2016.

\bibitem[Long et~al.(2015)Long, Cao, Wang, and Jordan]{long2015learning}
Long, M., Cao, Y., Wang, J., and Jordan, M.
\newblock Learning transferable features with deep adaptation networks.
\newblock In \emph{International Conference on Machine Learning}, pp.\
  97--105, 2015.

\bibitem[Ma et~al.(2018)Ma, Li, Wang, Erfani, Wijewickrema, Schoenebeck, Song,
  Houle, and Bailey]{ma2018characterizing}
Ma, X., Li, B., Wang, Y., Erfani, S.~M., Wijewickrema, S., Schoenebeck, G.,
  Song, D., Houle, M.~E., and Bailey, J.
\newblock Characterizing adversarial subspaces using local intrinsic
  dimensionality.
\newblock \emph{arXiv preprint arXiv:1801.02613}, 2018.

\bibitem[Madry et~al.(2017)Madry, Makelov, Schmidt, Tsipras, and
  Vladu]{madry2017towards}
Madry, A., Makelov, A., Schmidt, L., Tsipras, D., and Vladu, A.
\newblock Towards deep learning models resistant to adversarial attacks.
\newblock \emph{arXiv preprint arXiv:1706.06083}, 2017.

\bibitem[Miyato et~al.(2018)Miyato, Maeda, Koyama, and
  Ishii]{miyato2018virtual}
Miyato, T., Maeda, S.-i., Koyama, M., and Ishii, S.
\newblock Virtual adversarial training: a regularization method for supervised
  and semi-supervised learning.
\newblock \emph{IEEE transactions on pattern analysis and machine
  intelligence}, 41\penalty0 (8):\penalty0 1979--1993, 2018.

\bibitem[Naseer et~al.(2019)Naseer, Khan, Khan, Khan, and
  Porikli]{naseer2019cross}
Naseer, M.~M., Khan, S.~H., Khan, M.~H., Khan, F.~S., and Porikli, F.
\newblock Cross-domain transferability of adversarial perturbations.
\newblock In \emph{Advances in Neural Information Processing Systems}, pp.\
  12885--12895, 2019.

\bibitem[Papernot et~al.(2016)Papernot, McDaniel, and
  Goodfellow]{papernot2016transferability}
Papernot, N., McDaniel, P., and Goodfellow, I.
\newblock Transferability in machine learning: from phenomena to black-box
  attacks using adversarial samples.
\newblock \emph{arXiv preprint arXiv:1605.07277}, 2016.

\bibitem[Russakovsky et~al.(2015)Russakovsky, Deng, Su, Krause, Satheesh, Ma,
  Huang, Karpathy, Khosla, Bernstein, et~al.]{russakovsky2015imagenet}
Russakovsky, O., Deng, J., Su, H., Krause, J., Satheesh, S., Ma, S., Huang, Z.,
  Karpathy, A., Khosla, A., Bernstein, M., et~al.
\newblock Imagenet large scale visual recognition challenge.
\newblock \emph{International journal of computer vision}, 115\penalty0
  (3):\penalty0 211--252, 2015.

\bibitem[Salman et~al.(2020)Salman, Ilyas, Engstrom, Kapoor, and
  Madry]{salman2020adversarially}
Salman, H., Ilyas, A., Engstrom, L., Kapoor, A., and Madry, A.
\newblock Do adversarially robust imagenet models transfer better?
\newblock \emph{arXiv preprint arXiv:2007.08489}, 2020.

\bibitem[Shinya et~al.(2019)Shinya, Simo-Serra, and
  Suzuki]{shinya2019understanding}
Shinya, Y., Simo-Serra, E., and Suzuki, T.
\newblock Understanding the effects of pre-training for object detectors via
  eigenspectrum.
\newblock In \emph{Proceedings of the IEEE International Conference on Computer
  Vision Workshops}, pp.\  0--0, 2019.

\bibitem[Tram{\`e}r et~al.(2017{\natexlab{a}})Tram{\`e}r, Kurakin, Papernot,
  Goodfellow, Boneh, and McDaniel]{tramer2017ensemble}
Tram{\`e}r, F., Kurakin, A., Papernot, N., Goodfellow, I., Boneh, D., and
  McDaniel, P.
\newblock Ensemble adversarial training: Attacks and defenses.
\newblock \emph{arXiv preprint arXiv:1705.07204}, 2017{\natexlab{a}}.

\bibitem[Tram{\`e}r et~al.(2017{\natexlab{b}})Tram{\`e}r, Papernot, Goodfellow,
  Boneh, and McDaniel]{tramer2017space}
Tram{\`e}r, F., Papernot, N., Goodfellow, I., Boneh, D., and McDaniel, P.
\newblock The space of transferable adversarial examples.
\newblock \emph{arXiv preprint arXiv:1704.03453}, 2017{\natexlab{b}}.

\bibitem[Utrera et~al.(2020)Utrera, Kravitz, Erichson, Khanna, and
  Mahoney]{utrera2020adversarially}
Utrera, F., Kravitz, E., Erichson, N.~B., Khanna, R., and Mahoney, M.~W.
\newblock Adversarially-trained deep nets transfer better.
\newblock \emph{arXiv preprint arXiv:2007.05869}, 2020.

\bibitem[Wang et~al.(2019{\natexlab{a}})Wang, Singh, Michael, Hill, Levy, and
  Bowman]{wang2018glue}
Wang, A., Singh, A., Michael, J., Hill, F., Levy, O., and Bowman, S.~R.
\newblock {GLUE}: A multi-task benchmark and analysis platform for natural
  language understanding.
\newblock In \emph{International Conference on Learning Representations},
  2019{\natexlab{a}}.
\newblock URL \url{https://openreview.net/forum?id=rJ4km2R5t7}.

\bibitem[Wang et~al.(2020)Wang, Pei, Pan, Chen, Wang, and Li]{t3}
Wang, B., Pei, H., Pan, B., Chen, Q., Wang, S., and Li, B.
\newblock T3: Tree-autoencoder constrained adversarial text generation for
  targeted attack.
\newblock In \emph{Proceedings of the 2020 Conference on Empirical Methods in
  Natural Language Processing (EMNLP)}, pp.\  6134--6150, Online, November
  2020. Association for Computational Linguistics.
\newblock \doi{10.18653/v1/2020.emnlp-main.495}.
\newblock URL \url{https://www.aclweb.org/anthology/2020.emnlp-main.495}.

\bibitem[Wang et~al.(2019{\natexlab{b}})Wang, Dai, P{\'o}czos, and
  Carbonell]{wang2019characterizing}
Wang, Z., Dai, Z., P{\'o}czos, B., and Carbonell, J.
\newblock Characterizing and avoiding negative transfer.
\newblock In \emph{Proceedings of the IEEE Conference on Computer Vision and
  Pattern Recognition}, pp.\  11293--11302, 2019{\natexlab{b}}.

\bibitem[Xie et~al.(2019)Xie, Zhang, Zhou, Bai, Wang, Ren, and
  Yuille]{xie2019improving}
Xie, C., Zhang, Z., Zhou, Y., Bai, S., Wang, J., Ren, Z., and Yuille, A.~L.
\newblock Improving transferability of adversarial examples with input
  diversity.
\newblock In \emph{Proceedings of the IEEE Conference on Computer Vision and
  Pattern Recognition}, pp.\  2730--2739, 2019.

\bibitem[Xu et~al.(2019)Xu, Li, Yang, and Lin]{xu2019larger}
Xu, R., Li, G., Yang, J., and Lin, L.
\newblock Larger norm more transferable: An adaptive feature norm approach for
  unsupervised domain adaptation.
\newblock In \emph{Proceedings of the IEEE International Conference on Computer
  Vision}, pp.\  1426--1435, 2019.

\bibitem[Yosinski et~al.(2014)Yosinski, Clune, Bengio, and
  Lipson]{yosinski2014transferable}
Yosinski, J., Clune, J., Bengio, Y., and Lipson, H.
\newblock How transferable are features in deep neural networks?
\newblock In \emph{Advances in neural information processing systems}, pp.\
  3320--3328, 2014.

\bibitem[Zamir et~al.(2018)Zamir, Sax, Shen, Guibas, Malik, and
  Savarese]{zamir2018taskonomy}
Zamir, A.~R., Sax, A., Shen, W., Guibas, L.~J., Malik, J., and Savarese, S.
\newblock Taskonomy: Disentangling task transfer learning.
\newblock In \emph{Proceedings of the IEEE Conference on Computer Vision and
  Pattern Recognition}, pp.\  3712--3722, 2018.

\bibitem[Zhang et~al.(2015)Zhang, Zhao, and LeCun]{zhang2015character}
Zhang, X., Zhao, J., and LeCun, Y.
\newblock Character-level convolutional networks for text classification.
\newblock \emph{arXiv preprint arXiv:1509.01626}, 2015.

\bibitem[Zhang et~al.(2014)Zhang, Luo, Loy, and Tang]{zhang2014facial}
Zhang, Z., Luo, P., Loy, C.~C., and Tang, X.
\newblock Facial landmark detection by deep multi-task learning.
\newblock In \emph{European conference on computer vision}, pp.\  94--108.
  Springer, 2014.

\bibitem[Zhou et~al.(2018)Zhou, Hou, Chen, Tang, Huang, Gan, and
  Yang]{zhou2018transferable}
Zhou, W., Hou, X., Chen, Y., Tang, M., Huang, X., Gan, X., and Yang, Y.
\newblock Transferable adversarial perturbations.
\newblock In \emph{Proceedings of the European Conference on Computer Vision
  (ECCV)}, pp.\  452--467, 2018.

\end{thebibliography}
\bibliographystyle{icml2021}

\clearpage
\onecolumn
\appendix

\icmltitle{Supplementary Material: \\
Uncovering the Connections Between\\
Adversarial Transferability and Knowledge Transferability}

{\centering\textbf{Contents Summary}
\begin{itemize}[leftmargin=5em,rightmargin=5em, itemsep=-0.5mm]
\vspace{-2mm}
    \item Section~\ref{sec:example-alpha}: An Example Illustrating the Necessity of both $\alpha_1,\alpha_2$ in Characterizing the Relation Between Adversarial Transferability and Knowledge Transferability.
    \item Section~\ref{sec:adv-direction}: Detailed discussion about the direction of adversarial transfer from $f_S\to f_T$ in subsection~\ref{subsec:theory-conclusion}.
    \item Section~\ref{sec:proof-1}: Proofs of the propositions in section~\ref{sec:preliminaries}.
    \begin{itemize}
        \item \ref{subsec:prop-1}: Proof of Proposition~\ref{prop:1}
        \item \ref{subsec:prop-2}: Proof of Proposition~\ref{prop:2}
    \end{itemize}
    \item Section~\ref{sec:proof-2}: Proofs of the theorems and propositions in section~\ref{sec:theory}.
    \begin{itemize}
        \item \ref{subsec:proof-theorem-1}: Proof of Theorem~\ref{theorem-1}
        \item \ref{subsec:proof-prop-3}: Proof of Proposition~\ref{prop:3}
        \item \ref{subsec:proof-theorem-2}: Proof of Theorem~\ref{theorem-2}
        \item \ref{subsec:proof-prop-4}: Proof of Theorem~\ref{prop:4}
        \item \ref{subsec:proof-prop-5}: Proof of Theorem~\ref{prop:5}
    \end{itemize}
    \item Section~\ref{sec:aux-lemmas}: Auxiliary lemmas. 
    \item Section~\ref{sec:appendix-synthetic-exp}: Details and additional results of the synthetic experiments. 
    \item Section~\ref{section:appendix_exp}: Details of model training and adversarial examples generations in the experiments section, and ablation study on controlling the adversarial transferability. 
    \vspace{1em}
\end{itemize}}

\section{An Example Illustrating the Necessity of both $\alpha_1,\alpha_2$ in Characterizing the Relation Between Adversarial Transferability and Knowledge Transferability}\label{sec:example-alpha}

$\alpha_1$ and $\alpha_2$ (Definition~\ref{def:alpha_1}\&\ref{def:alpha_2}) represent complementary aspects of the adversarial transferability: 
$\alpha_1$ can be understood as how often the adversarial attack transfers, while $\alpha_2$ encodes directional information of the output deviation caused by adversarial attacks. Recall that $\alpha_1, \alpha_2 \in [0,1]$ (higher values indicate better adversarial transferability). As we show in our theoretical results reveal that high $\alpha_1$ alone is not enough, \textit{i.e.}, both the proposed metrics are necessary to characterize adversarial transferability and the relation between adversarial and knowledge transferabilities. 

We provide a one-dimensional example showing that large $\alpha_1$ only is not enough to indicate high knowledge transferability. Suppose the ground truth target function $f_T(x) = x^2$, and the source function $f_S(x)=\texttt{sgn}(x)\cdot x^2$ where $\texttt{sgn}(\cdot)$ denotes the sign function. Let the adversarial loss be the deviation in function output, and the data distribution be the uniform distribution on $[-1, 1]$. As we can see,  the direction that makes either $f_T$ or $f_S$ deviates the most is always the same, \textit{i.e.}, 
{in this example even with $\alpha_1=1$  achieves its maximum and adversarial attacks always transfer, regardless of the choice of $f_1\to f_2$ or $f_2\to f_1$. However, 
there does not exist an affine function $g$ (\textit{i.e.}, fine-tuning) making $g\circ f_S$ close to $f_T$ on $[-1, 1]$}. Indeed, one can verify that $\alpha_2=0$ in this case (either $f_1\to f_2$ or $f_2\to f_1$), which contributes to the low knowledge transferability.
However, if we move the data distribution to $[0, 2]$, we can have $\alpha_1=\alpha_2=1$ (either $f_1\to f_2$ or $f_2\to f_1$) indicating high adversarial transferability, and indeed it achieves $f_S=f_T$ showing perfect knowledge transferability.

\section{Detailed Discussion About the Direction of Adversarial Transfer From $f_S\to f_T$ in Subsection~\ref{subsec:theory-conclusion}}\label{sec:adv-direction}

In this section, we present a detailed discussion, in addition to  subsection~\ref{subsec:theory-conclusion}, about the connection between function matching distance and knowledge transfer distance when the direction of adversarial transfer is from $f_S\to f_T$.

Recall that, to complete the story, it remains to connect the function matching distance to knowledge transferability. As the adversarial transfer is symmetric (\emph{i.e.}, either from $f_S\to f_T$ or $f_T\to f_S$), we are able to use the placeholders $\star, \diamond \in \{S, T\}$ all the way through. However, as the knowledge transfer is asymmetric (\emph{i.e.}, $f_S\to y$ to the target ground truth), we need to instantiate the direction of adversarial transfer to further our discussion. We have discussed the direction of adversarial transfer from $f_T\to f_S$ in the main paper, where we show that  the function matching distance of this direction, \emph{i.e.}, 
\begin{align}
    \min_{g\in \sG} \  \| f_T - g\circ f_S\|_{\gD, \mH_T},   \tag{\ref{eq:function-distance-S-T}} 
\end{align}
can both upper and lower bound the knowledge transfer distance, \emph{i.e.},
\begin{align}
    \min_{g\in \sG} \  \| y - g\circ f_S\|_{\gD, \mH_T}.   \tag{\ref{def:know-transfer-dist}} 
\end{align}

The direction of adversarial transfer from $f_S\to f_T$ corresponds to $(\star, \diamond) = (S, T)$. Accordingly, the function matching distance  (\eqref{def:function-matching-dist}) becomes
\begin{align}
    \min_{g\in \sG} \  \| f_S - g\circ f_T\|_{\gD, \mH_S}. \tag{\ref{eq:function-distance-T-S}}
\end{align}
Since the affine transformation $g$ acts on the target reference model, it can not be directly viewed as a surrogate transfer loss. However, interesting interpretations can be found in this  direction, depending on the output dimension of $f_S:\sR^n \to \sR^m$ and $f_T:\sR^n \to \sR^d$. 

In this subsection in the appendix we provide detailed discussion on the connection between the function matching distance of the direction of adversarial transfer from $f_S\to f_T$ (\eqref{eq:function-distance-T-S})
and the knowledge transfer distance (\eqref{def:know-transfer-dist}). We build this connection by providing the relationships between the two directions of function matching distance, \emph{i.e.}, \eqref{eq:function-distance-S-T} and \eqref{eq:function-distance-T-S}. That is being said, since we know \eqref{def:know-transfer-dist} and \eqref{eq:function-distance-S-T} are tied together, we only need to provide relationships between  \eqref{eq:function-distance-S-T} and  \eqref{eq:function-distance-T-S} to show the connection between \eqref{eq:function-distance-T-S} and \eqref{def:know-transfer-dist}.


Suppose $g:\sR^d \to \sR^m$ is full rank, and loosely speaking we can derive the following intuitions. 
\begin{itemize}
    \item If $d<m$, then $g$ is injective and there exists $g^{-1}:\sR^m \to \sR^d$ such that $g^{-1}\circ g$ is the identity function. That is, if $g$ can map $f_T$ to closely track $f_S$, then reversely $g^{-1}$ can map $f_S$ to $f_T$, showing \eqref{eq:function-distance-T-S} upper bounds \eqref{eq:function-distance-S-T} in some sense.
    \item If $d>m$, then $g$ is surjective. By symmetry, \eqref{eq:function-distance-S-T} upper bounds \eqref{eq:function-distance-T-S} in some sense.
    \item It is when $m=d$ that \eqref{eq:function-distance-S-T} and \eqref{eq:function-distance-T-S} coincide. 
\end{itemize}

Formally, we have the following theorem.
\begin{theorem}\label{prop:5}
    Denote $\Tilde{g}_{T,S}:\sR^m\to \sR^d$ as the optimal solution of \eqref{eq:function-distance-S-T}, and $\Tilde{g}_{S,T}:\sR^d\to \sR^m$ as the optimal solution of \eqref{eq:function-distance-T-S}. Suppose the two optimal affine maps $\Tilde{g}_{T,S}, \Tilde{g}_{S,T}$ are both full-rank. For $\vv \in \sR^m$, denote the matrix representation of $\Tilde{g}_{T,S}$ as $\Tilde{g}_{T,S}(\vv)=\Tilde{\mW}_{T,S}\vv + \Tilde{\vb}_{T,S}$. Similarly, for $\vw \in \sR^d$, denote the matrix representation of $\Tilde{g}_{S,T}$ as $\Tilde{g}_{S,T}(\vw)=\Tilde{\mW}_{S,T}\vw + \Tilde{\vb}_{S,T}$. We have the following statements.
    
    If $d<m$, then $\Tilde{g}_{S,T}$ is injective, and we have:
    \begin{align}
        \|f_T - \Tilde{g}_{T,S}\circ f_S\|_{\gD, \mH_T} \leq \sqrt{\|(\Tilde{\mW}_{S,T}^\top\Tilde{\mW}_{S,T})^{-1}\|_F\cdot \|\mH_T\|_F} \cdot \|f_S - \Tilde{g}_{S,T}\circ f_T\|_{\gD}. \label{eq:prop-5-1}
    \end{align}
    If $d>m$, then $\Tilde{g}_{T,S}$ is injective, and we have:
    \begin{align}
        \|f_S - \Tilde{g}_{S,T}\circ f_T\|_{\gD, \mH_S} \leq \sqrt{\|(\Tilde{\mW}_{T,S}^\top\Tilde{\mW}_{T,S})^{-1}\|_F\cdot \|\mH_S\|_F} \cdot \|f_T - \Tilde{g}_{T,S}\circ f_S\|_{\gD}. \label{eq:prop-5-2}
    \end{align}
    If $d=m$, then both $\Tilde{g}_{S,T}$ and $\Tilde{g}_{T,S}$ are bijective, and we have both (\ref{eq:prop-5-1}) and (\ref{eq:prop-5-2}) stand.
\end{theorem}

That is, when the direction of adversarial transfer is from $f_S\to f_T$, the indicating relation between the function matching distance if this direction (\eqref{eq:function-distance-T-S}) and knowledge transferability would possibly be unidirectional, depending on the dimensions.

\section{Proofs in Section~\ref{sec:preliminaries}}\label{sec:proof-1}

In this section, we present proofs for Proposition~\ref{prop:1} and Proposition~\ref{prop:2}.

\subsection{Proof of Proposition~\ref{prop:1}}\label{subsec:prop-1}

\begin{proposition}[Proposition~\ref{prop:1} Restated]
The $\alpha_2^{f_1\to f_2}$ can be reformulated as
\begin{align}
    (\alpha_2^{f_1\to f_2})^2   = \E_{\vx_1, \vx_2} \left[  \theta_{f_1\to f_1}(\vx_1, \vx_2)\theta_{f_1\to f_2}(\vx_1, \vx_2)\right], \label{eq:prop-1-1}
\end{align}
where $\vx_1, \vx_2\overset{\text{i.i.d.}}{\sim}\gD$, and 
\begin{align}
    \theta_{f_1\to f_1}(\vx_1, \vx_2) &=\langle \widehat{\Delta_{f_1\to f_1}(\vx_1)}, \widehat{\Delta_{f_1\to f_1}(\vx_2)} \rangle \\
    \theta_{f_1\to f_2}(\vx_1, \vx_2) &=\langle \widehat{\Delta_{f_1\to f_2}(\vx_1)}, \widehat{\Delta_{f_1\to f_2}(\vx_2)} \rangle 
\end{align}
\end{proposition}
\begin{proof}
    Recall that we want to show
    \begin{align}
        \big\|\E_{\vx} [ \widehat{\Delta_{f_1\to f_1}(\vx)} \widehat{\Delta_{f_1\to f_2}(\vx)}^\top ] \big\|_F^2 = (\alpha_2^{f_1\to f_2})^2   = \E_{\vx_1, \vx_2} \left[  \theta_{f_1\to f_1}(\vx_1, \vx_2)\theta_{f_1\to f_2}(\vx_1, \vx_2)\right],
    \end{align}
    and the proof of this proposition is done by applying some trace tricks, as shown below.
    \begin{align}
         \theta_{f_1\to f_1}(\vx_1, \vx_2)\theta_{f_1\to f_2}(\vx_1, \vx_2)&=\langle \widehat{\Delta_{f_1\to f_1}(\vx_1)}, \widehat{\Delta_{f_1\to f_1}(\vx_2)} \rangle\cdot \langle \widehat{\Delta_{f_1\to f_2}(\vx_1)}, \widehat{\Delta_{f_1\to f_2}(\vx_2)} \rangle \\
         &=\langle \widehat{\Delta_{f_1\to f_1}(\vx_2)} , \widehat{\Delta_{f_1\to f_1}(\vx_1)} \rangle\cdot \langle \widehat{\Delta_{f_1\to f_2}(\vx_1)}, \widehat{\Delta_{f_1\to f_2}(\vx_2)} \rangle \\
         &= \widehat{\Delta_{f_1\to f_1}(\vx_2)}^\top \widehat{\Delta_{f_1\to f_1}(\vx_1)} \widehat{\Delta_{f_1\to f_2}(\vx_1)}^\top \widehat{\Delta_{f_1\to f_2}(\vx_2)}  \\
         &= \tr\left( \widehat{\Delta_{f_1\to f_1}(\vx_2)}^\top \widehat{\Delta_{f_1\to f_1}(\vx_1)} \widehat{\Delta_{f_1\to f_2}(\vx_1)}^\top \widehat{\Delta_{f_1\to f_2}(\vx_2)}\right)  \\
         &= \tr\left(\widehat{\Delta_{f_1\to f_1}(\vx_1)} \widehat{\Delta_{f_1\to f_2}(\vx_1)}^\top \widehat{\Delta_{f_1\to f_2}(\vx_2)} \widehat{\Delta_{f_1\to f_1}(\vx_2)}^\top \right) \label{eq:prop-1-2}
    \end{align}
    Plugging \eqref{eq:prop-1-2} into \eqref{eq:1-1-1}, we have
    \begin{align}
        (\alpha_2^{f_1\to f_2})^2   &= \E_{\vx_1, \vx_2} \left[ \tr\left(\widehat{\Delta_{f_1\to f_1}(\vx_1)} \widehat{\Delta_{f_1\to f_2}(\vx_1)}^\top \widehat{\Delta_{f_1\to f_2}(\vx_2)} \widehat{\Delta_{f_1\to f_1}(\vx_2)}^\top \right)\right]\\
        &=  \tr\left(\E_{\vx_1, \vx_2} \left[ \widehat{\Delta_{f_1\to f_1}(\vx_1)} \widehat{\Delta_{f_1\to f_2}(\vx_1)}^\top \widehat{\Delta_{f_1\to f_2}(\vx_2)} \widehat{\Delta_{f_1\to f_1}(\vx_2)}^\top \right]\right)\\
        &=  \tr\left(\E_{\vx_1} \left[ \widehat{\Delta_{f_1\to f_1}(\vx_1)} \widehat{\Delta_{f_1\to f_2}(\vx_1)}^\top \right] \cdot \E_{\vx_2} \left[ \widehat{\Delta_{f_1\to f_2}(\vx_2)} \widehat{\Delta_{f_1\to f_1}(\vx_2)}^\top \right]\right), \label{eq:prop-1-3}
    \end{align}
    where the last equality is because that $\vx_1, \vx_2$ are $i.i.d.$ samples from the same distribution. 
    
    Therefore, we can re-write the $\vx_1, \vx_2$ to be the same $\vx\sim \gD$ and realize that the two matrices are in fact the same one.
    \begin{align}
        (\ref{eq:prop-1-3})=&  \tr\left(\E_{\vx} \left[ \widehat{\Delta_{f_1\to f_1}(\vx)} \widehat{\Delta_{f_1\to f_2}(\vx)}^\top \right] \cdot \E_{\vx} \left[ \widehat{\Delta_{f_1\to f_2}(\vx)} \widehat{\Delta_{f_1\to f_1}(\vx)}^\top \right]\right)\\
        &= \big\|\E_{\vx} [ \widehat{\Delta_{f_1\to f_1}(\vx)} \widehat{\Delta_{f_1\to f_2}(\vx)}^\top ] \big\|_F^2.
    \end{align}
\end{proof}

\subsection{Proof of Proposition~\ref{prop:2}}\label{subsec:prop-2}
\begin{proposition}[Proposition~\ref{prop:2} Restated]
The adversarial transferability metrics $\alpha_1^{f_1\to f_2}$, $\alpha_2^{f_1\to f_2}$ and $(\alpha_1 * \alpha_2)^{f_1\to f_2}$ are in $[0,1]$.
\end{proposition}
\begin{proof}
    Let us begin with
    \begin{align}
    \alpha_1^{f_1\to f_2}(\vx) =\frac{\ell_{adv}(f_2(\vx),f_2(\vx+\bm{\delta}_{f_1, \epsilon}(\bm{x})) )}{\ell_{adv}(f_2(\vx),f_2(\vx+\bm{\delta}_{f_2, \epsilon}(\vx)) )} .
\end{align}
Recall that $\ell_{adv}(\cdot)\geq 0$, and the definition of adversarial attack:
\begin{align}
    \bm{\delta}_{f,\epsilon}(\bm{x}) \ =\  \argmax_{\|\bm{\delta}\|\leq \epsilon} \   \ell_{adv}(f(\vx),f(\vx+\bm{\delta}) ),
\end{align} 
and we can see that by definition,
\begin{align}
    0\leq \ell_{adv}(f_2(\vx),f_2(\vx+\bm{\delta}_{f_1, \epsilon}(\bm{x})) ) \leq \ell_{adv}(f_2(\vx),f_2(\vx+\bm{\delta}_{f_2, \epsilon}(\vx)) ).
\end{align}
Therefore, 
\begin{align}
    0\leq \frac{\ell_{adv}(f_2(\vx),f_2(\vx+\bm{\delta}_{f_1, \epsilon}(\bm{x})) )}{\ell_{adv}(f_2(\vx),f_2(\vx+\bm{\delta}_{f_2, \epsilon}(\vx)) )} \leq 1,
\end{align}
where we define $0/0=0$ if necessary.

Hence, $\alpha_1^{f_1\to f_2}=\E_{\vx \sim \gD}[\alpha_1^{f_1\to f_2}(\vx)]$ is also in $[0,1]$.

Next, we use Proposition~\ref{prop:1} to prove the same property for $\alpha_2^{f_1 \to f_2}$. Note that
\begin{align}
    (\alpha_2^{f_1\to f_2})^2   = \E_{\vx_1, \vx_2} \left[ \langle \widehat{\Delta_{f_1\to f_1}(\vx_1)}, \widehat{\Delta_{f_1\to f_1}(\vx_2)} \rangle\cdot \langle \widehat{\Delta_{f_1\to f_2}(\vx_1)}, \widehat{\Delta_{f_1\to f_2}(\vx_2)} \rangle\right] \label{eq:prop-2-1}
\end{align}
is the expectation of the product of two inner products, where each inner product is of two unit-length vector. That is being said,  $\langle \widehat{\Delta_{f_1\to f_1}(\vx_1)}, \widehat{\Delta_{f_1\to f_1}(\vx_2)} \rangle \in [-1, 1]$ and $\langle \widehat{\Delta_{f_1\to f_2}(\vx_1)}, \widehat{\Delta_{f_1\to f_2}(\vx_2)} \rangle \in [-1, 1]$. Therefore, we know that 
\begin{align}
    \E_{\vx_1, \vx_2} \left[ \langle \widehat{\Delta_{f_1\to f_1}(\vx_1)}, \widehat{\Delta_{f_1\to f_1}(\vx_2)} \rangle\cdot \langle \widehat{\Delta_{f_1\to f_2}(\vx_1)}, \widehat{\Delta_{f_1\to f_2}(\vx_2)} \rangle\right] \in [-1, 1].
\end{align}
In addition, we know from \eqref{eq:prop-2-1} that it is non-negative, and hence
\begin{align}
    (\alpha_2^{f_1\to f_2})^2 \in [0,1].
\end{align}
As $\alpha_2^{f_1\to f_2}$ itself is also non-negative by definition, we can see that $\alpha_2^{f_1\to f_2}\in [0,1]$.

Finally, we move to prove $(\alpha_1 * \alpha_2)^{f_1\to f_2}\in [0,1]$. Recall that 
\begin{align}
    (\alpha_1 * \alpha_2)^{f_1\to f_2}= \big\|\E_{\vx\sim \gD}[ \alpha_1^{f_1\to f_2}(\vx)\widehat{\Delta_{f_1\to f_1}(\vx)} \widehat{\Delta_{f_1\to f_2}(\vx)}^\top ]\big\|_F.
\end{align}
If we see $\alpha_1^{f_1\to f_2}(\vx)\widehat{\Delta_{f_1\to f_1}(\vx)}$ as a whole, we can show exactly the same as the Proposition~\ref{prop:1} that 
\begin{align}
    ((\alpha_1 * \alpha_2)^{f_1\to f_2})^2   = \E_{\vx_1, \vx_2} \left[  \theta_{f_1\to f_1}(\vx_1, \vx_2)\theta_{f_1\to f_2}(\vx_1, \vx_2)\right], \label{eq:prop-2-2}
\end{align}
where
\begin{align}
    \theta_{f_1\to f_1}(\vx_1, \vx_2) &=\langle\alpha_1^{f_1\to f_2}(\vx_1) \widehat{\Delta_{f_1\to f_1}(\vx_1)}, \alpha_1^{f_1\to f_2}(\vx_2)\widehat{\Delta_{f_1\to f_1}(\vx_2)} \rangle \\
    \theta_{f_1\to f_2}(\vx_1, \vx_2) &=\langle \widehat{\Delta_{f_1\to f_2}(\vx_1)}, \widehat{\Delta_{f_1\to f_2}(\vx_2)} \rangle .
\end{align}
Similarly, as $\alpha_1^{f_1\to f_2}(\vx)\in [0,1]$, we can see that $\theta_{f_1\to f_1}(\vx_1, \vx_2)\theta_{f_1\to f_2}(\vx_1, \vx_2)\in [-1,1]$, and hence 
\begin{align}
    \E_{\vx_1, \vx_2} \left[  \theta_{f_1\to f_1}(\vx_1, \vx_2)\theta_{f_1\to f_2}(\vx_1, \vx_2)\right]\in [-1,1].
\end{align}
Noting that \eqref{eq:prop-2-2} is non-negative, we conclude that 
\begin{align}
    ((\alpha_1 * \alpha_2)^{f_1\to f_2})^2 \in [0,1].
\end{align}
Since $(\alpha_1 * \alpha_2)^{f_1\to f_2}$ itself is non-negative as well, we can see that $(\alpha_1 * \alpha_2)^{f_1\to f_2}\in [0,1]$.

Therefore, the three adversarial transferability metrics are all within $[0,1]$.
\end{proof}

\section{Proofs in Section~\ref{sec:theory}}\label{sec:proof-2}
In this section, we prove the two theorems and the two propositions presented in section~\ref{sec:theory}, which are our main theories.
\subsection{Proof of Theorem~\ref{theorem-1}}\label{subsec:proof-theorem-1}
We introduce two lemmas before proving Theorem~\ref{theorem-1}.
\begin{lemma}\label{lemma:theorem-1-1}
The square of the gradient matching distance is
 \begin{align}
      \min_{g\in \sG} \quad \|\nabla f_\star^\top -  \nabla(g\circ f_\diamond)^\top\|^2_{\gD, \mH_\star} = \|\nabla f_\star^\top\|^2_{\gD, \mH_\star}-\langle\mP^\top \mH_\star \mP, \mJ^\dagger\rangle,
 \end{align}
 where $g\in \sG$ are affine transformations, and 
 \begin{align}
     \mP=\E_{\vx\sim \gD}\left[ \nabla f_\star(\vx)^\top \nabla f_\diamond (\vx)\right], \qquad
     \mJ=\E_{\vx\sim \gD}\left[ \nabla f_\diamond(\vx)^\top \nabla f_\diamond (\vx)\right].
 \end{align}
\end{lemma}
\begin{proof}
\begin{align}
    \min_{g\in \sG} \quad \|\nabla f_\star^\top -  \nabla(g\circ f_\diamond)^\top\|^2_{\gD, \mH_\star} &= \min_{\mW} \quad \|\nabla f_\star^\top -  \mW\nabla f_\diamond^\top\|^2_{\gD, \mH_\star}\\
    &=\min_{\mW} \quad \E_{\vx\in \gD} \|\nabla f_\star(\vx)^\top -  \mW\nabla f_\diamond(\vx)^\top\|^2_{ \mH_\star},\label{eq:1-1-1}
\end{align}
where $\mW$ is a matrix.

We can see that (\ref{eq:1-1-1}) is a convex program, where the optimal solution exists in a closed-form form, as shown in the following. Denote $l(\mW)=\|\nabla f_\star(\vx)^\top -  \mW\nabla f_\diamond(\vx)^\top\|^2_{ \mH_\star}$, we have
\begin{align}
    l(\mW) &= \E_{\vx\sim \gD} \left[ \|\nabla f_\star(\vx)^\top\|^2_{ \mH_\star} + \| \mW\nabla f_\diamond(\vx)^\top\|^2_{ \mH_\star} - 2\langle \nabla f_\star(\vx)^\top , \mW\nabla f_\diamond(\vx)^\top \rangle_{ \mH_\star} \right]\\
    &=\E_{\vx\sim \gD} \left[\|\nabla f_\star(\vx)^\top\|^2_{ \mH_\star} + \tr \left(  \nabla f_\diamond(\vx) \mW^\top \mH_\star \mW\nabla f_\diamond(\vx)^\top\right) - 2\tr\left( \nabla f_\star(\vx) \mH_\star \mW\nabla f_\diamond(\vx)^\top \right) \right]\\
    &=\E_{\vx\sim \gD} \left[\|\nabla f_\star(\vx)^\top\|^2_{ \mH_\star} + \tr \left(  \mH_\star \mW\nabla f_\diamond(\vx)^\top \nabla f_\diamond(\vx) \mW^\top\right) - 2\tr\left(\mH_\star \mW\nabla f_\diamond(\vx)^\top  \nabla f_\star(\vx) \right)\right].
\end{align}
Taking the derivative of $l(\cdot)$ w.r.t. $\mW$, we have
\begin{align}
    \frac{\partial l}{\partial \mW} &= \E_{\vx\sim \gD} \left[ 2\mH_\star\left( \mW\nabla f_\diamond(\vx)^\top - \nabla f_\star(\vx)^\top \right)   \nabla f_\diamond (\vx) \right]\\
    &= 2\mH_\star\E_{\vx\sim \gD}\left[ \mW\nabla f_\diamond(\vx)^\top \nabla f_\diamond (\vx) - \nabla f_\star(\vx)^\top  \nabla f_\diamond (\vx)  \right]\\
    &= 2\mH_\star \left(\mW \E_{\vx\sim \gD}\left[ \nabla f_\diamond(\vx)^\top \nabla f_\diamond (\vx)\right] - \E_{\vx\sim \gD}\left[ \nabla f_\star(\vx)^\top  \nabla f_\diamond (\vx)  \right] \right). \label{eq:1-1-1-n}
\end{align}
Since $l(\cdot)$ is convex, if there exists a $\Tilde{\mW}$ such that $ \frac{\partial l}{\partial \mW}\big |_{\mW=\Tilde{\mW}}=\bm{0}$ then we know that $\Tilde{\mW}$ is an optimal solution. Luckily, we can find such solution easily by using pseudo inverse, \emph{i.e.},
\begin{align}
    \Tilde{\mW} &= \E_{\vx\sim \gD}\left[ \nabla f_\star(\vx)^\top \nabla f_\diamond (\vx)\right] \left( \E_{\vx\sim \gD}\left[ \nabla f_\diamond(\vx)^\top \nabla f_\diamond (\vx)\right]  \right)^\dagger\\
    &=\mP \mJ^\dagger, \label{eq:1-1-2}
\end{align}
where we denote $\mP=\E_{\vx\sim \gD}\left[ \nabla f_\star(\vx)^\top \nabla f_\diamond (\vx)\right]$ and $\mJ=\E_{\vx\sim \gD}\left[ \nabla f_\diamond(\vx)^\top \nabla f_\diamond (\vx)\right] $.

We can verify that such $\Tilde{\mW}$ indeed make the partial derivative (\eqref{eq:1-1-1-n}) zero. In \eqref{eq:1-1-1-n}, we have
\begin{align}
\Tilde{\mW}\E_{\vx\sim \gD}\left[ \nabla f_\diamond(\vx)^\top \nabla f_\diamond (\vx)\right] - \E_{\vx\sim \gD}\left[ \nabla f_\star(\vx)^\top  \nabla f_\diamond (\vx)  \right] =  \mP \mJ^\dagger \mJ - \mP. \label{eq:1-1-1-nn}
\end{align}
To continue, we can see from Lemma~\ref{lemma:aux-include} that $\ker(\mJ)\subseteq \ker(\mP)$ which means   $\rowsp(\mP)\subseteq \rowsp(\mJ)$, where $\ker(\cdot)$ denotes the kernel of a matrix, and $\rowsp(\cdot)$ denotes the row space of a matrix. Therefore, by definition of the pseudo-inverse, we can see that $\mP \mJ^\dagger \mJ = \mP$, \emph{i.e.}, $(\ref{eq:1-1-1-nn})=\bm{0}$, and hence $\Tilde{\mW}$ is indeed the optimal solution.

Plugging (\ref{eq:1-1-2})   into (\ref{eq:1-1-1}), we have the optimal value as 
\begin{align}
    (\ref{eq:1-1-1}) &= l(\Tilde{\mW})\\
    &=\E_{\vx\sim \gD} \left[\|\nabla f_\star(\vx)^\top\|^2_{ \mH_\star} + \tr \left(  \mH_\star \Tilde{\mW}\nabla f_\diamond(\vx)^\top \nabla f_\diamond(\vx) \Tilde{\mW}^\top\right) - 2\tr\left(\mH_\star \Tilde{\mW}\nabla f_\diamond(\vx)^\top  \nabla f_\star(\vx) \right)\right]\\
    &= \|\nabla f_\star^\top\|^2_{\gD, \mH_\star}+\tr\left( \mH_\star \Tilde{\mW}\mJ \Tilde{\mW}^\top -2\mH_\star \Tilde{\mW}\mP^\top\right)\\
    &= \|\nabla f_\star^\top\|^2_{\gD, \mH_\star}+\tr\left( \mH_\star \mP \mJ^\dagger\mJ \mJ^\dagger \mP^\top -2\mH_\star \mP \mJ^\dagger\mP^\top\right)\\
    &=\|\nabla f_\star^\top\|^2_{\gD, \mH_\star}-\tr\left(\mH_\star \mP \mJ^\dagger\mP^\top\right)\\
    &=\|\nabla f_\star^\top\|^2_{\gD, \mH_\star}-\langle\mP^\top \mH_\star \mP, \mJ^\dagger\rangle.
\end{align}
\end{proof}

Next, we present another lemma to analyze the term $\mP^\top \mH_\star \mP$.
\begin{lemma}\label{lemma:theorem-1-2}
    In this lemma, we break down the matrix representation of $\mP^\top \mH_\star \mP$ into pieces relating to the output deviation caused by the generalized adversarial attacks (defined in \eqref{def:generalied-Delta})
    \begin{align}
        \mP^\top \mH_\star \mP= \E_{\vx_1, \vx_2 \overset{\text{i.i.d.}}{\sim} \gD}  \sum_{i,j=1}^n\left(  \Delta^{(i)}_{f_\star\to f_\star}(\vx_1)^\top  \mH_\star  \Delta^{(j)}_{f_\star\to f_\star}(\vx_2)\right)\cdot \left(  \Delta^{(i)}_{f_\star\to f_\diamond}(\vx_1)\Delta^{(j)}_{f_\star\to f_\diamond}(\vx_2)^\top \right).
    \end{align}
\end{lemma}
\begin{proof}
    Denote a symmetric decomposition of the positive semi-definitive matrix $\mH_\star$ as
    \begin{align}
        \mH_\star = \mT^\top \mT,
    \end{align}
    where $\mT$ is of the same dimension of $\mH_\star$. We note that the choice of decomposition does not matter.
    
    Then, plugging in the definition of $\mP$, we can see that
    \begin{align}
        \mP^\top \mH_\star \mP&= \E_{\vx\sim \gD}\left[ \nabla f_\diamond (\vx)^\top \nabla f_\star(\vx)\right] \cdot \mT^\top \mT \cdot \E_{\vx\sim \gD}\left[ \nabla f_\star(\vx)^\top \nabla f_\diamond (\vx)\right]\\
        &= \E_{\vx\sim \gD}\left[ \nabla f_\diamond (\vx)^\top \nabla f_\star(\vx) \mT^\top\right] \cdot \E_{\vx\sim \gD}\left[ \mT \nabla f_\star(\vx)^\top \nabla f_\diamond (\vx)\right].    \label{eq:1-2-1}
    \end{align}
    A key observation to connect the above equation to the adversarial attack (\eqref{def:adv-attack-norm}) is that, 
    \begin{align}
         \bm{\delta}_{f_\star,\epsilon}(\vx)\ &= \quad \argmax_{\|\bm{\delta}\|_2\leq \epsilon}\  \|\nabla f_\star(\vx)^\top\bm{\delta}\|_{\mH_\star}\\
         &= \ \argmax_{\|\bm{\delta}\|_2\leq \epsilon}\quad  \|\nabla f_\star(\vx)^\top\bm{\delta}\|_{\mH_\star}^2\\
         &= \ \argmax_{\|\bm{\delta}\|_2\leq \epsilon}\quad \bm{\delta}^\top \nabla f_\star(\vx) \mH_\star \nabla f_\star(\vx)^\top\bm{\delta}\\
         &= \ \argmax_{\|\bm{\delta}\|_2\leq \epsilon}\quad  \|\mT \nabla f_\star(\vx)^\top\bm{\delta}\|_{2}^2.
    \end{align}
    That is being said, the adversarial attack is the right singular vector corresponding to the largest singular value (in absolute value) of $\mT \nabla f_\star(\vx)^\top$.
    
    Similarly, we can see the singular values $\bm{\sigma}_{f_\star, \mH_\star}(\vx)\in \sR^n$, defined as the descending (in absolute value) singular values of the Jacobian $\nabla f_\star(\vx)^\top \in \sR^{\cdot \times n}$ in the $\mH_\star$ inner product space (\eqref{def:sigma}), are the singular values of $\mT \nabla f_\star(\vx)^\top$. 
    
    With this perspective, if we write down the singular value decomposition of $\mT \nabla f_\star(\vx)^\top$, \emph{i.e.},
    \begin{align}
        \mT \nabla f_\star(\vx)^\top = \mU_\star(\vx)\Sigma_\star(\vx)\mV_\star^\top (\vx),
    \end{align}
    we can observe that:
    \begin{enumerate}
        \item $\Sigma_\star(\vx)$ is diagonalized singular values $\bm{\sigma}_{f_\star, \mH_\star}(\vx)$;
        \item The $i^{th}$ column of $\mV_\star(\vx)$ is the $i^{th}$ generalized attack $ \bm{\delta}_{f_\star}^{(i)}(\vx)$ (defined in \eqref{def:generalized-attack});
        \item The $i^{th}$ column of $\mU_\star(\vx)\Sigma(\vx)$ is $\mT \Delta^{(i)}_{f_\star\to f_\star}(\vx)$ where $\Delta^{(i)}_{f_\star\to f_\star}(\vx)$ is the output deviation (defined in \eqref{def:generalied-Delta});
        \item The $i^{th}$ column of $\nabla f_\diamond (\vx)^\top \mV_\star(\vx)$ is the output deviation $\Delta^{(i)}_{f_\star\to f_\diamond}(\vx)$  (defined in \eqref{def:generalied-Delta}).
    \end{enumerate}  
    With the four key observations, we can break down the Jacobian matrices as
    \begin{align}
        \nabla f_\diamond (\vx)^\top \nabla f_\star(\vx) \mT^\top &=
        \begin{pmatrix}
            \Delta^{(1)}_{f_\star\to f_\diamond}(\vx) \dots \Delta^{(n)}_{f_\star\to f_\diamond}(\vx)
        \end{pmatrix}
        \begin{pmatrix}
            \Delta^{(1)}_{f_\star\to f_\star}(\vx)^\top \mT^\top\\
            \vdots\\
            \Delta^{(n)}_{f_\star\to f_\star}(\vx)^\top \mT^\top
        \end{pmatrix}\\
        &= \sum_{i=1}^n \Delta^{(i)}_{f_\star\to f_\diamond}(\vx) \Delta^{(i)}_{f_\star\to f_\star}(\vx)^\top \mT^\top. \label{eq:1-2-2}
    \end{align}
    Therefore, plugging it into the \eqref{eq:1-2-1}, we have
    \begin{align}
        (\ref{eq:1-2-1}) &= \E_{\vx\sim \gD}\left[ \sum_{i=1}^n \Delta^{(i)}_{f_\star\to f_\diamond}(\vx) \Delta^{(i)}_{f_\star\to f_\star}(\vx)^\top \mT^\top \right] \cdot \E_{\vx\sim \gD}\left[ \sum_{i=1}^n \mT \Delta^{(i)}_{f_\star\to f_\star}(\vx) \Delta^{(i)}_{f_\star\to f_\diamond}(\vx)^\top \right]\\
        &= \E_{\vx_1, \vx_2 \overset{\text{i.i.d.}}{\sim} \gD}\left[ \sum_{i=1}^n\left(  \Delta^{(i)}_{f_\star\to f_\diamond}(\vx_1) \Delta^{(i)}_{f_\star\to f_\star}(\vx_1)^\top \mT^\top \right) \sum_{j=1}^n \left( \mT \Delta^{(j)}_{f_\star\to f_\star}(\vx_2) \Delta^{(j)}_{f_\star\to f_\diamond}(\vx_2)^\top \right) \right]\\
        &= \E_{\vx_1, \vx_2 \overset{\text{i.i.d.}}{\sim} \gD}  \sum_{i,j=1}^n \left(  \Delta^{(i)}_{f_\star\to f_\diamond}(\vx_1) \Delta^{(i)}_{f_\star\to f_\star}(\vx_1)^\top   \mH_\star \Delta^{(j)}_{f_\star\to f_\star}(\vx_2) \Delta^{(j)}_{f_\star\to f_\diamond}(\vx_2)^\top \right)\\
        &= \E_{\vx_1, \vx_2 \overset{\text{i.i.d.}}{\sim} \gD}  \sum_{i,j=1}^n\left(  \Delta^{(i)}_{f_\star\to f_\star}(\vx_1)^\top  \mH_\star  \Delta^{(j)}_{f_\star\to f_\star}(\vx_2)\right)\cdot \left(  \Delta^{(i)}_{f_\star\to f_\diamond}(\vx_1)\Delta^{(j)}_{f_\star\to f_\diamond}(\vx_2)^\top \right),
    \end{align}
    where the last equality is due to that $\Delta^{(i)}_{f_\star\to f_\star}(\vx_1)^\top   \mH_\star \Delta^{(j)}_{f_\star\to f_\star}(\vx_2)$ is a scalar value.
    
\end{proof}

Equipped with Lemma~\ref{lemma:theorem-1-1} and Lemma~\ref{lemma:theorem-1-2}, we are able to prove the Theorem~\ref{theorem-1}. 

\begin{theorem}[Theorem~\ref{theorem-1} Restated]
Given the target and source models $f_\star, f_\diamond$, where $(\star, \diamond)\in\{(S, T), (T,S)\}$, the gradient matching distance (\eqref{def:gradient-matching-dist}) can be written as
    \begin{align}
        & \min_{g\in \sG} \quad \|\nabla f_\star^\top -  \nabla(g\circ f_\diamond)^\top\|_{\gD, \mH_\star}  = \left(1-\frac{\E[\vv^{\star,\diamond}_{}(\vx_1)^\top \mA_2^{\star,\diamond}({\vx_1, \vx_2}) \vv^{\star,\diamond}_{}(\vx_2)]  }{\|\nabla f_\star^\top\|_{\gD, \mH_\star}^2\cdot \|\mJ^\dagger\|^{-1}_{\mH_\diamond}} \right)^{\frac{1}{2}}\|\nabla f_\star^\top\|_{\gD, \mH_\star},
    \end{align}
    where the expectation is taken over $\vx_1, \vx_2 \overset{\text{i.i.d.}}{\sim} \gD$, and
    \begin{align}
        \vv^{\star,\diamond}_{}(\vx)&=\sigma^{(1)}_{f_\diamond, \mH_\diamond}(\vx)\bm{\sigma}_{f_\star, \mH_\star}(\vx) \odot \mA_1^{\star, \diamond}(\vx)\\
        \mJ&=\E_{\vx\sim\gD} [\nabla f_\diamond(\vx)^\top \nabla f_\diamond(\vx)].
    \end{align}
    Moreover, $\mA_2^{\star,\diamond}({\vx_1, \vx_2})$ is a matrix, and its element in the $i^{th}$ row and $j^{th}$ column is 
    \begin{align}
        \mA_2^{\star,\diamond}({\vx_1, \vx_2})^{(i, j)}=\langle \widehat{ \Delta^{(i)}_{f_\star\to f_\star}(\vx_1)}  \big |_{\mH_\star}, \widehat{ \Delta^{(j)}_{f_\star\to f_\star}(\vx_2)}\big |_{\mH_\star} \rangle \cdot \langle \widehat{ \Delta^{(i)}_{f_\star\to f_\diamond}(\vx_1)}\big |_{\mH_\diamond}, \widehat{ \Delta^{(j)}_{f_\star\to f_\diamond}(\vx_2)}\big |_{\mH_\diamond} \rangle_{\widehat{ \mJ^\dagger}|_{\mH_\diamond}} .
    \end{align}
\end{theorem}
\begin{proof}
    Combining the result from Lemma~\ref{lemma:theorem-1-1} and Lemma~\ref{lemma:theorem-1-2}, and applying the linearity of the inner product,  we have 
    \begin{align}
        &\min_{g\in \sG} \quad \|\nabla f_\star^\top -  \nabla(g\circ f_\diamond)^\top\|^2_{\gD, \mH_\star} \\
        =&\|\nabla f_\star^\top\|^2_{\gD, \mH_\star}-\bigg\langle\E_{\vx_1, \vx_2 \overset{\text{i.i.d.}}{\sim} \gD}  \sum_{i,j=1}^n\left(  \Delta^{(i)}_{f_\star\to f_\star}(\vx_1)^\top   \mH_\star \Delta^{(j)}_{f_\star\to f_\star}(\vx_2)\right)\cdot \left(  \Delta^{(i)}_{f_\star\to f_\diamond}(\vx_1)\Delta^{(j)}_{f_\star\to f_\diamond}(\vx_2)^\top \right), \mJ^\dagger\bigg\rangle\\
        =&\|\nabla f_\star^\top\|^2_{\gD, \mH_\star}-\E_{\vx_1, \vx_2 \overset{\text{i.i.d.}}{\sim} \gD}  \sum_{i,j=1}^n \left(  \Delta^{(i)}_{f_\star\to f_\star}(\vx_1)^\top   \mH_\star \Delta^{(j)}_{f_\star\to f_\star}(\vx_2)\right)\cdot \bigg\langle  \Delta^{(i)}_{f_\star\to f_\diamond}(\vx_1)\Delta^{(j)}_{f_\star\to f_\diamond}(\vx_2)^\top, \mJ^\dagger\bigg\rangle\\
        =&\|\nabla f_\star^\top\|^2_{\gD, \mH_\star}-\E_{\vx_1, \vx_2 \overset{\text{i.i.d.}}{\sim} \gD}  \sum_{i,j=1}^n \left(  \Delta^{(i)}_{f_\star\to f_\star}(\vx_1)^\top   \mH_\star \Delta^{(j)}_{f_\star\to f_\star}(\vx_2)\right)\cdot \tr \left(  \Delta^{(i)}_{f_\star\to f_\diamond}(\vx_1)\Delta^{(j)}_{f_\star\to f_\diamond}(\vx_2)^\top \mJ^\dagger\right)\\
        =&\|\nabla f_\star^\top\|^2_{\gD, \mH_\star}-\E_{\vx_1, \vx_2 \overset{\text{i.i.d.}}{\sim} \gD}  \sum_{i,j=1}^n  \underbrace{\left(  \Delta^{(i)}_{f_\star\to f_\star}(\vx_1)^\top   \mH_\star \Delta^{(j)}_{f_\star\to f_\star}(\vx_2)\right) }_{X_1} \cdot \underbrace{\left(\Delta^{(i)}_{f_\star\to f_\diamond}(\vx_1)^\top \mJ^\dagger   \Delta^{(j)}_{f_\star\to f_\diamond}(\vx_2)\right)}_{X_2}. \label{eq:1-3-1}
    \end{align}
    As the generalized first adversarial transferability $\mA_1$ is about the magnitude of the output deviation (defined in \eqref{def:alpha-1-generalized}), and we can separate the  $\mA_1$  out from the above equation. Then, what left should be about the directions about the output deviation, which we will put into the matrix $\mA_2$, \emph{i.e.}, the generalized second adversarial transferability.
    
    Recall that the generalized the first adversarial transferability is a $n$-dimensional vector $\mA_1^{\star, \diamond}(\vx)$ including the adversarial losses of all of the generalized adversarial attacks, where the $i^{th}$ element in the vector is
\begin{align}
    \mA_1^{\star, \diamond}(\vx)^{(i)}=\frac{\|\Delta_{f_\star\to f_\diamond}^{(i)}(\vx)\|_{\mH_\diamond}}{\|\nabla f_\diamond (\vx)\|_{\mH_\diamond}}.
\end{align}
    Moreover, to connect the magnitude of the output deviation to the generalized singular values (\eqref{def:generalized-attack}), we have
    \begin{align}
        \| \Delta^{(i)}_{f_\star\to f_\star}(\vx)\|_{\mH_\star}&=\|\nabla f_{\star}(\vx)^\top \bm{\delta}^{(i)}_{f_\star}(\vx)\|_{\mH_*}={\sigma}_{f_\star, \mH_\star}^{(i)}(\vx),
    \end{align}
    and similarly,
    \begin{align}
        \|\nabla f_\diamond (\vx)\|_{\mH_\diamond} = \|\nabla f_\diamond (\vx)\bm{\delta}^{(1)}_{f_\diamond}(\vx)\|_{\mH_\diamond} ={\sigma}_{f_\diamond, \mH_\diamond}^{(1)}(\vx).
    \end{align}
    Therefore, we can finally rewrite the $X_1, X_2$ in \eqref{eq:1-3-1} as
    \begin{align}
        X_1 &= {\sigma}_{f_\star, \mH_\star}^{(i)}(\vx_1){\sigma}_{f_\star, \mH_\star}^{(j)}(\vx_2)\cdot \langle \widehat{ \Delta^{(i)}_{f_\star\to f_\star}(\vx_1)}  \big |_{\mH_\star}, \widehat{ \Delta^{(j)}_{f_\star\to f_\star}(\vx_2)}\big |_{\mH_\star} \rangle\\
        X_2&= \mA_1^{\star, \diamond}(\vx_1)^{(i)}\mA_1^{\star, \diamond}(\vx_2)^{(j)} \cdot \langle \widehat{ \Delta^{(i)}_{f_\star\to f_\diamond}(\vx_1)}\big |_{\mH_\diamond}, \widehat{ \Delta^{(j)}_{f_\star\to f_\diamond}(\vx_2)}\big |_{\mH_\diamond} \rangle_{\widehat{ \mJ^\dagger}|_{\mH_\diamond}} \cdot {\sigma}_{f_\diamond, \mH_\diamond}^{(1)}(\vx_1){\sigma}_{f_\diamond, \mH_\diamond}^{(1)}(\vx_2) \|\mJ^\dagger\|_{\mH_{\diamond}}.
    \end{align}
    Recall the $(i,j)^{th}$ entry of the matrix $\mA_2$ is
    \begin{align}
        \mA_2^{\star, \diamond}(\vx_1, \vx_2)^{(i,j)}=\langle \widehat{ \Delta^{(i)}_{f_\star\to f_\star}(\vx_1)}  \big |_{\mH_\star}, \widehat{ \Delta^{(j)}_{f_\star\to f_\star}(\vx_2)}\big |_{\mH_\star} \rangle \cdot \langle \widehat{ \Delta^{(i)}_{f_\star\to f_\diamond}(\vx_1)}\big |_{\mH_\diamond}, \widehat{ \Delta^{(j)}_{f_\star\to f_\diamond}(\vx_2)}\big |_{\mH_\diamond} \rangle_{\widehat{ \mJ^\dagger}|_{\mH_\diamond}} .
    \end{align}
    We can write
    \begin{align}
        X_1 X_2= {\sigma}^{(1)}_{f_\diamond, \mH_\diamond}(\vx_1) {\sigma}_{f_\star, \mH_\star}^{(i)}(\vx_1)\mA_1^{\star, \diamond}(\vx_1)^{(i)} \cdot\mA_2^{\star, \diamond}(\vx_1, \vx_2)^{(i,j)} \cdot {\sigma}^{(1)}_{f_\diamond, \mH_\diamond}(\vx_2) {\sigma}_{f_\star, \mH_\star}^{(j)}(\vx_2) \mA_1^{\star, \diamond}(\vx_2)^{(j)}  \|\mJ^\dagger\|_{\mH_{\diamond}}.
    \end{align}
    Plugging the above into \eqref{eq:1-3-1}, and rearranging the double summation, we have
    \begin{align}
        &(\ref{eq:1-3-1}) = \|\nabla f_\star^\top\|^2_{\gD, \mH_\star}\\
        &-\E_{\vx_1, \vx_2 \overset{\text{i.i.d.}}{\sim} \gD} \left[ (\sigma^{(1)}_{f_\diamond, \mH_\diamond}(\vx_1)\bm{\sigma}_{f_\star,\mH_\star}(\vx_1) \odot \mA_1^{\star, \diamond}(\vx_1))^\top \mA_2^{\star,\diamond}({\vx_1, \vx_2}) (\sigma^{(1)}_{f_\diamond, \mH_\diamond}(\vx_2)\bm{\sigma}_{f_\star, \mH_\star}(\vx_2) \odot \mA_1^{\star, \diamond}(\vx_2)) \right]\|\mJ^\dagger\|_{\mH_{\diamond}}. \label{eq:1-3-2}
    \end{align}
    Denoting 
    \begin{align}
        \vv^{\star,\diamond}_{}(\vx)&=\sigma^{(1)}_{f_\diamond, \mH_\diamond}(\vx)\bm{\sigma}_{f_\star, \mH_\star}(\vx) \odot \mA_1^{\star, \diamond}(\vx),
    \end{align}
    and rearranging \eqref{eq:1-3-2} give us the Theorem~\ref{theorem-1}.

\end{proof}

\subsection{Proof of Proposition~\ref{prop:3}}\label{subsec:proof-prop-3}
From the proof of Theorem~\ref{theorem-1} in the above subsection, we can see why this proposition holds.
\begin{proposition}[Proposition~\ref{prop:3} Restated]
    In Theorem~\ref{theorem-1}, 
    \begin{align}
        0 \leq\frac{\E[\vv^{\star,\diamond}_{}(\vx_1)^\top \mA_2^{\star,\diamond}({\vx_1, \vx_2}) \vv^{\star,\diamond}_{}(\vx_2)]  }{\|\nabla f_\star^\top\|_{\gD, \mH_\star}^2\cdot \|\mJ^\dagger\|^{-1}_{\mH_\diamond}} \leq 1.
    \end{align}
\end{proposition}
\begin{proof}
    Recall Theorem~\ref{theorem-1} states
    \begin{align}
        & \min_{g\in \sG} \quad \|\nabla f_\star^\top -  \nabla(g\circ f_\diamond)^\top\|_{\gD, \mH_\star}  = \left(1-\frac{\E[\vv^{\star,\diamond}_{}(\vx_1)^\top \mA_2^{\star,\diamond}({\vx_1, \vx_2}) \vv^{\star,\diamond}_{}(\vx_2)]  }{\|\nabla f_\star^\top\|_{\gD, \mH_\star}^2\cdot \|\mJ^\dagger\|^{-1}_{\mH_\diamond}} \right)^{\frac{1}{2}}\|\nabla f_\star^\top\|_{\gD, \mH_\star}.
    \end{align}
    We can see that the $\leq 1$ part stands, since $\min_{g\in \sG} \quad \|\nabla f_\star^\top -  \nabla(g\circ f_\diamond)^\top\|_{\gD, \mH_\star}$ is always non-negative.
    
    The $\geq 0$ part can be proved by observing
    \begin{align}
        \left(1-\frac{\E[\vv^{\star,\diamond}_{}(\vx_1)^\top \mA_2^{\star,\diamond}({\vx_1, \vx_2}) \vv^{\star,\diamond}_{}(\vx_2)]  }{\|\nabla f_\star^\top\|_{\gD, \mH_\star}^2\cdot \|\mJ^\dagger\|^{-1}_{\mH_\diamond}} \right)^{\frac{1}{2}}\|\nabla f_\star^\top\|_{\gD, \mH_\star} &= \min_{g\in \sG} \quad \|\nabla f_\star^\top -  \nabla(g\circ f_\diamond)^\top\|_{\gD, \mH_\star} \\
         \leq \|\nabla f_\star^\top -  \nabla(0\circ f_\diamond)^\top\|_{\gD, \mH_\star}&=\|\nabla f_\star^\top\|_{\gD, \mH_\star}
    \end{align}
    
\end{proof}

\subsection{Proof of Theorem~\ref{theorem-2}}\label{subsec:proof-theorem-2}
We introduce two lemmas before proving Theorem~\ref{theorem-2}.
\begin{lemma}\label{lemma:theorem-2-1}
Assume that function $h(\cdot)$ satisfies the $\beta$-smoothness under $\|\cdot\|_{\mH_\star}$ norm (Assumption~\ref{assum:smooth}), and assume there is a vector $\vx_0$ in the same space as $\vx \sim \gD$ such that $h(\vx_0)=0$.  Given $\tau>0$, there exists $\vx'$ as a function of $\vx$ such that $\|\vx-\vx'\|_2 \leq \tau$, and 
\begin{align}
    \|h(\vx)\|^2_{\mH_\star} \leq  2\left(\|\nabla h(\vx')^\top\|^2_{\mH_\star} + \beta^2 \left(\|\vx-\vx_0\|_2-\tau \right)^2_+ \right) \cdot\|\vx-\vx_0\|^2_2,
\end{align}
where the $(\cdot)_+$ is an operator defined by $\forall x\in \sR$: $(x)_+ =x$ if $x\geq 0$ and $(x)_+=0$ otherwise.
\end{lemma}
\begin{proof}
    To begin with, we note that the assumption of $h(\vx_0)=0$ is only used for this lemma, and the assumption will be naturally guaranteed when we invoke this lemma in the proof of Theorem~\ref{theorem-2}. 
    
    With the smoothness assumption, we know that $h(\cdot)$ has continuous gradient. Thus, we have
    \begin{align}
        \|h(\vx)\|_{\mH_\star} &= \|h(\vx)-h(\vx_0)\|_{\mH_\star}=\|\nabla h(\vx_0 +\xi (\vx-\vx_0))^\top (\vx-\vx_0)\|_{\mH_\star},
    \end{align}
    where the last equation is by mean value theorem and thus $\xi\in (0,1)$.
    
    Then, noting that $\|\cdot\|_{\mH_\star}$ and $\|\cdot\|_2$ are compatible (Lemma~\ref{lemma:aux-compati}), we have
    \begin{align}
        \|\nabla h(\vx_0 +&\xi (\vx-\vx_0))^\top (\vx-\vx_0)\|_{\mH_\star} \leq \|\nabla h(\vx_0 +\xi (\vx-\vx_0))^\top\|_{\mH_\star}\cdot \|(\vx-\vx_0)\|_{2}.
    \end{align}
    Now we discuss two cases to define a random variable $\vx'$ as a function of $\vx$. 
    
    If $(1-\xi) \|\vx-\vx_0\|_2\leq \tau$, we define $\vx'$ as
    \begin{align}
        \vx' = \vx_0 +\xi (\vx-\vx_0),
    \end{align}
    and we can see that $\|\vx'-\vx\|_2\leq \tau$.
    
    Otherwise, \emph{i.e.}, $(1-\xi) \|\vx-\vx_0\|_2> \tau$, we apply triangle inequality to derive
    \begin{align}
         \|\nabla h(\vx_0 +&\xi (\vx-\vx_0))^\top (\vx-\vx_0)\|_{\mH_\star} \\
         &=\|\nabla h(\vx_0 +\xi (\vx-\vx_0))^\top-\nabla h(\vx - \tau \widehat{(\vx-\vx_0)})^\top+\nabla h(\vx - \tau \widehat{(\vx-\vx_0)})^\top\|_{\mH_\star}\\
        &\leq \underbrace{\|\nabla h(\vx_0 +\xi (\vx-\vx_0))^\top-\nabla h(\vx - \tau \widehat{(\vx-\vx_0)})^\top\|_{\mH_\star}}_{X}+\|\nabla h(\vx - \tau \widehat{(\vx-\vx_0)})^\top\|_{\mH_\star},
    \end{align}
    where we define 
    \begin{align}
        \vx' = \vx - \tau \widehat{(\vx-\vx_0)}.
    \end{align}
    By definition, in this case $\|\vx'-\vx\|_2\leq \tau$ as well.
    We then treat $X$: it can be bounded using $\beta$-smoothness, \emph{i.e.},
    \begin{align}
        X &\leq \beta \|\vx_0 +\xi (\vx-\vx_0)-\vx + \tau \widehat{(\vx-\vx_0)})\|_2\\
        &= \beta \| \tau \widehat{(\vx-\vx_0)} - (1-\xi)(\vx-\vx_0))\|_2\\
        &= \beta \big|\tau - (1-\xi)\cdot\|(\vx-\vx_0)\|_2\big|\\
        &= \beta \left((1-\xi)\cdot\|(\vx-\vx_0)\|_2-\tau \right),
    \end{align}
    where the last step is because we are exactly considering the case of $(1-\xi)\cdot\|(\vx-\vx_0)\|_2>\tau$.
    
    Therefore, combining the two cases together, we can write
    \begin{align}
        \|\nabla h(\vx_0 +\xi (\vx-\vx_0))^\top\|_{\mH_\star} \leq \beta \left((1-\xi)\cdot\|(\vx-\vx_0)\|_2-\tau \right)_+ + \|\nabla h(\vx')^\top\|_{\mH_\star},
    \end{align}
    where $\|\vx-\vx'\|\leq \tau$.
    
    Combining the above, we have 
    \begin{align}
         \|h(\vx)\|_{\mH_\star} \leq  \left(\|\nabla h(\vx')^\top\|_{\mH_\star} + \beta \left(\|\vx-\vx_0\|_2-\tau \right)_+ \right)\cdot\|\vx-\vx_0\|_2.
    \end{align}
    Take the square on both sides, and apply the Cauchy-Schwarz inequality, we have the lemma proved.
    \begin{align}
         \|h(\vx)\|^2_{\mH_\star} &\leq  \left(\|\nabla h(\vx')^\top\|_{\mH_\star} + \beta \left(\|\vx-\vx_0\|_2-\tau \right)_+ \right)^2\cdot\|\vx-\vx_0\|^2_2\\
         & \leq  2\left(\|\nabla h(\vx')^\top\|^2_{\mH_\star} + \beta^2 \left(\|\vx-\vx_0\|_2-\tau \right)^2_+ \right)\cdot\|\vx-\vx_0\|^2_2.
    \end{align}
\end{proof} 

\begin{lemma}\label{lemma:theorem-2-2}
    Assume that function $h(\cdot)$ satisfies the $\beta$-smoothness under $\|\cdot\|_{\mH_\star}$ norm (Assumption~\ref{assum:smooth}). Given $\tau>0$, there exists $\vx'_i$ as a function of $\vx$ for $\forall i\in [n]$ such that $\|\vx-\vx_i'\|_2 \leq \tau$, and 
\begin{align}
     \tau^2\cdot \| \nabla h(\vx)^\top\|^2_{\mH_\star} \leq  3 \left( \sum_{i=1}^n  \| h(\vx'_i)\|^2_{\mH_\star} + n\| h(\vx) \|^2_{\mH_\star} + n\tau^4 \beta^2 \right).
\end{align}
\end{lemma}
\begin{proof}
    Denote the dimension of $\vx$ as $n$, and let $\mU$ be an orthogonal matrix in $\sR^{n\times n}$, where we denote its column vectors as $\vu_i\in \sR^n$ for $i\in [n]$. Applying the mean value theorem, there exists $\xi_i\in (0,1)$ such that
    \begin{align}
        h(\vx+\tau \vu_i) - h(\vx) &= \nabla h(\vx +\tau \xi_i \vu_i )^\top \tau\vu_i\\
        &=\tau \left( \nabla h(\vx)^\top \vu_i + (\nabla h(\vx +\tau \xi_i \vu_i )^\top  -\nabla h(\vx )^\top) \vu_i \right).
    \end{align}
    Rearranging the equality, we have
    \begin{align}
        \nabla h(\vx)^\top \vu_i = \frac{1}{\tau}\bm{\gamma}_i,
    \end{align}
    where we denote 
    \begin{align}
        \bm{\gamma}_i = h(\vx+\tau \vu_i) - h(\vx) - \tau (\nabla h(\vx +\tau \xi_i \vu_i )^\top  -\nabla h(\vx )^\top) \vu_i.
    \end{align}
    Collecting each $\bm{\gamma}_i$ for $i\in [n]$ into a matrix $\mathbf{\Gamma}=[\bm{\gamma}_1 ... \bm{\gamma}_n]$, we can re-formulate the above equality as
    \begin{align}
        \tau \nabla h(\vx)^\top \mU &= \mathbf{\Gamma}\\
        \tau \nabla h(\vx)^\top &= \mathbf{\Gamma} \mU^\top,
    \end{align}
    where the last equality is because that $\mU$ is orthogonal.
    
    Taking the $\|\cdot\|^2_{\mH_\star}$ on both sides, with some linear algebra manipulation  we can derive
    \begin{align}
       \tau^2\cdot \| \nabla h(\vx)^\top\|^2_{\mH_\star} &= \|\mathbf{\Gamma} \mU^\top\|^2_{\mH_\star}\\
       &=\tr( \mU \mathbf{\Gamma}^\top \mH_\star \mathbf{\Gamma} \mU^\top)= \tr(  \mathbf{\Gamma}^\top \mH_\star \mathbf{\Gamma} )=\tr(   \mH_\star \mathbf{\Gamma}\mathbf{\Gamma}^\top )\\
       &=\tr( \mH_\star \sum_{i=1}^n \bm{\gamma}_i  \bm{\gamma}_i^\top)=\sum_{i=1}^n\tr( \mH_\star  \bm{\gamma}_i  \bm{\gamma}_i^\top)=\sum_{i=1}^n\tr(\bm{\gamma}_i^\top \mH_\star  \bm{\gamma}_i  )\\
       &=\sum_{i=1}^n \|\bm{\gamma}_i\|^2_{\mH_\star}. \label{eq:2-2-1}
    \end{align}
    Taking $\|\bm{\gamma}_i\|_{\mH_\star}$ to work on further, we can derive its upper bound as
    \begin{align}
        \|\bm{\gamma}_i\|_{\mH_\star}&=\| h(\vx+\tau \vu_i) - h(\vx) - \tau (\nabla h(\vx +\tau \xi_i \vu_i )^\top  -\nabla h(\vx )^\top) \vu_i\|_{\mH_\star}\\
        &\leq \| h(\vx+\tau \vu_i)\|_{\mH_\star} + \| h(\vx) \|_{\mH_\star} + \tau \|(\nabla h(\vx +\tau \xi_i \vu_i )^\top  -\nabla h(\vx )^\top) \vu_i\|_{\mH_\star}\\
        &\leq \| h(\vx+\tau \vu_i)\|_{\mH_\star} + \| h(\vx) \|_{\mH_\star} + \tau \|\nabla h(\vx +\tau \xi_i \vu_i )^\top  -\nabla h(\vx )^\top\|_{\mH_\star}\\
        &\leq \| h(\vx+\tau \vu_i)\|_{\mH_\star} + \| h(\vx) \|_{\mH_\star} + \tau^2 \beta \xi_i\\
        &\leq \| h(\vx+\tau \vu_i)\|_{\mH_\star} + \| h(\vx) \|_{\mH_\star} + \tau^2 \beta,  \label{eq:2-2-2}
    \end{align}
    where the first inequality is by triangle inequality, the second inequality is by Lemma~\ref{lemma:aux-compati} and the fact that $\|\vu_i\|_2=1$, the third inequality is done by applying the $\beta$-smoothness assumption, and the last inequality is by the fact that $\xi_i \in (0,1)$ from the mean value theorem.
    
    Plugging the \eqref{eq:2-2-2} into \eqref{eq:2-2-1}, we have
    \begin{align}
        \tau^2\cdot \| \nabla h(\vx)^\top\|^2_{\mH_\star} \leq \sum_{i=1}^n  \left( \| h(\vx+\tau \vu_i)\|_{\mH_\star} + \| h(\vx) \|_{\mH_\star} + \tau^2 \beta \right)^2\\
        \leq \sum_{i=1}^n  3 \left( \| h(\vx+\tau \vu_i)\|^2_{\mH_\star} + \| h(\vx) \|^2_{\mH_\star} + \tau^4 \beta^2 \right)\\
        = 3 \sum_{i=1}^n  \| h(\vx+\tau \vu_i)\|^2_{\mH_\star} + 3n\| h(\vx) \|^2_{\mH_\star} + 3n\tau^4 \beta^2 ,
    \end{align}
    where the inequality is done Cauchy-Schwarz inequality.
    
    Denoting $\vx_i'=\vx+\tau \vu_i$, we have the lemma proved.
    
\end{proof}

\begin{theorem}[Theorem~\ref{theorem-2} Restated]
Given a data distribution $\gD$ and $\tau>0$, there exist distributions $\gD_1, \gD_2$ such that the type-1 Wasserstein distance $W_1(\gD, \gD_1)\leq \tau$ and $W_1(\gD, \gD_2)\leq \tau$ satisfying
    \begin{align}
        {\tfrac{1}{2B^2 }}\|h_{\star, \diamond}\|^2_{\gD, \mH_\star} &\leq \|\nabla h_{\star, \diamond}'^\top\|^2_{\gD_1, \mH_\star}+\beta^2(B-\tau)_+^2\\
        \tfrac{1}{3n}\|\nabla h_{\star, \diamond}'^\top\|^2_{\gD, \mH_\star}&\leq \tfrac{2}{\tau^2} \|h_{\star, \diamond}\|^2_{\gD_2, \mH_\star}+\beta^2\tau^2,
    \end{align}
    where $n$ is the dimension of $\vx\sim\gD$, and $B=\inf_{\vx_0\in \sR^n} \sup_{\vx \in \supp(\gD)} \|\vx-\vx_0\|_2$ is the radius of the $\supp(\gD)$. The $(\cdot)_+$ is an operator defined by $\forall x\in \sR$: $(x)_+ =x$ if $x\geq 0$ and $(x)_+=0$ otherwise.
\end{theorem}
\begin{proof}
    Let us begin with recalling the definition of $h_{\star, \diamond}$ and $h_{\star, \diamond}'$.
    
    The optimal affine transformation $g\in \sG$ in the function matching distance (\ref{def:function-matching-dist}) is $\Tilde{g}$, and one of the optimal $g\in \sG$ in the gradient matching distance is (\ref{def:gradient-matching-dist-new})  $\Tilde{g}'$. Accordingly, we denote
\begin{align}
    h_{\star, \diamond} :=f_\star - \Tilde{g}\circ f_\diamond \quad \text{and} \quad h_{\star, \diamond}' :=f_\star - \Tilde{g}'\circ f_\diamond, 
\end{align}
and we can see that the gradient matching distance and the function matching distance can be written as
\begin{align}
    (\ref{def:function-matching-dist}) = \|h_{\star, \diamond}\|_{\gD, \mH_\star}  \quad \text{and} \quad (\ref{def:gradient-matching-dist-new}) = \|\nabla {h'_{\star, \diamond}}^\top\|_{\gD, \mH_\star}. \label{eq:2-3-5-n}
\end{align}

    \textbf{The first inequality.} Then, we can prove the first inequality using Lemma~\ref{lemma:theorem-2-1}. 
    
    Let $\vx_0\in \sR^n$ be a free variable, and then set $\vb = h_{\star, \diamond}'(\vx_0)$. Noting that $\|h_{\star, \diamond}\|^2_{\gD, \mH_\star}$ by definition is the minimum of this function distance, we have 
    \begin{align}
        \|h_{\star, \diamond}\|^2_{\gD, \mH_\star} \leq \|h'_{\star, \diamond}-\vb\|^2_{\gD, \mH_\star}. \label{eq:2-3-1}
    \end{align}
    Denoting $h:=h'_{\star, \diamond}-\vb$, we can see $h(\vx_0)=0$. Therefore, $h$ can be used to invoke Lemma~\ref{lemma:theorem-2-1}. That is, there exists $\vx'$ as a function of $\vx$ such that $\|\vx-\vx'\|_2 \leq \tau$, and 
    \begin{align}
        \|h(\vx)\|^2_{\mH_\star} \leq  2\left(\|\nabla h(\vx')^\top\|^2_{\mH_\star} + \beta^2 \left(\|\vx-\vx_0\|_2-\tau \right)^2_+ \right) \cdot\|\vx-\vx_0\|^2_2,
    \end{align}
    Taking the expectation of $\vx\sim \gD$ of the both sides, and denote the induced distribution for $\vx'$ as $\gD_1$, we have
    \begin{align}
        \E_{\vx \sim \gD} \|h(\vx)\|^2_{\mH_\star} \leq 2 \E_{\vx \sim \gD} \left(\|\nabla h(\vx')^\top\|^2_{\mH_\star} + \beta^2 \left(\|\vx-\vx_0\|_2-\tau \right)^2_+ \right) \cdot\|\vx-\vx_0\|^2_2.
    \end{align}
    Recall that $\vx_0$ is a free variable, we can tighten the bound by
    \begin{align}
        \E_{\vx \sim \gD} \|h(\vx)\|^2_{\mH_\star} \leq  \inf_{\vx_0 \in \sR^n} 2 \E_{\vx \sim \gD} \left(\|\nabla h(\vx')^\top\|^2_{\mH_\star} + \beta^2 \left(\|\vx-\vx_0\|_2-\tau \right)^2_+ \right) \cdot\|\vx-\vx_0\|^2_2. \label{eq:2-3-2}
    \end{align}
    Note that we can have tighter but similar results if we keep the $\inf{\vx_0 \in \sR^n}$. However, by plugging in the radius 
    $$B=\inf_{\vx_0\in \sR^n}\sup_{ \vx \in \supp(\gD)} \|\vx-\vx_0\|_2$$
    we can make the presentation much more simplified without losing its core messages. 
    
    That is,
    \begin{align}
        (\ref{eq:2-3-2}) \leq 2  \left(\E_{\vx' \sim \gD_1}\|\nabla h(\vx')^\top\|^2_{\mH_\star} + \beta^2 \left(B-\tau \right)^2_+ \right) B^2. 
    \end{align}
    Combining the above inequality and \eqref{eq:2-3-1},  and noting that
    \begin{align}
        \E_{\vx \sim \gD} \|h(\vx)\|^2_{\mH_\star} &= \|h\|^2_{\gD, \mH_\star}\\
        \E_{\vx' \sim \gD_1}\|\nabla h(\vx')^\top\|^2_{\mH_\star}&=\|\nabla h^\top\|^2_{\gD_1, \mH_\star},
    \end{align}
    we have
    \begin{align}
        \|h_{\star, \diamond}\|^2_{\gD, \mH_\star} &\leq \|h'_{\star, \diamond}-\vb\|^2_{\gD, \mH_\star}=\E_{\vx \sim \gD} \|h(\vx)\|^2_{\mH_\star}\\
        &\leq 2  \left(\E_{\vx' \sim \gD_1}\|\nabla h(\vx')^\top\|^2_{\mH_\star} + \beta^2 \left(B-\tau \right)^2_+ \right) B^2 \\
        &= 2  \left(\|\nabla h^\top\|^2_{\gD_1, \mH_\star} + \beta^2 \left(B-\tau \right)^2_+ \right) B^2 
    \end{align}
    Noting that $h$ and $h'_{\star, \diamond}$ only differs by a constant shift $\vb$, we can see $\nabla h=\nabla h'_{\star, \diamond}$. Therefore, by replacing $\nabla h^\top$ by $\nabla h_{\star, \diamond}'^\top$
    we finally have the first inequality in Theorem~\ref{theorem-2}
    \begin{align}
        \|h_{\star, \diamond}\|^2_{\gD, \mH_\star}\leq 2\left( \|\nabla h_{\star, \diamond}'^\top\|^2_{\gD_1, \mH_\star} + \beta^2 \left(B-\tau \right)^2_+ \right) B^2 .
    \end{align}
    It remains to show the Wasserstein distance between $\gD_1$ and $\gD$.
    As $\vx'$ is a function of the random variable $\vx\sim \gD$ with $\|\vx'-\vx\|_2\leq \tau$, and $\gD_1$ is the induced distribution of $\vx'$ as a function of $\vx$, we can see that by the definition of type-1 Wasserstein distance between $\gD$ and $\gD_1$ is bounded by $\tau$. 
    
    Denote $\sJ(\gD, \gD')$ as the set of all joint distributions that have marginals $\gD$ and $\gD'$, and recall the definition of type-1 Wasserstein distance is
    \begin{align}
        W_1(\gD, \gD_1) =\inf_{\gJ \in \sJ(\gD, \gD_1)} \int \|\vx - \vx'\|_2 \diff \gJ(\vx, \vx').
    \end{align}
    Denote $\gJ_0$ as the joint distribution such that in $(\vx, \vx')\sim \gJ$ we always have $\vx'$ being a function of $\vx$ as how $\vx'$ is defined. We can see that
    \begin{align}
        W_1(\gD, \gD_1)&=\inf_{\gJ \in \sJ(\gD, \gD_1)} \int \|\vx - \vx'\|_2 \diff \gJ(\vx, \vx') \leq \int \|\vx - \vx'\|_2 \diff \gJ_0(\vx, \vx')\leq \int \tau \diff \gJ_0(\vx, \vx')\\
        &=\tau.  \label{eq:2-3-1-n}
    \end{align}
    Therefore, we have the first inequality in the theorem proved .
    
    \textbf{The second inequality.} Invoking Lemma~\ref{lemma:theorem-2-2} with $h_{\star, \diamond}$, and rearranging the inequality, we have 
    \begin{align}
        \tfrac{1}{3n} \| \nabla h_{\star, \diamond}(\vx)^\top\|^2_{\mH_\star} \leq \tfrac{2}{\tau^2} \left( \sum_{i=1}^n  \tfrac{1}{2n}\| h_{\star, \diamond}(\vx'_i)\|^2_{\mH_\star} + \tfrac{1}{2} \| h_{\star, \diamond}(\vx) \|^2_{\mH_\star} \right) + \tau^2 \beta^2.
    \end{align}
    Taking the expectation on both sides, we have
    \begin{align}
        \tfrac{1}{3n} \E_{\vx\sim \gD} \| \nabla h_{\star, \diamond}(\vx)^\top\|^2_{\mH_\star} \leq \tfrac{2}{\tau^2} \underbrace{ \E_{\vx\sim \gD} \left( \sum_{i=1}^n  \tfrac{1}{2n}\| h_{\star, \diamond}(\vx'_i)\|^2_{\mH_\star} + \tfrac{1}{2} \| h_{\star, \diamond}(\vx) \|^2_{\mH_\star} \right) }_{X} + \tau^2 \beta^2. \label{eq:2-3-2-n}
    \end{align}
    Note that $X$ can be reformulated to be the expectation of an induced distribution from $x\sim \gD$, since $\vx_i'$ is a pre-defined function of $\vx$. Denote $\gD_2$ as the distribution induced by the following sampling process: first, sample $\vx\sim \gD$; then, 
    \begin{align}
        \vx' &= \vx \quad \text{with probability } \tfrac{1}{2}\\
        \vx' &= \vx'_i \quad \text{with probability } \tfrac{1}{2n} \text{ for } \forall i\in [n].
    \end{align}
    Therefore, we can write $X$ as
    \begin{align}
         X=\| h_{\star, \diamond}\|^2_{\gD_2, \mH_\star} . \label{eq:2-3-3-n}
    \end{align}
    Similarly to \eqref{eq:2-3-1-n}, it also holds that $W_1(\gD, \gD_2)\leq \tau$.
    
    To finally complete the proof, noting that $\|\nabla h_{\star, \diamond}'^\top\|^2_{\gD, \mH_\star}$ is the minimum of this gradient distance (\eqref{eq:2-3-5-n}), we have
    \begin{align}
        \|\nabla h_{\star, \diamond}'^\top\|^2_{\gD, \mH_\star} \leq \E_{\vx\sim \gD} \| \nabla h_{\star, \diamond}(\vx)^\top\|^2_{\mH_\star}. \label{eq:2-3-4-n}
    \end{align}
    Combining \eqref{eq:2-3-2-n}, \eqref{eq:2-3-3-n} and \eqref{eq:2-3-4-n}, we have the second inequality proved. 
    
    Hence, we have proved Theorem~\ref{theorem-2}.
\end{proof}

\subsection{Proof of Theorem~\ref{prop:4}}\label{subsec:proof-prop-4}

\begin{theorem}[Theorem~\ref{prop:4} Restated] The surrogate transfer loss (\ref{eq:function-distance-S-T}) and the true transfer loss (\ref{def:know-transfer-dist}) are close, with an error of  $\| f_T - y\|_{\gD, \mH_T}$.
\begin{align}
       -\| f_T - y\|_{\gD, \mH_T} \leq (\ref{def:know-transfer-dist})-(\ref{eq:function-distance-S-T}) \leq  \| f_T - y\|_{\gD, \mH_T}
\end{align}
\end{theorem}
\begin{proof}
    Let us begin by recall the definition of the surrogate transfer loss (\ref{eq:function-distance-S-T}) and the true transfer loss (\ref{def:know-transfer-dist}).
    \begin{align}
        (\ref{eq:function-distance-S-T}) &:= \min_{g\in \sG} \  \| f_T - g\circ f_S\|_{\gD, \mH_T}\\
        (\ref{def:know-transfer-dist}) &:= \min_{g\in \sG} \ \| y - g\circ f_S\|_{\gD, \mH_T}. 
    \end{align}
    Denote 
    \begin{align}
        \Tilde{g}'&:=\argmin_{g\in \sG} \  \| f_T - g\circ f_S\|_{\gD, \mH_T}\\
        \Tilde{g}&:=\argmin_{g\in \sG} \  \| y - g\circ f_S\|_{\gD, \mH_T}.
    \end{align}
    First, we show an upper bound for (\ref{eq:function-distance-S-T}).
    \begin{align}
        (\ref{eq:function-distance-S-T}) \leq  \| f_T - \Tilde{g}\circ f_S\|_{\gD, \mH_T} \leq  \| y - \Tilde{g}\circ f_S\|_{\gD, \mH_T} + \| f_T - y\|_{\gD, \mH_T} =  (\ref{def:know-transfer-dist}) + \| f_T - y\|_{\gD, \mH_T}, \label{eq:prop-4-1}
    \end{align}
    where the last inequality is by triangle inequality.
    
    Similarly, we can derive its lower bound.
    \begin{align}
        (\ref{eq:function-distance-S-T}) &=  \| f_T - \Tilde{g}'\circ f_S\|_{\gD, \mH_T} \geq \| y - \Tilde{g}'\circ f_S\|_{\gD, \mH_T} - \| f_T - y\|_{\gD, \mH_T}\\
        &\geq \min_{g\in \sG} \| y - g\circ f_S\|_{\gD, \mH_T} - \| f_T - y\|_{\gD, \mH_T}= (\ref{def:know-transfer-dist}) - \| f_T - y\|_{\gD, \mH_T}, \label{eq:prop-4-2}
    \end{align}
    where the first inequality is by triangle inequality.
    
    Combining \eqref{eq:prop-4-1} and \eqref{eq:prop-4-2}, we have the proposition proved.
    
\end{proof}

\subsection{Proof of Theorem~\ref{prop:5}}\label{subsec:proof-prop-5}
\begin{theorem}[Theorem~\ref{prop:5} Restated]
    Denote $\Tilde{g}_{T,S}:\sR^m\to \sR^d$ as the optimal solution of \eqref{eq:function-distance-S-T}, and $\Tilde{g}_{S,T}:\sR^d\to \sR^m$ as the optimal solution of \eqref{eq:function-distance-T-S}. Suppose the two optimal affine maps $\Tilde{g}_{T,S}, \Tilde{g}_{S,T}$ are both full-rank. For $\vv \in \sR^m$, denote the matrix representation of $\Tilde{g}_{T,S}$ as $\Tilde{g}_{T,S}(\vv)=\Tilde{\mW}_{T,S}\vv + \Tilde{\vb}_{T,S}$. Similarly, for $\vw \in \sR^d$, denote the matrix representation of $\Tilde{g}_{S,T}$ as $\Tilde{g}_{S,T}(\vw)=\Tilde{\mW}_{S,T}\vw + \Tilde{\vb}_{S,T}$. We have the following statements.
    
    If $d<m$, then $\Tilde{g}_{S,T}$ is injective, and we have:
    \begin{align}
        \|f_T - \Tilde{g}_{T,S}\circ f_S\|_{\gD, \mH_T} \leq \sqrt{\|(\Tilde{\mW}_{S,T}^\top\Tilde{\mW}_{S,T})^{-1}\|_F\cdot \|\mH_T\|_F} \cdot \|f_S - \Tilde{g}_{S,T}\circ f_T\|_{\gD}. \tag{\ref{eq:prop-5-1}}
    \end{align}
    If $d>m$, then $\Tilde{g}_{T,S}$ is injective, and we have:
    \begin{align}
        \|f_S - \Tilde{g}_{S,T}\circ f_T\|_{\gD, \mH_S} \leq \sqrt{\|(\Tilde{\mW}_{T,S}^\top\Tilde{\mW}_{T,S})^{-1}\|_F\cdot \|\mH_S\|_F} \cdot \|f_T - \Tilde{g}_{T,S}\circ f_S\|_{\gD}. \tag{\ref{eq:prop-5-2}}
    \end{align}
    If $d=m$, then both $\Tilde{g}_{S,T}$ and $\Tilde{g}_{T,S}$ are bijective, and we have both (\ref{eq:prop-5-1}) and (\ref{eq:prop-5-2}) stand.
\end{theorem}
\begin{proof}
    Observing the symmetry, we only need to prove the following claim.
    
    \textbf{Claim.} \textit{For $\star, \diamond\in \{S,T\}$ and $\star\neq \diamond$, if $\Tilde{g}_{\star, \diamond}$ is injective, then}
    \begin{align}
        \|f_\diamond - \Tilde{g}_{\diamond,\star}\circ f_\star\|^2_{\gD, \mH_\diamond} \leq {\|(\Tilde{\mW}_{\star,\diamond}^\top\Tilde{\mW}_{\star,\diamond})^{-1}\|_F\cdot \|\mH_\diamond\|_F} \cdot \|f_\star - \Tilde{g}_{\star,\diamond}\circ f_\diamond\|^2_{\gD}.
    \end{align}
    \begin{proof}[Proof of the Claim]
        We have mostly done with this claim with Lemma~\ref{lemma:aux-injective}. Noting that $\Tilde{g}_{\diamond,\star}$ is the minimizer of $\min_{g\in \sG}\ \|f_\diamond - {g}\circ f_\star\|^2_{\gD, \mH_\diamond}$, we have
        \begin{align}
            \|f_\diamond - \Tilde{g}_{\diamond,\star}\circ f_\star\|^2_{\gD, \mH_\diamond} &\leq \|f_\diamond - \Tilde{g}^{-1}_{\star, \diamond}\circ f_\star\|^2_{\gD, \mH_\diamond} = \E_{\vx\sim \gD}\left[  \|f_\diamond(\vx) - \Tilde{g}^{-1}_{\star, \diamond}( f_\star(\vx))\|^2_{\mH_\diamond} \right]\\
            &\leq \E_{\vx\sim \gD} \left[{\| (\Tilde{\mW}_{\star, \diamond}^\top \Tilde{\mW}_{\star, \diamond})^{-1}\|_F \cdot \|\mH_\diamond\|_F }\cdot \|f_\star(\vx)  - \Tilde{g}_{\star, \diamond}(f_\diamond (\vx))\|^2_2\right]\\
            &={\| (\Tilde{\mW}_{\star, \diamond}^\top \Tilde{\mW}_{\star, \diamond})^{-1}\|_F \cdot \|\mH_\diamond\|_F }\cdot \|f_\star  - \Tilde{g}_{\star, \diamond}\circ f_\diamond \|^2_\gD,
        \end{align}
        where the second inequality is by invoking Lemma~\ref{lemma:aux-injective}.
    \end{proof}
    Taking the square root of this claim, and applying $(\diamond=T, \star=S)$ or $(\diamond=S, \star=T)$, we immediatly have the first two statements about the case of $d<m$ or $d>m$. Finally, noting that when $m=d$, both $\Tilde{g}_{S,T}$ and $\Tilde{g}_{T,S}$ are bijective and thus also injective, we can see that both (\ref{eq:prop-5-1}) and (\ref{eq:prop-5-2}) stand.
\end{proof}

\section{Auxiliary Lemmas}\label{sec:aux-lemmas}
\begin{lemma}[Compatibility of $\|\cdot\|_\mH$ and $\|\cdot\|_2$]\label{lemma:aux-compati}
    Let $\mH\in \sR^{m\times m}$ be a positive semi-definite matrix, and denote $\mH=\mT^\top \mT$ as its symmetric decomposition with $\mT\in \sR^{m\times m}$. For $\mW\in \sR^{m \times n}$ and $\vv \in \sR^n$, we have
    \begin{align}
        \|\mW\vv\|_\mH \leq \|\mW\|_\mH\cdot \|\vv\|_2.
    \end{align}
\end{lemma}
\begin{proof}
    \begin{align}
        \|\mW\vv\|^2_\mH &= \vv^\top \mW^\top \mT^\top \mT \mW \vv= \| \mT \mW \vv\|^2_2\\
        &\leq \| \mT \mW\|^2_2 \cdot\|\vv\|^2_2 \leq \| \mT \mW\|^2_F \cdot\|\vv\|^2_2,
    \end{align}
    where $\|\cdot\|_F$ is the Frobenius norm. 
    Then, we can continue as
    \begin{align}
        \| \mT \mW\|^2_F = \tr(\mW^\top \mT^\top \mT \mW)=\tr(\mW^\top\mH  \mW)=\|\mW\|_\mH^2.
    \end{align}
    Combining the above two parts, we have the lemma proved.
\end{proof}

\begin{lemma}[Expectation Preserves the Inclusion Relationship Between Linear Spaces]\label{lemma:aux-include}
    Given a distribution $\vx\sim \gD$ in $\sR^n$, we denote the associated probability measure as $\mu$. Given linear maps $\mM_\vx:\sR^n\to\sR^m$ and $\mN_\vx:\sR^n \to \sR^d$, noting that they are both functions of $\vx$, we have the following statement.
    \begin{align}
        \ker\left(\E_{\vx\sim \gD} \mM_x^\top \mM_x \right) \subseteq \ker\left(\E_{\vx\sim \gD} \mN_x^\top \mM_x \right) ,
    \end{align}
    where $\ker(\cdot)$ denotes the kernel space of a given liner map.
\end{lemma}
\begin{proof}
    It suffice to show for $\forall \vv\in \ker\left(\E_{\vx\sim \gD} \mM_x^\top \mM_x \right)$, we also have $\vv\in \ker\left(\E_{\vx\sim \gD} \mN_x^\top \mM_x \right)$.
    
    Denote $\mP := \E_{\vx\sim \gD} \mM_x^\top \mM_x $, and let $\vv \in \ker(\mP)$, we have
    \begin{align}
        \mP \vv = \bm{0}.
    \end{align}
    Noting that $\mP$ is positive semi-definite, we have the following equivalent statements.
    \begin{align}
        \vv\in \ker(\mP) \quad \iff \quad  \vv^\top \mP \vv =0,
    \end{align}
    where the '$\implies$' direction is trivial, and the '$\impliedby$' direction can be proved by decomposing $\mP=\mT^\top \mT$ as two matrices and noting that
    \begin{align}
        \vv^\top \mT^\top \mT \vv =0 \quad\implies \quad \|\mT \vv\|_2^2 = 0  \quad\implies \quad \mT \vv=\bm{0} \quad\implies\quad \mT^\top \mT \vv=\bm{0} \quad \implies \mP\vv =\bm{0}.
    \end{align}
    Therefore, we have
    \begin{align}
        \vv^\top \mP \vv &= 0 \\
        \implies \quad  \E_{\vx\sim\gD} \left[ \vv^\top \mM_x^\top \mM_x \vv \right] &= 0\\
        \implies \quad  \E_{\vx\sim\gD} \left[ \|\mM_x \vv\|_2^2 \right] &= 0\\
        \implies \quad  \int  \|\mM_x \vv\|_2^2 \diff \mu &= 0,
    \end{align}
    which implies $\mM_x \vv=\bm{0}$ almost everywhere w.r.t. $\mu$.
    
    Therefore, applying $\vv$ to $\E_{\vx\sim \gD} [\mN_x^\top \mM_x ]$ and we have
    \begin{align}
        \E_{\vx\sim \gD} \left[ \mN_x^\top \mM_x \right] \vv &= \int \mN_x^\top \mM_x \vv \diff \mu\\
        &=\int_{a.e.} \mN_x^\top \bm{0}\diff \mu\\
        &=\bm{0},
    \end{align}
    which means  $\vv \in \ker\left(\E_{\vx\sim \gD} \mN_x^\top \mM_x \right)$.

\end{proof}

\begin{lemma}[Inverse an Injective Linear Map]\label{lemma:aux-injective} Given a full-rank injective affine transformation $g:\sR^m \to \sR^d$, we denote its matrix representation as $g(\vv)=\mW \vv+\vb$ where $\vv\in \sR^m, \mW\in \sR^{d\times m}, \vb\in\sR^d$. The inverse of $g$ is $g^{-1}:\sR^d\to \sR^m$ defined by $g^{-1}(\vw):=\mW^\dagger\vw - \mW^\dagger\vb$ for $\vw\in \sR^d$, \emph{i.e.}, $g^{-1}\circ g$ is the identity function. Moreover, given a positive semi-definite matrix $\mH$, for $\forall \vv\in \sR^m$ and $\forall \vw \in \sR^d$, we have
\begin{align}
    \sqrt{\| (\mW^\top \mW)^{-1}\|_F \cdot \|\mH\|_F }\cdot \|\vw - g(\vv)\|_2 \geq \|\vv - g^{-1}(\vw)\|_\mH. 
\end{align}
\end{lemma}
\begin{proof}
    First, let us verify that $g^{-1}\circ g$ is the identity function. The conditions of $g$ being full-rank and injective are equivalent to $\mW$ being full-rank and $d\geq m$. That is being said, $\mW^\top \mW$ is invertible and $\mW^\dagger=(\mW^\top \mW)^{-1}\mW^\top$. Therefore, for $\forall \vv \in \sR^m$, we have
    \begin{align}
        g^{-1}\circ g (\vv) &= \mW^\dagger(\mW \vv + \vb) - \mW^\dagger\vb = \mW^\dagger\mW \vv \\
        &=(\mW^\top \mW)^{-1}\mW^\top\mW \vv =\vv.
    \end{align}
    That is, $g^{-1}\circ g$ is indeed the identity function.
    
    Next, to prove the inequality, let us start from the right-hand-side of the inequality.
    \begin{align}
        \|\vv - g^{-1}(\vw)\|_\mH &=  \|g^{-1}\circ g (\vv) - g^{-1}(\vw)\|_\mH\\
        &=\|\mW^\dagger(g (\vv) - \vw)\|_\mH\\
        &\leq \|\mW^\dagger\|_\mH \cdot \|g (\vv) - \vw\|_2, \label{eq:aux-3-1}
    \end{align}
    where the inequality is done by applying Lemma~\ref{lemma:aux-compati}.
    
    To complete the prove, we can see that
    \begin{align}
        \|\mW^\dagger\|^2_\mH &= \|(\mW^\top \mW)^{-1}\mW^\top\|^2_\mH =\tr(\mW (\mW^\top \mW)^{-1}\mH(\mW^\top \mW)^{-1}\mW^\top)\\
        &=\tr((\mW^\top \mW)^{-1}\mH(\mW^\top \mW)^{-1}\mW^\top \mW)=\tr((\mW^\top \mW)^{-1}\mH)\\
        &=\langle (\mW^\top \mW)^{-1}, \mH \rangle\\
        &\leq \| (\mW^\top \mW)^{-1}\|_F \cdot \|\mH\|_F . \label{eq:aux-3-2}
    \end{align}
    Plugging the square root of \eqref{eq:aux-3-2} into \eqref{eq:aux-3-1}, we have the lemma proved.
\end{proof}

\section{Additional Details of Synthetic Experiments}\label{sec:appendix-synthetic-exp}
In this section, we complete the description of the settings and methods used in the synthetic experiments. Moreover, we report two additional sets of results in cross-architecture scenarios.

In the main paper (section~\ref{sec:synthetic-exp}), the synthetic experiments are done on the setting where source models have the same architecture as the target model, \emph{i.e.}, all the models are one-hidden-layer neural networks with width $m=100$. A natural question is what would the results be if using different architectures? That is, the architecture of the  source models are different from the target model. To answer this question, we present two additional sets of synthetic experiments where the width of the source models is $m=50$ or $m=200$, different from the target model (width $m=100$).

As we have presented in the main paper about the description of the methods and models used in this experiment, here we present the detailed description of the settings and the datasets being used.

\begin{figure}[h!]\centering
    \begin{minipage}{0.24\linewidth}\centering
        \includegraphics[width=\linewidth]{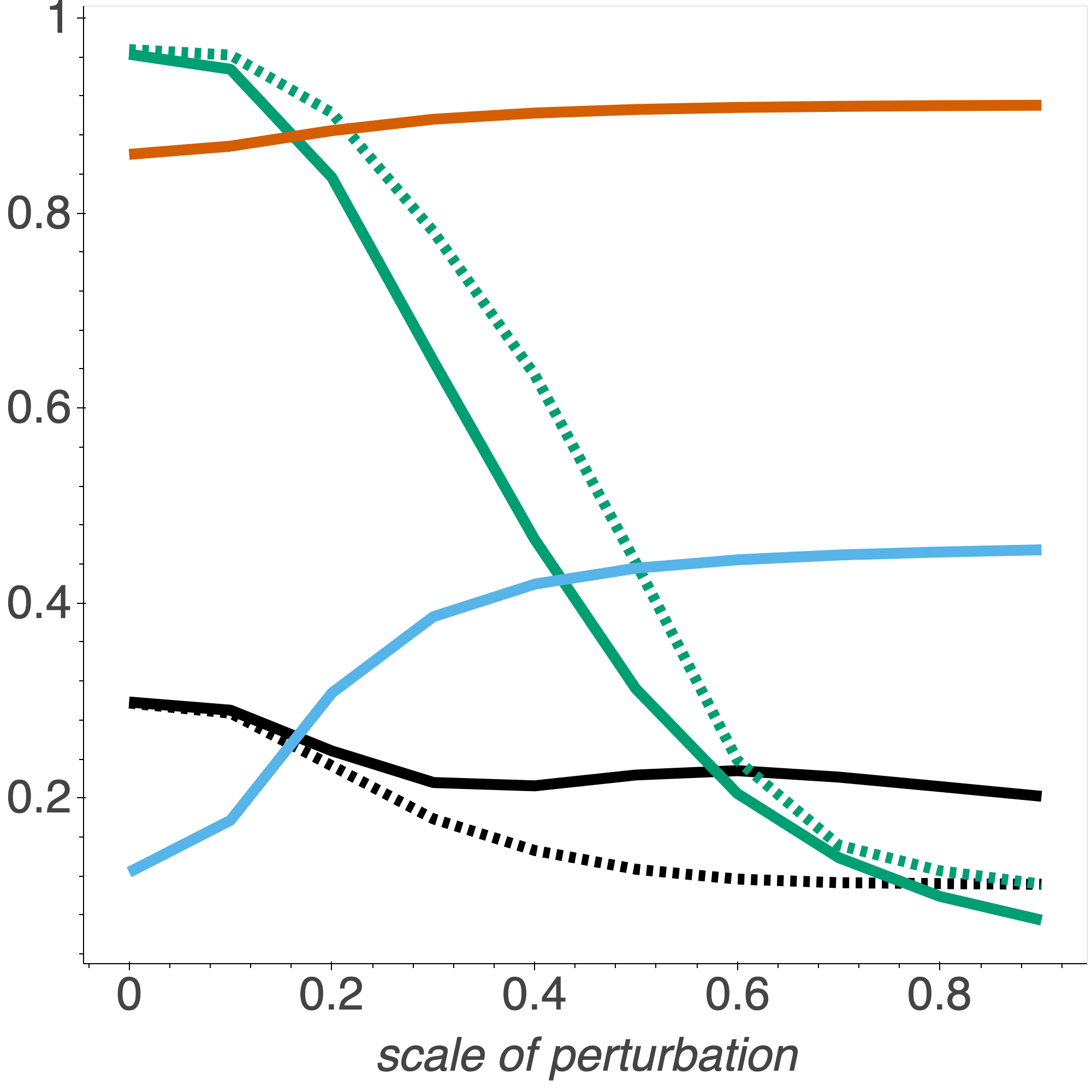}
        {\small (a) width=$50$, $\bm{\delta}_{f_\star}^{(1)}$}
    \end{minipage}
    \begin{minipage}{0.24\linewidth}\centering
        \includegraphics[width=\linewidth]{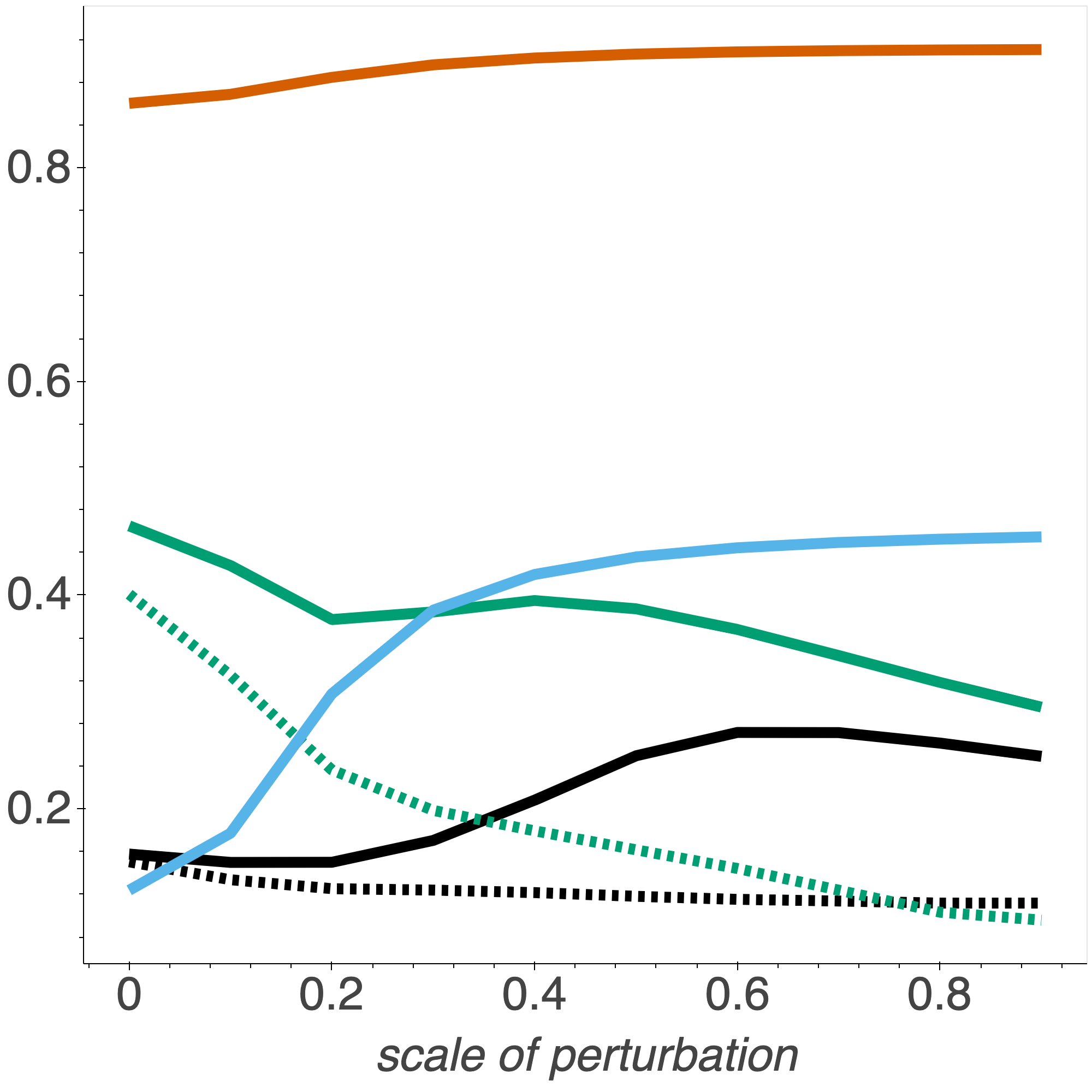}
        {\small (b) width=$50$, $\bm{\delta}_{f_\star}^{(2)}$}
    \end{minipage}
    \begin{minipage}{0.24\linewidth}\centering
        \includegraphics[width=\linewidth]{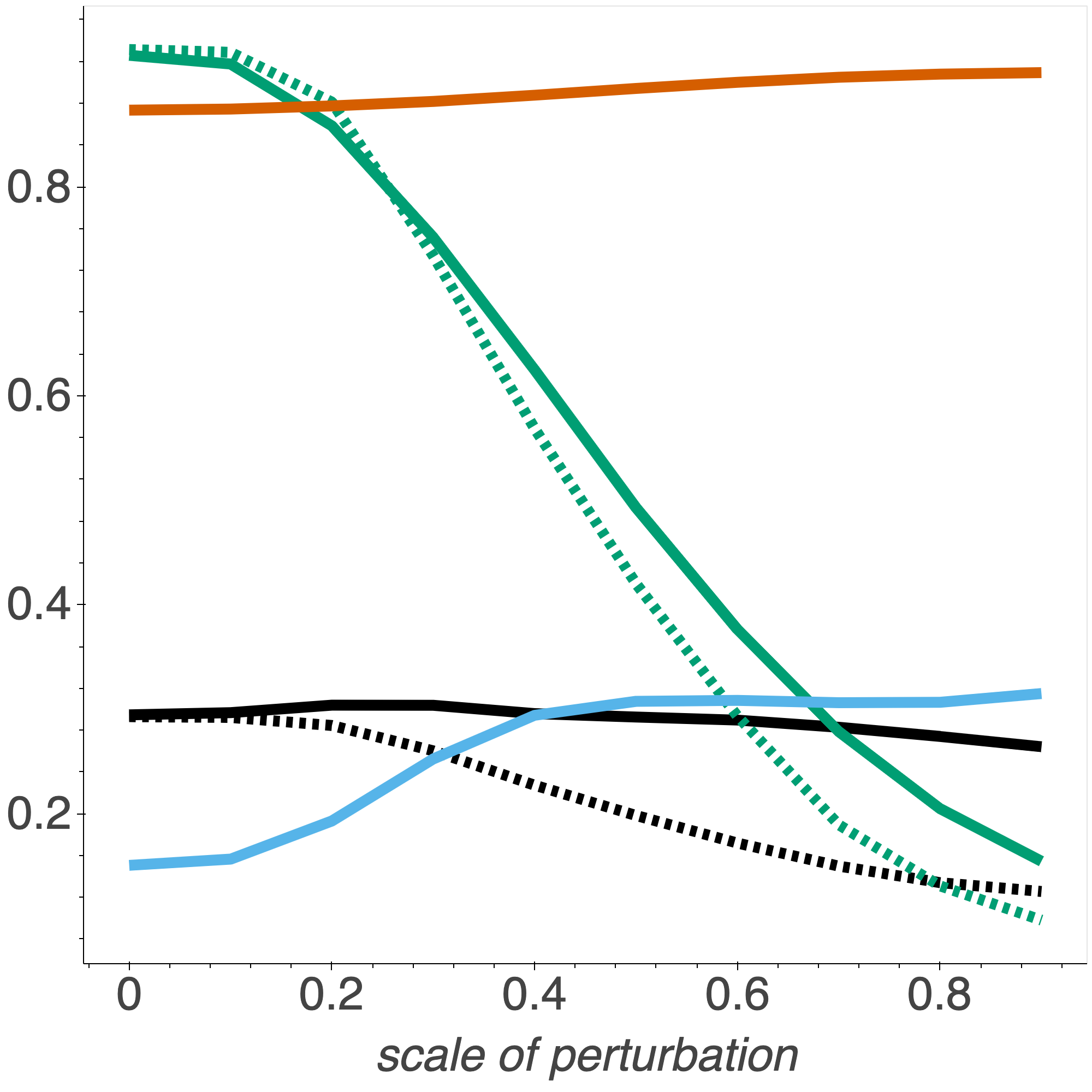}
        {\small (c) width=$200$, $\bm{\delta}_{f_\star}^{(1)}$}
    \end{minipage}
    \begin{minipage}{0.24\linewidth}\centering
        \includegraphics[width=\linewidth]{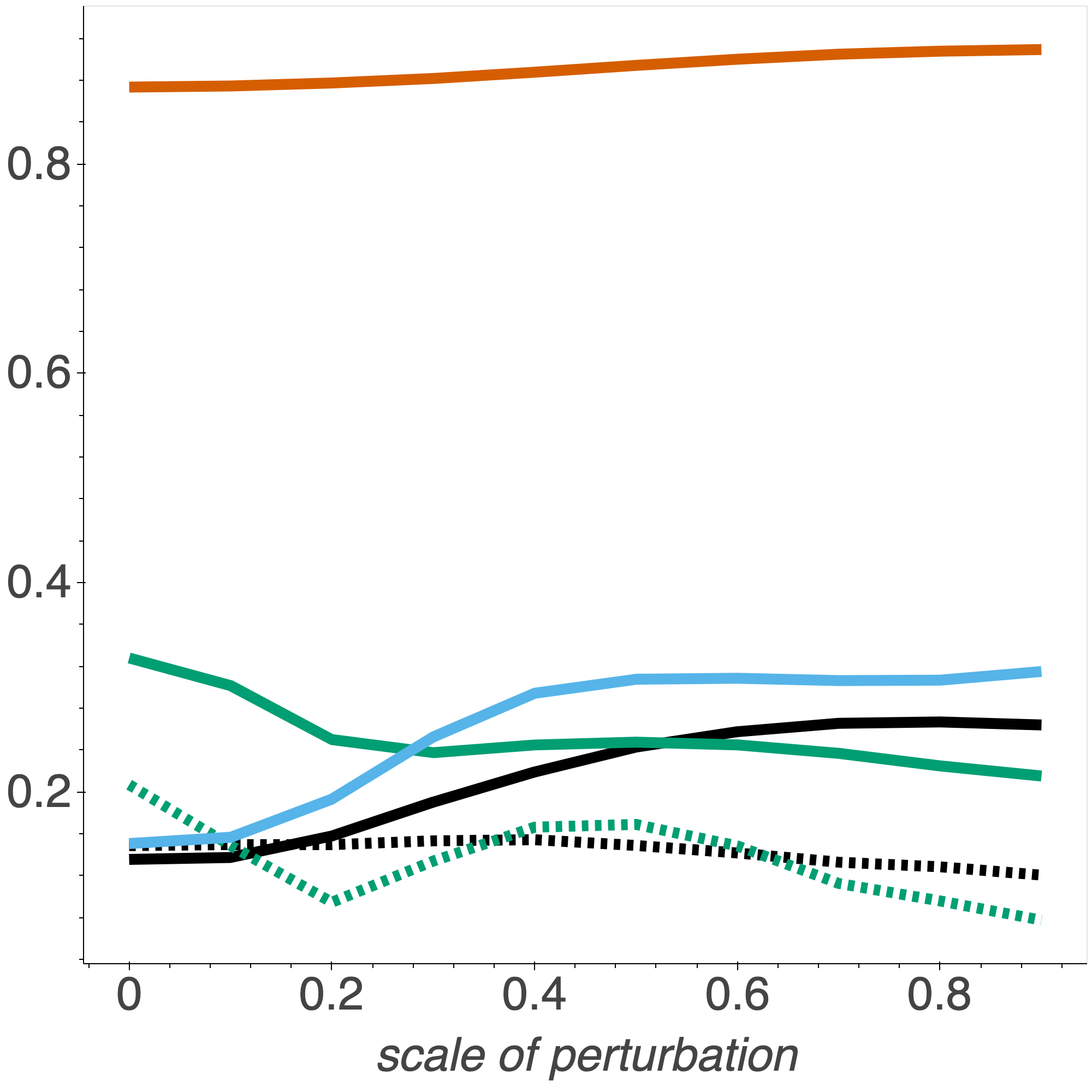}
        {\small (d) width=$200$, $\bm{\delta}_{f_\star}^{(2)}$}
    \end{minipage}
    \caption{ \small In this figure, 'width' is the width of the source models (one-hidden-layer neural networks). As defined in \eqref{def:generalized-attack},  $\bm{\delta}_{f_\star}^{(1)}$ corresponds to the regular adversarial attacks, while  $\bm{\delta}_{f_\star}^{(2)}$ the secondary adversarial attack. That is, $\bm{\delta}_{f_\star}^{(2)}$ represents the other information in the adversarial transferring process compared with the first. The x-axis shows the scale of perturbation $t\in [0,1]$ that controls how much the source model deviates from its corresponding reference source model. There are in total 6 quantities reported. Specifically, $\alpha_1^{f_T\to f_S}$ is \textbf{black solid}; $\alpha_1^{f_S\to f_T}$ is \textbf{black dotted};  $\alpha_2^{f_T\to f_S}$ is \textcolor{green}{\textbf{green solid}}; $\alpha_2^{f_S\to f_T}$ is \textcolor{green}{\textbf{green dotted}}; the gradient matching loss is \textcolor{red}{\textbf{red solid}}; and the knowledge transferability distance is  \textcolor{blue}{\textbf{blue solid}}.
    }
    \label{fig:exp-synthetic-width}
\end{figure}

\textbf{Settings.}
We follow the small-$\epsilon$ setting used in the theory, \emph{i.e.}, the adversarial attack are constrained to a small magnitude, so that we can use its first-order Talyor approximation.

\textbf{Dataset.}
Denote a radial basis function as $\phi_i(\vx) = e^{-\|\vx-\bm{\mu}_i\|_2^2/(\sigma_i)^2}$, and for each input data we form its corresponding $M$-dimensional  feature vector as $\bm{\phi}(\vx)=[\phi_1(\vx),\dots,\phi_M(\vx)]^\top$. We set the dimension of $\vx$ to be $50$. For each radial basis function $\phi_i(\vx), i\in [M]$, $\bm{\mu}_i$ is sampled from $U(-0.5, 0.5)^{50}$, and $\sigma_i^2$ is sampled from $U(0, 100)$.  We use $M=100$ radial basis functions so that the feature vector is $100$-dimensional. Then, we set the target ground truth to be $y(\vx)=\mW\bm{\phi}(\vx)+\vb$ where $\mW\in \sR^{10\times 100}, \vb\in\sR^{10}$ are sampled from $U(-0.5, 0.5)$ element-wise.  We generate $N=5000$ samples of $\vx$ from a Gaussian mixture formed by $10$ Gaussians with different centers but the same covariance matrix $\mathbf{\Sigma}=\mI$. The centers are sampled randomly from $U(-0.5, 0.5)^{50}$. That is, the dataset  $D=\{(\vx_i, \vy_i)\}_{i=1}^N$ consists of  $N=5000$ sample from the distribution, where $\vx_i$ is $50$-dimensional, $\vy_i$ is $10$-dimensional. The ground truth target $\vy_i$ are computed using the ground truth target function $y(\vx_i)$. That is, we want our neural networks to approximate $y(\cdot)$ on the Gaussian mixture. 

\textbf{Methods of Additional Experiments.} Note that we have provided the detailed description of the methods used in the main paper synthetic experiments. Here, we present the methods for two additional sets of synthetic experiments, using the same dataset and settings, but different architectures. In the main paper, the source model and the target model are of the same architecture, and the source models are perturbed target model. Here, we use the same target model $f_T$ (width $m=100$) trained on the dataset $D$,  but two different architectures for source models. That is, the source models and the target model are of different width. 

To derive the source models, we first train two reference source models on $D$ with width $m = 50$ and $m=200$. For each of the reference models, denoting the weights of the model as $\mW$, we randomly sample a direction $\mV$ where each entry of $\mV$ is sampled from $U(-0.5, 0.5)$, and choose a scale $t\in [0, 1]$. Subsequently, we perturb the model weights of the clean source model as $\mW’ := \mW + t\mV$, and define the source model $f_S$ to be a one-hidden-layer neural network with weights $\mW’$. 
Then, we compute each of the quantities we care about, including $\alpha_1$, $\alpha_2$ from both $f_S\to f_T$ and $f_T\to f_S$, the gradient matching distance (\eqref{def:gradient-matching-dist}), and the actual knowledge transfer distance (\eqref{def:know-transfer-dist}). We use the standard $\ell_2$ loss as the adversarial loss function.

\textbf{Results.} We present four sets of result in Figure~\ref{fig:exp-synthetic-width}. The indication relations between adversarial transferability and knowledge transferability can be observed in the cross-architecture setting. Moreover: 1. the metrics $\alpha_1, \alpha_2$ are more meaningful if using the regular attacks; 2. the gradient matching distance tracks the actual knowledge transferability loss; 3. the directions of $f_T\to f_S$ and $f_S \to f_T$ are similar.

\section{Details of the Empirical Experiments}
\label{section:appendix_exp}
All experiments are run on a single GTX2080Ti.

\subsection{Datasets}
\subsubsection{Image Datasets}
\begin{itemize}
    \item \textbf{CIFAR10}:\footnote{\url{https://www.cs.toronto.edu/~kriz/cifar.html}}: it consists of 60000 32$\times$32 colour images in 10 classes, with 6000 images per class. There are 50000 training images and 10000 test images.
    \item \textbf{STL10}:\footnote{\url{https://cs.stanford.edu/~acoates/stl10/}}: it consists of 13000 labeled 96$\times$96 colour images in 10 classes, with 1300 images per class. There are 5000 training images and 8000 test images.  
500 training images (10 pre-defined folds), 800 test images per class.
\end{itemize}
\subsubsection{NLP Datasets}
\begin{itemize}
    \item \textbf{IMDB}:\footnote{\url{https://datasets.imdbws.com/}} Document-level sentiment classification on positive and negative movie reviews. We use this dataset to train the target model.
    \item \textbf{AG's News (AG)}: Sentence-level classification with regard to four news topics: World, Sports, Business, and Science/Technology. Following \citet{zhang2015character}, we concatenate the title and description fields for each news article. We use this dataset to train the source model.
    \item \textbf{Fake News Detection (Fake)}: Document-level classification on whether a news article is fake or not. The dataset comes from the Kaggle Fake News Challenge\footnote{\url{https://www.kaggle.com/c/fake-news/data}}. We concatenate the title and news body of each article. We use this dataset to train the source model.
    \item \textbf{Yelp}:  Document-level sentiment classification on positive and negative reviews \citep{zhang2015character}. Reviews with a rating of 1 and 2 are labeled negative and 4 and 5 positive. We use this dataset to train the source model.
\end{itemize}

\subsection{Adversarial Trasnferability Indicating Knowledge Transferability}
\label{sec:adv}
\subsubsection{Image}
For all the models, both source and target, in the Cifar10 to STL10 experiment, we train them by SGD with momentumn and learning rate 0.1 for 100 epochs. For knowledge tranferability, we randomly reinitialize and train the source models' last layer for 10 epochs on STL10. Then we generate adversarial examples with the target model on the validation set and measure the adversarial transferability by feeding these adversarial examples to the source models. We employ two adversarial attacks in this experiments and show that they achieve the same propose in practice: First, we generate adversarial examples by 50 steps of projected gradient descent and epsilon $0.1$ (Results shown in Table \ref{table:exp1}). Then, we generate adversarial examples by the more efficient FGSM with epsilon $0.1$ (Results shown in Table \ref{table:exp1_fgsm}) and show that we can efficiently identify candidate models without the expensive PGD attacks.

To further visualize the averaged relation presented in Table \ref{table:exp1} and \ref{table:exp1_fgsm}, we plot scatter plots Figure \ref{fig:exp1_img} and Figure \ref{fig:exp1_pgd_img} with per sample $\alpha_1$ as x axis and per sample transfer loss as y axis. Transfer loss is the cross entropy loss predicted by the source model with last layer fine-tuned on STL10.  The Pearson score indicates strong correlation between adversarial transferability and knowledge transferability.

\begin{figure}[hbt!]
    \centering
    \includegraphics[width =0.5 \linewidth]{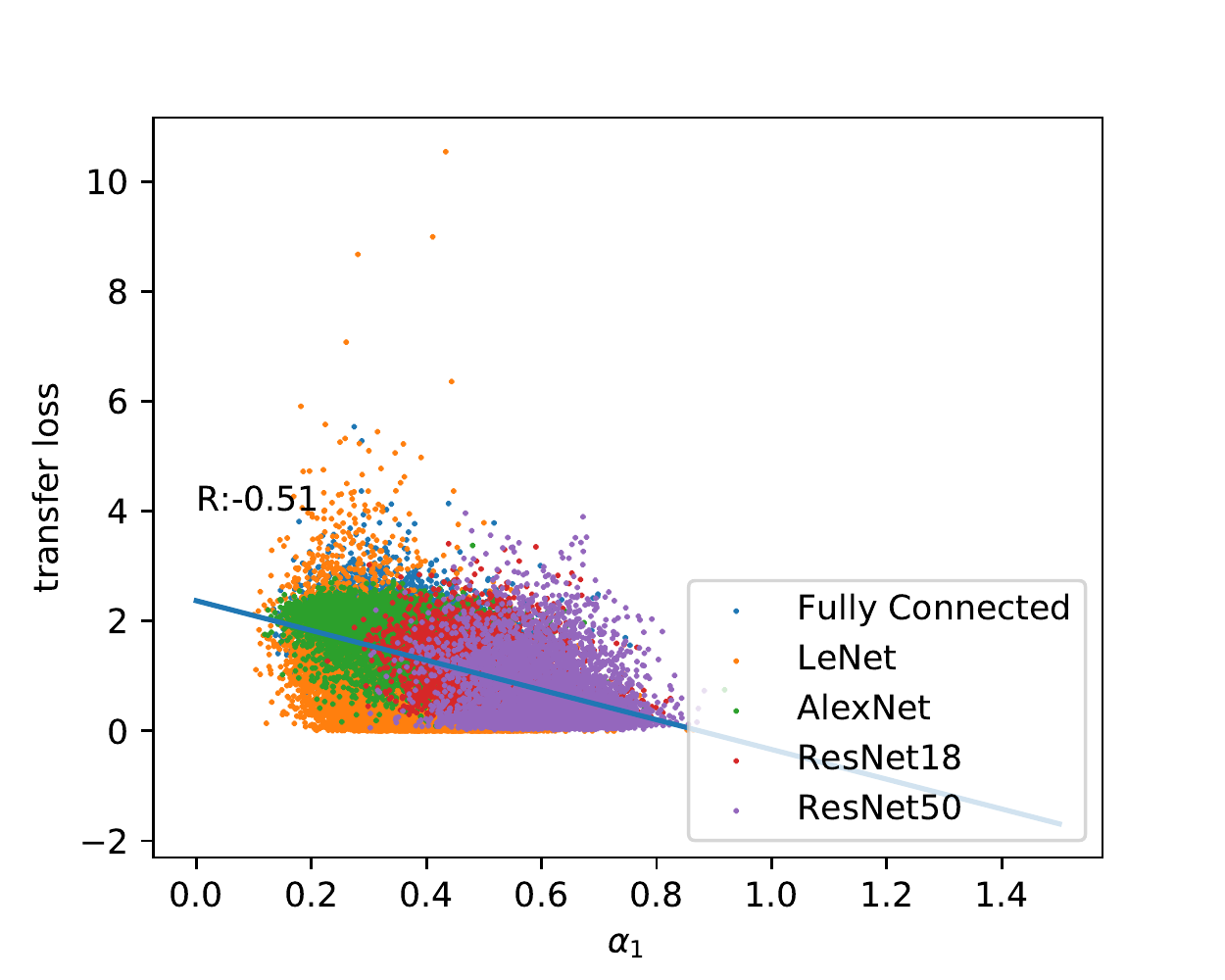}
    \caption{Distribution of per sample knowledge transfer loss and $\alpha_1$. The adversarial samples are generated by PGD. The Pearson score shows strong negative  correlation between $\alpha_1$ and the knowledge transfer loss. The higher the transfer loss is, the lower the knowledge transferability is, and the lower the $\alpha_1$ is.}
    \label{fig:exp1_pgd_img}
\end{figure}

\begin{figure}[hbt!]
    \centering
    \includegraphics[width =0.5 \linewidth]{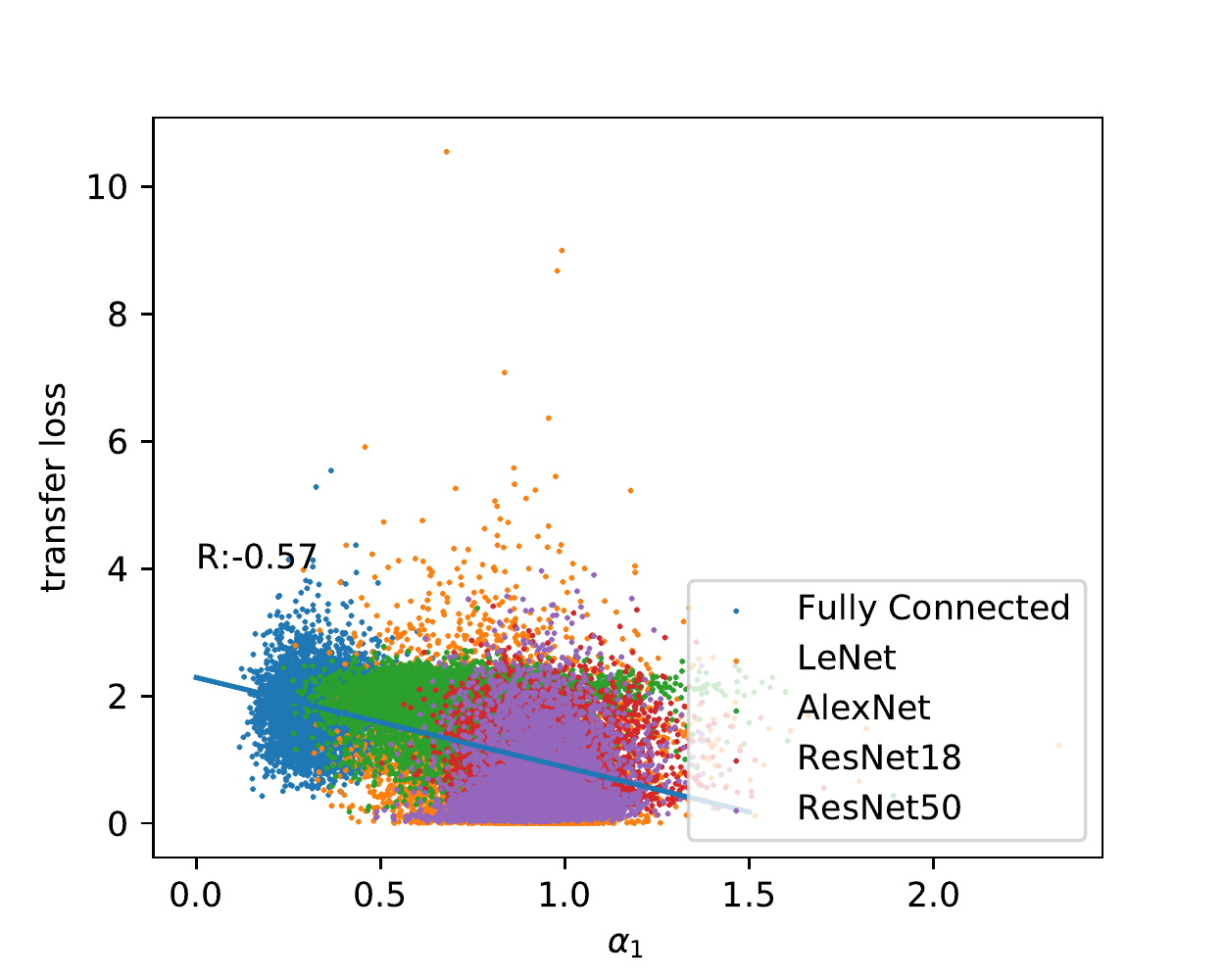}
    \caption{Distribution of per-sample knowledge transfer loss and $\alpha_1$. The adversarial samples are generated by FGSM. The Pearson score shows negative strong correlation between $\alpha_1$ and transfer loss. The higher the transfer loss is, the lower the knowledge transferability is, the lower the $\alpha_1$ should be.}
    \label{fig:exp1_img}
\end{figure}

\begin{table}[h]
    \centering
    \begin{small}
     \resizebox{0.5\textwidth}{!}{
    \begin{tabular}{ccccc}
     \toprule
     \textbf{Model} &\textbf{Knowledge Trans.}&
     $\alpha_1$ &  
     $\alpha_2$ & 
     \textbf{$\alpha_1 * \alpha_2$}\\
     \midrule
      Fully Connected & 28.30 & 0.279 & 0.117 & 0.0103\\
    \hline
    AlexNet & 45.65 & 0.614 & 0.208 & 0.0863\\
    \hline
    LeNet & 55.09 & 0.803 & 0.298 & 0.205\\
    \hline
    ResNet18 & 76.60 & 1.000 & 0.405 & 0.410\\
    \hline
    ResNet50 & 77.92 & 0.962 & 0.392 & 0.368\\
    
     \bottomrule
     \end{tabular}
     }
     \end{small}
    \caption{\small Knowledge transferability (Knowledge Trans.) among different model architectures. Adversarial examples are generated using FGSM attacks. Our correlation analysis shows Pearson score of -0.57 between the transfer loss and $\alpha_1$. Lower transfer loss corresponds to higher transfer accuracy. More details can be found in Figure \ref{fig:exp1_img}} 
    \label{table:exp1_fgsm}
\end{table}

We note that in the figures where we report per-sample $\alpha_1$, although ideally $\alpha_1\in [0,1]$, we can observe that for some samples they have $\alpha_1>1$ due to the attacking algorithm is not ideal in practice. However, the introduced sample-level noise does not affect the overall results, \emph{e.g.}, see the averaged results in our tables, or the overall correlation in these figures.

\subsubsection{NLP}
 In the NLP experiments, to train source and target models, we finetune BERT-base models on different datasets for 3 epochs with learning rate equal to $5e-5$ and warm-up steps equal to the $10\%$ of the total training steps.  For knowledge tranferability,  we random initialize the last layer of source models and fine-tune all layers of BERT for 1 epoch on the targeted dataset (IMDB). 
Based on the test data from the target model, we generate \boxin{$1,000$} textual adversarial examples via the state-of-the-art adversarial attacks T3 \citep{t3} with adversarial learning rate equal to 0.2, maximum iteration steps equal to 100, and $c=\kappa=100$.


\begin{figure}[hbt!]
    \centering
    \includegraphics[width=0.5\linewidth]{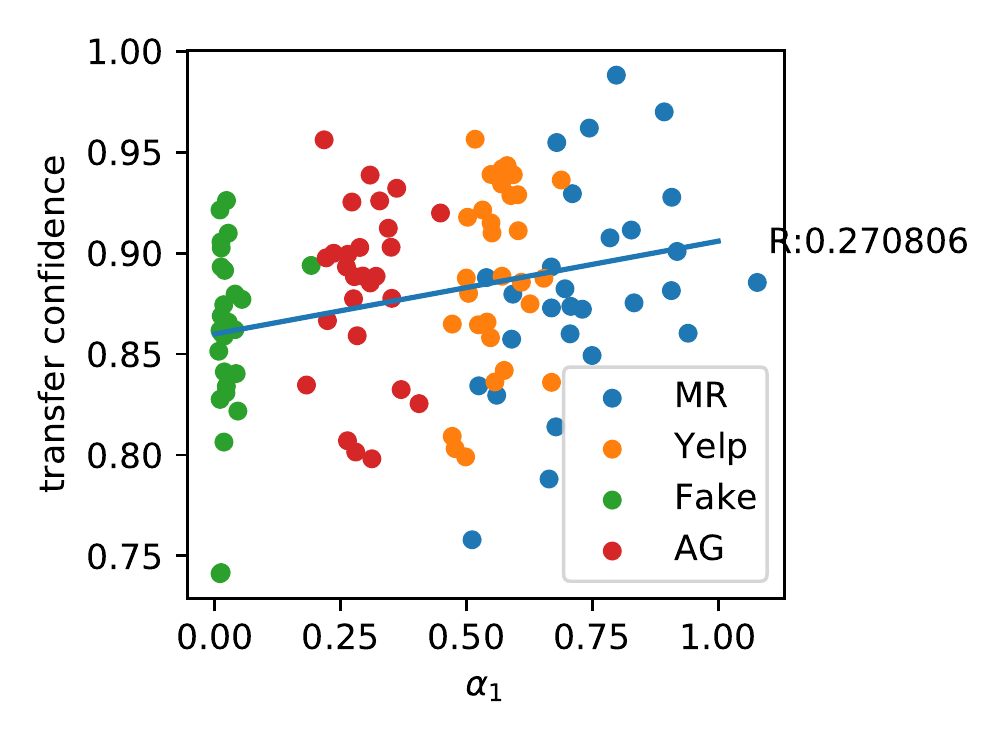}
    \caption{Distribution of per-batch knowledge transfer confidence and $\alpha_1$. The Pearson score shows positive correlation between $\alpha_1$ and transfer confidence. The higher the confidence, the higher the knowledge transferability.}
    \label{fig:exp1_nlp_img}
\end{figure}

\subsubsection{Ablation studies on controlling adversarial transferability}\label{sec:ablation}

We conduct series of experiments on controlling adversarial transferability between source models and target model by promoting their Loss Gradient Diversity. \citet{demontis2019adversarial} shows that for two models $f_S$ and $f_T$, the cosine similarity between their loss gradient vectors $\nabla_x \ell_{f_S}$ and $\nabla_x \ell_{f_T}$ could be a significant indicator measuring two models' adversarial transferability. Moreover, \citet{kariyappa2019improving} claims that adversarial transferability betwen two models could be well controlled by regularizing the cosine similairity between their loss gradient vectors. Inspired by this, we train several source models $f_S$ to one target model $f_T$ with following training loss:
    $$\mathcal{L}_{\text{train}} = \mathcal{L}_{\text{CE}}(f_S(\vx), y)) + \rho\cdot \mathcal{L}_{\text{cos}}(\nabla_{\hat{\vx}} \ell_{f_S}, \nabla_{\hat{\vx}} \ell_{f_T})$$
    where $\mathcal{L}_{\text{CE}}$ refers to cross-entropy loss and $\mathcal{L}_{\text{cos}}(\cdot, \cdot)$ the cosine similarity metric. $\vx$ presents \emph{source domain} instances while $\hat{\vx}$ presents \emph{target domain} instances. We explore $\rho\in\{0.0, 0.5, 1.0, 2.0, 5.0\}$ and finetune each source model for $50$ epochs with learning rate as $0.01$. For knowledge transferability, we random initialize the last layer of each source model and finetune it on STL-10 for 10 epochs with learning rate as $0.01$. During the adversarial example generation, we utilize standard $\ell_\infty$ PGD attack with perturbation scale $\epsilon=0.1$ and 50 attack iterations with step size as $\epsilon/10$.
    
    \begin{table}[!t]
\centering
\caption{\small Knowledge transferability (Knowledge Trans.) among different source models (controlling adversarial transferability by promoting Loss Gradient Diversity). Adversarial transferability is measured by using the adversarial examples generated against the Target Model to attack the Source Models and estimate $\alpha_1$ and $\alpha_2$.}

\begin{tabular}{ccccc}
\toprule
\textbf{Model} & \textbf{Knowledge Trans.} & $\alpha_1$ & $\alpha_2$ & $\alpha_1 * \alpha_2$ \\ \hline
$\rho=0.0$     & 73.91                    & 0.394     & 0.239      & 0.103                 \\ \hline
$\rho=0.5$     & 73.11                     & 0.385     & 0.246      & 0.102                 \\ \hline
$\rho=1.0$     & 72.47                     & 0.371     & 0.244      & 0.100                 \\ \hline
$\rho=2.0$     & 71.62                     & 0.370     & 0.244      & 0.100                 \\ \hline
$\rho=5.0$    & 72.16                     & 0.378     & 0.240      & 0.098                 \\ \bottomrule
\end{tabular}
\label{tab:adx-dverge1}
\end{table}

    Table~\ref{tab:adx-dverge1} shows the relationship between knowledge transferability and adversarial transferability of different source model trained by different $\rho$. With the increasing of $\rho$, the adversarial transferabiltiy between source model and target model decreases ($\alpha_1, \alpha_1 * \alpha_2$ become smaller), and the knowledge transferability also decreases. We also plot the $\alpha_1$ with its corresponding transfer loss on each instance, as shown in Figure \ref{fig:adv_img}. The negative correlation between $\alpha_1$ and transfer loss confirms our theoretical insights.
    
\begin{figure}[hbt!]
    \centering
    \small
    \includegraphics[width=0.5\linewidth]{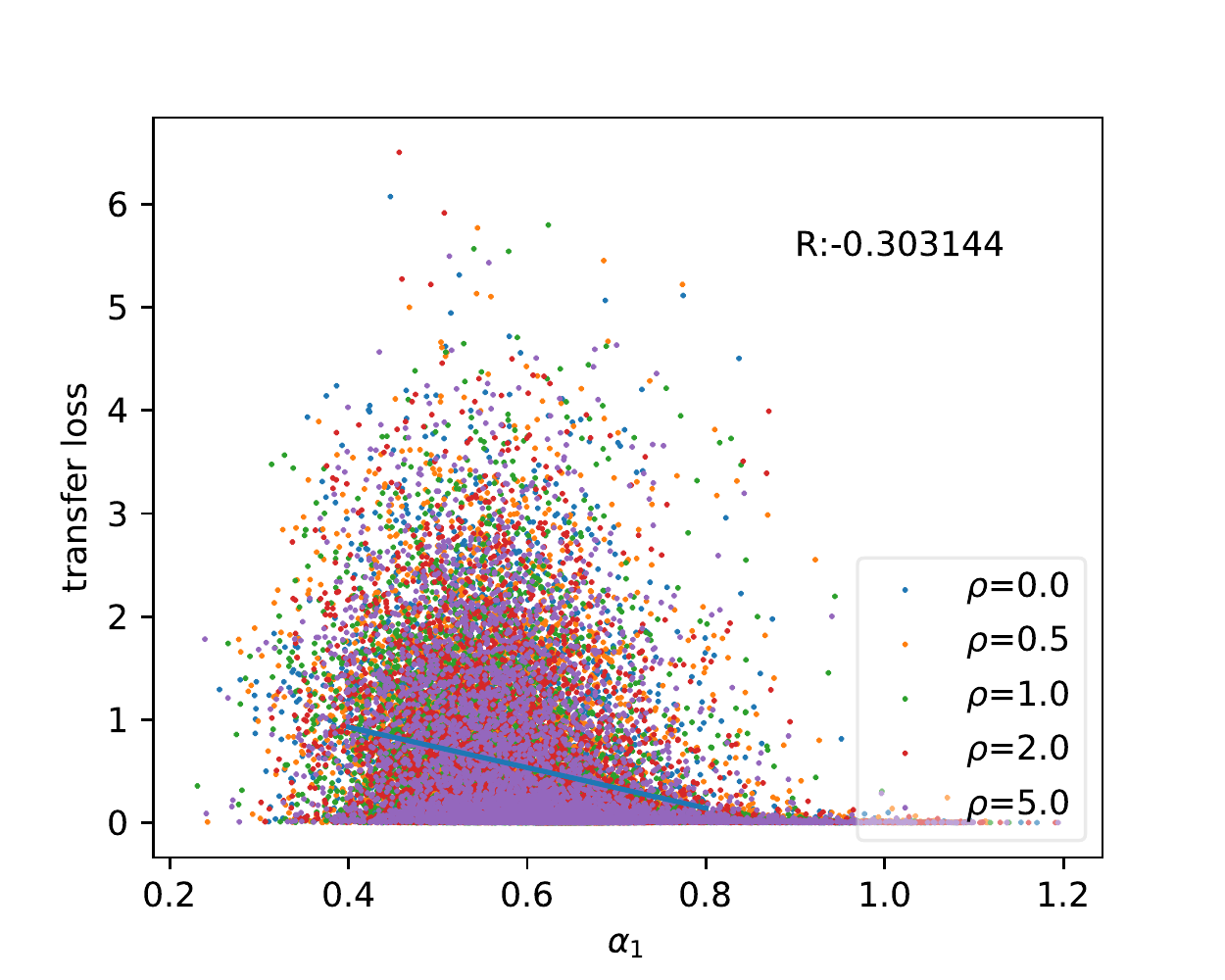}
    \caption{Distribution of per-sample knowledge transfer loss and $\alpha_1$. The Pearson score shows negative correlation between $\alpha_1$ and transfer loss. The higher the loss is, the lower the knowledge transferability is, the lower the $\alpha_1$ should be.}
    \label{fig:adv_img}
\end{figure}

\subsection{Knowledge Trasnferability  Indicating Adversarial Transferability}
\subsubsection{Image}
  We follow the same setup in the previous image experiment for source model training, transfer learning as well as generation of adversarial examples. However, there is one key difference: Instead of generating adversarial examples on the target model and measuring adversarial transferability on source models, we generate adversarial examples on each source model and measure the adversarial transferability by feeding these adversarial examples to the target model.
  
  Similarly, we also visualize the results (Table \ref{table:exp2}) and compute the Pearson score. Due to the significant noise introduced by per-sample calculation, the R score is not as significant as figure \ref{fig:exp1_img}, but the trend is still correct and valid, which shows that higher knowledge transferability indicates higher adversarial transferability.

\begin{figure}[hbt!]
    \centering
    \includegraphics[width=0.5\linewidth]{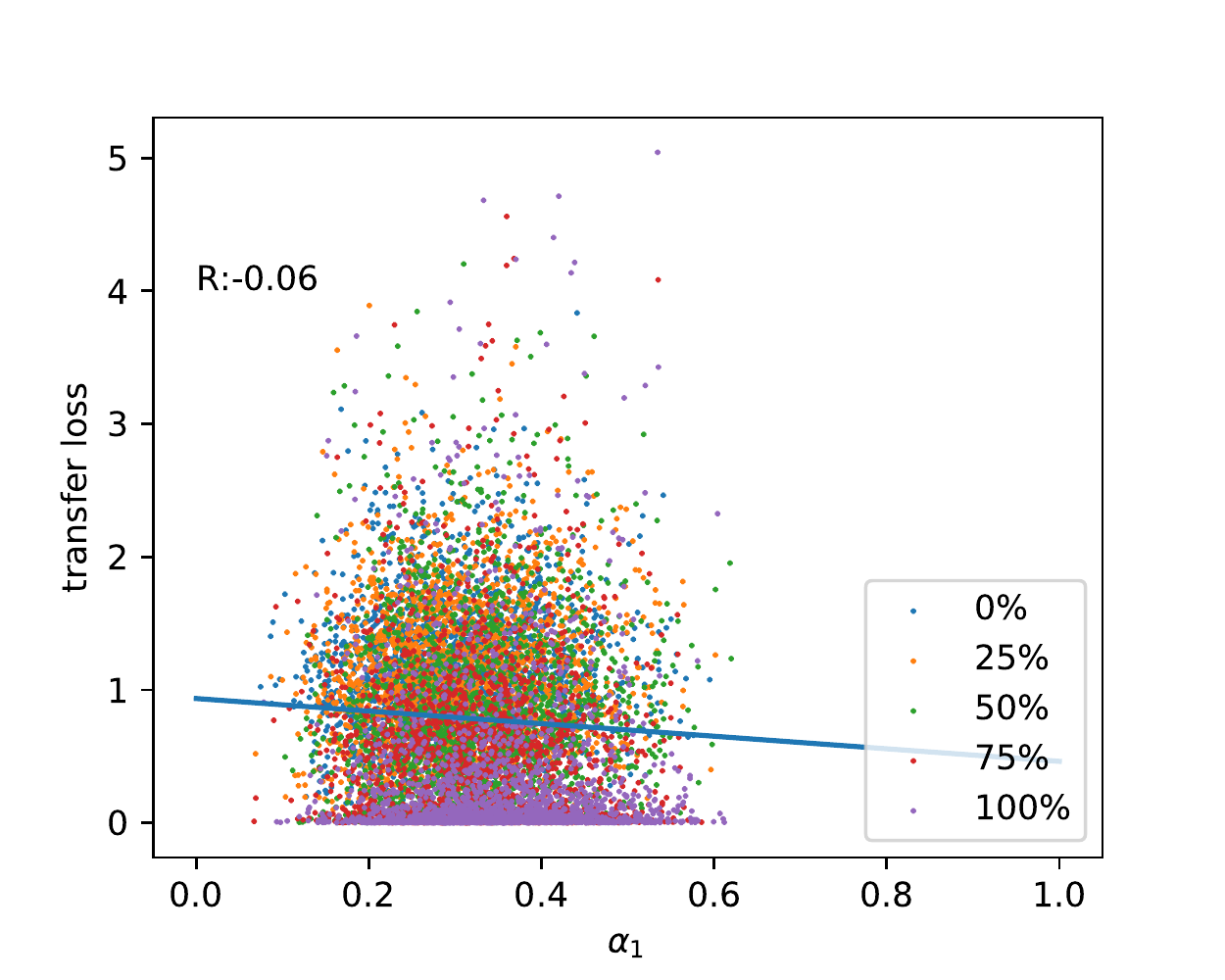}
    \caption{Distribution of per-sample knowledge transfer loss and $\alpha$. The Pearson score shows negative strong correlation between $\alpha$ and transfer loss. The higher the loss is, the lower the knowledge transferability is, and the lower the $\alpha_1$  is.}
    \label{fig:exp2_img}
\end{figure}

\subsubsection{NLP}
We follow the same setup to train the models and generate textual adversarial examples as \S \ref{sec:adv} in the NLP experiments. We note that to measure the adversarial transferability, we generate $1,000$ adversarial examples on each source model based on the test data from the target model, and measure the adversarial transferability by feeding these adversarial examples to the target model.


\begin{figure}[hbt!]
    \centering
    \includegraphics[width=0.5\linewidth]{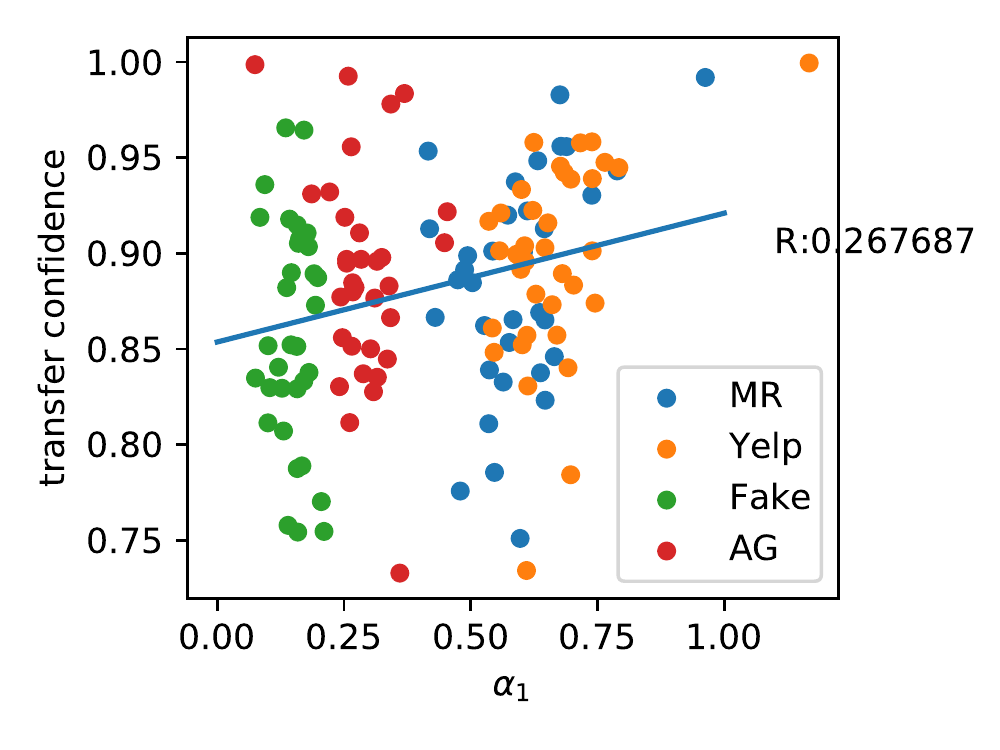}
    \caption{Distribution of per-batch knowledge transfer confidence and $\alpha_1$. The Pearson score shows positive correlation between $\alpha_1$ and transfer confidence. The higher the confidence, the higher the knowledge transferability.}
    \label{fig:exp2_nlp_img}
\end{figure}


\end{document}